\documentclass[twoside]{article}

\usepackage[accepted]{aistats2021}
\usepackage[round]{natbib}

\usepackage{microtype}
\usepackage{subfigure}
\usepackage{booktabs} 

\usepackage{dsfont}
\usepackage{amsmath,amssymb,amsfonts,amsthm}
\usepackage{algorithm,algpseudocode,algorithmicx}
\usepackage{mathtools,commath}
\usepackage{graphicx,multirow}
\usepackage[round]{natbib}

\usepackage{xcolor}
\usepackage{authblk}
\usepackage{hyperref, url}
\usepackage{wrapfig, lipsum, mdwlist}
\usepackage[calc]{adjustbox}
\usepackage{cutwin}

\newtheorem{theorem}{Theorem}
\newtheorem{corollary}[theorem]{Corollary}
\newtheorem{lemma}[theorem]{Lemma}
\newtheorem{proposition}[theorem]{Proposition}
\newtheorem{definition}{Definition}
\newtheorem{example}{Example}
\newtheorem{remark}{Remark}

\newlength{\strutheight}
\usepackage{lmodern}
\allowdisplaybreaks
\begin{document}

\twocolumn[
\aistatstitle{Efficient Computation and Analysis of Distributional Shapley Values}

\aistatsauthor{ Yongchan Kwon \And Manuel A. Rivas \And  James Zou }
\aistatsauthor{\texttt{\{yckwon,mrivas,jamesz\}@stanford.edu}}
\aistatsaddress{ Department of Biomedical Data Science, Stanford University, CA, USA} ]

\begin{abstract}
Distributional data Shapley value (DShapley) has recently been proposed as a principled framework to quantify the contribution of individual datum in machine learning. DShapley develops the foundational game theory concept of Shapley values into a statistical framework and can be applied to identify data points that are useful (or harmful) to a learning algorithm. Estimating DShapley is computationally expensive, however, and this can be a major challenge to using it in practice. Moreover, there has been little mathematical analyses of how this value depends on data characteristics. In this paper, we derive the first analytic expressions for DShapley for the canonical problems of linear regression, binary classification, and non-parametric density estimation. These analytic forms provide new algorithms to estimate DShapley that are several orders of magnitude faster than previous state-of-the-art methods. Furthermore, our formulas are directly interpretable and provide quantitative insights into how the value varies for different types of data. We demonstrate the practical efficacy of our approach on multiple real and synthetic datasets.
\end{abstract}

\section{Introduction}
\label{s:intro}
Data valuation has emerged as an important topic for machine learning (ML) as well as for the broader discussions around the economics of data. Proposed policies such as the Designing Accounting Safeguard to Help Broaden Oversight and Regulations on Data Act, also known as DASHBOARD Act, and the Data Dividend in the US would stipulate that companies need to quantify the value of the data that they collect from customers \citep{dashboard2019, wadhwa2020economic}. Such valuation could have important implications for policy, regulation, taxation and potentially even for individual compensation \citep{posner2018radical}. Recently data Shapley, a data valuation framework based on the foundational Shapley value in economics, has gained significant attention \citep{ghorbani2019, jia2019}. Data Shapley is appealing from a policy perspective because it inherits the same fair allocation properties that the original Shapley value uniquely satisfies. Moreover, it has shown to empirically capture the notion of which datum helps or harms the ML model.

A fundamental limitation of data Shapley, however, is that it is defined with respect to a fixed dataset. The statistical and random nature of data is ignored. Accordingly, data Shapley needs to be recalculated even when the dataset changes slightly, which is computationally expensive, and it could also be unstable for randomly drawn datasets.
To tackle these challenges, \citet{ghorbani2020} proposed  distributional Shapley value (DShapley) as the natural statistical extension of the Shapley value, by considering the expected value of data Shapley value with respect to the underlying distribution.
While DShapley is numerically more stable and does not require the aforementioned recalculation, DShapley is still mathematically challenging to analyze and computationally hard to estimate.

In this paper, we address these challenges by developing rigorous analyses and computationally efficient algorithms for DShapley. \textbf{Theoretical contributions:} we develop the first analytic expressions for DShapley for linear regression, binary classification, and non-parametric density estimation, which are widely used canonical examples of supervised and unsupervised learning. Our formulations are also easy to interpret and provide direct insights into how DShapley behaves for different data. \textbf{Algorithmic contributions:} based on our theory, we provide new algorithms to efficiently estimate DShapley which is several orders of magnitude faster than previous state-of-the-art methods. We support our analyses with numerical experiments on both real and synthetic datasets. 

\paragraph{Related works}
Shapley value was initially proposed in a seminar paper \citep{shapley1953} and has been studied extensively in the field of cooperative game theory \citep{dubey1981, grabisch1999, aumann2015}.
Shapley value has been widely applied in economics \citep{gul1989, moulin1992}, management science \citep{dubey1982} and has also been appeared in ML literature. 
Examples include feature selection \citep{cohen2005, zaeri2018feature}, data marketplace design \citep{agarwal2019, fernandez2020}, and model explanation \citep{lundberg2017, chen2018, sundararajan2019many, ghorbani2020neuron}.

Another body is Shapley value-based data valuation methods, yet most of the literature focuses on data Shapley values \citep{ghorbani2019, jia2019}.
Data Shapley value has been shown to empirically work better than other methods of data valuation, such as using leave-one-out residual estimate \citep{cook1982residuals}, or influence-based scores \citep{hampel1974, koh2017}, but it can cause expensive computational costs when data are regularly collected. Other promising data valuation schemes have been proposed to leverage reinforcement learning \citep{yoon2019data}. These approaches lack the fairness principles that has uniquely satisfied by the Shapley value.

DShapley was introduced as a rigorous statistical extension of Shapley value \citep{ghorbani2020}.
Previous to our work, the only computationally efficient form for data Shapley is just for the nearest neighbor classifier \citep{jia2019b}; and similar results are not known for DShapley.   
Our work develops principled and efficient methods for analyzing and computing DShapley.

\section{Preliminaries}
\label{s:preliminaries}
We review existing Shapley value-based data valuation methods.
To begin with, we define some notations.
Let $Z$ be a random variable for data defined on $\mathcal{Z} \subseteq \mathbb{R}^d$ and denote its distribution by $P_{Z}$. 
In supervised learning, we set $Z = (X, Y)$ defined on $\mathcal{X} \times \mathcal{Y}$, where $X$ and $Y$ are the input and its label, respectively, and in unsupervised learning $Z = X$.
We denote a utility function by $U: \cup_{j=0} ^{\infty} \mathcal{Z} ^j \to \mathbb{R}$. 
Here, the utility function describes model performance. For instance, in classification, $U(S)$ could be the test accuracy of a model trained using a subset $S \subseteq \mathcal{X} \times \mathcal{Y}$.
We define the marginal contribution of $z^* \in \mathcal{Z}$ with respect to $S \subseteq \mathcal{Z}$ as $\Delta(z^*; U, S) := U(S\cup \{z^*\})-U(S)$.
We use the conventions $\mathcal{Z} ^0 := \{ \emptyset \}$ and $U(\emptyset) = 0$.
For a set $S$, we denote its cardinality by $|S|$, and we use $[m]$ to denote a set of integers $\{ 1, \dots, m\}$. 
 
Data Shapley value applies the cooperative game theory concept of Shapley value to ML problems \citep{ghorbani2019,jia2019}.
More precisely, given a utility function $U$ and a fixed dataset $B \subseteq \mathcal{Z}$ with $|B|=m$, data Shapley value of a point $z^* \in B$ is defined as
\begin{align}
    \phi(z^*; U, B) := \frac{1}{m} \sum_{j=1} ^m \frac{1}{\binom{m-1}{j-1}} \sum_{ S \in B_{j} ^{\backslash z^*} } \Delta(z^*; U, S),
    \label{eqn:def_data_shapley_value}
\end{align}
where $B_{j} ^{\backslash z^*} := \{ S \subseteq B \backslash \{z^*\}: |S|=j-1 \}$ for $j \in \mathbb{N}$.
Note that the cardinality $|B_{j} ^{\backslash z^*}|$ is $\binom{m-1}{j-1}$ for all $j \in [m]$. 
That is, data Shapley value \eqref{eqn:def_data_shapley_value} is a weighted average of the marginal contribution $\Delta(z^*; U, S)$. 
Data Shapley provides a principled data valuation regime in that the value \eqref{eqn:def_data_shapley_value} uniquely satisfies the natural properties of fair valuation, namely, symmetry, null player, and additivity \citep{ghorbani2019, jia2019b}.
We review these properties and the uniqueness of data Shapley value in Appendix.

Despite the aforementioned promising theoretical characteristics, data Shapley value has a critical limitation; the original data Shapley value is defined with respect to a fixed dataset $B$. Even if a single point in $B$ is changed, in principle, all of the values should be recomputed and the exact computation of the value costs exponential computational complexity. This is particularly problematic in typical statistics and ML settings, where the data points are regularly collected from an underlying distribution. In order to resolve this issue and capture the statistical nature of data valuation, DShapley has been proposed where data Shapley is treated as a random variable \citep{ghorbani2020}. 
To be more specific, given a utility function $U$, a data distribution $P_Z$, and some $m \in \mathbb{N}$, \citet{ghorbani2020} defined DShapley of a point $z^*$ as
\begin{align}
    \nu( z^* ; U, P_{Z},m) := \mathbb{E}_{B \sim P_{Z}^{m-1}}[\phi(z^*; U, B\cup\{z^*\})].
    \label{eqn:def_distri_shapley_value}
\end{align}
DShapley \eqref{eqn:def_distri_shapley_value} is the expectation of data Shapley value \eqref{eqn:def_data_shapley_value} over random datasets of size $m$ containing $z^*$. \citet{ghorbani2020} further showed that DShapley possesses some desirable properties. For instance, DShapley is stable under small perturbations to the data points themselves and to the underlying data distribution (\citet[Theorems 2.7 and 2.8]{ghorbani2020}), which have not been clear with \eqref{eqn:def_data_shapley_value}. However, estimating DShapley is still computationally expensive and thus it critically hampers the practical use of DShapley. In this paper, we focus on canonical problems of linear regression, binary classification, and non-parametric density estimation, deriving new expressions for DShapley that lead to new mathematical insights and efficient computation algorithms.

\section{Distributional Shapley values for linear regression and classification}
\label{s:dsv_linear}
We present rigorous analyses of DShapley for linear regression problems.
In Sec.~\ref{s:dsv_no_assumptions}, we first provide a general reformulation of DShapley without distributional assumptions on inputs.
In Sec.~\ref{s:dsv_gaussian_inputs}, we simplify DShapley as a function of Mahalanobis distance and an error when inputs are Gaussian.
In Sec.~\ref{s:dsv_sub-Gaussian_inputs}, we consider sub-Gaussian inputs and provide upper and lower bounds for DShapley.
In Sec.~\ref{s:application-to-classification}, we present an application of our theoretical result to binary classification.

\subsection{A general reformulation of distributional Shapley values}
\label{s:dsv_no_assumptions}
Throughout this section, we let $(X,Y)$ be a pair of input and output random variables defined on $\mathcal{X} \times \mathcal{Y} \subseteq \mathbb{R}^p \times \mathbb{R}$.
We assume that $Y = X^T \beta + e$ is the underlying linear model where $e$ is a random error whose mean is zero and variance is $\sigma^2$.
Here, $X$ can come from an arbitrary distribution with bounded first two moments.
For a subset $S \subseteq \mathcal{X} \times \mathcal{Y}$, we denote a design matrix and its corresponding output vector based on $S$ by $X_S \in \mathbb{R}^{|S| \times p}$ and $Y_S \in \mathbb{R}^{|S|}$, respectively.
For $\gamma \geq 0$, the ridge regression estimator based on $S$ is defined as $\hat{\beta}_{S,\gamma} := (X_S ^T X_S + \gamma I_p)^{-1} X_S ^T Y_S$ where $I_p$ is the $p \times p$ identity matrix. 
For $q \in \mathbb{N}$, a constant $C_{\mathrm{lin}} > 0$, and an estimator $\hat{\beta} \in \mathbb{R}^{p}$, we define a utility function as $U_q (S, \hat{\beta}) := ( C_{\mathrm{lin}} - \int  (y - x^T \hat{\beta} )^2 dP_{X,Y}(x,y)) \mathds{1}(|S| \geq q)$. Here, $\mathds{1}(\cdot)$ is the indicator function. To this end, we suppress the notation if the ridge regression estimator is used, \textit{i.e.}, $U_{q,\gamma}(S) := U_q (S, \hat{\beta}_{S,\gamma})$. 
We denote the Gaussian distribution with mean $\mu$ and covariance $\Sigma$ by $\mathcal{N}(\mu, \Sigma)$.
Lastly, we denote the data to be valued by $(x^*, y^*)$ and its error by $e^{*} := y^* - x^{*T}\beta$.

The DShapley can be equivalently expressed as follows~\citep{ghorbani2020}: 
\begin{align}
    & \nu( (x^*, y^*); U_{q,\gamma}, P_{X,Y},m) \notag \\
    & = \mathbb{E}_{ j \sim [m]} \mathbb{E}_{S \sim P_{X,Y}^{j-1}} [\Delta((x^*, y^*); U_{q,\gamma}, S)]
    \label{eqn:reform_distri_shapley_value}
\end{align}
where $j \sim [m]$ denotes $j$ follows a uniform distribution over $[m]$. 
Using Equation \eqref{eqn:reform_distri_shapley_value}, we further derive a general reformulation of DShapley in the following proposition.

\begin{proposition}[A general form of DShapley]
Let $\mathbb{E}[Y \mid X]= X^T \beta$,  $\mathrm{Var}(Y \mid X) = \sigma^2$, and $\mathbb{E}(X X^T) = \Sigma_X$.
Then, for any $q \geq 2$ and some fixed constant $C_{\mathrm{lin}}$, DShapley of a point $(x^*, y^*)$ with the ridge regression estimator is given by 
\begin{align*}
& \nu((x^*, y^*); U_{q,\gamma}, P_{X,Y},m) \notag \\
&= \frac{1}{m} \sum_{j=q} ^m  \mathbb{E}_{X_S \sim P_X ^{j-1}} \Bigg[ \frac{ x^{*T} A_{S, \gamma} ^{-1} \Sigma_X A_{S, \gamma} ^{-1} x^{*} }{(1+x^{*T} A_{S, \gamma} ^{-1} x^* )^2} \\
&\times \left( (2 + x^{*T} A_{S, \gamma} ^{-1} x^{*})\sigma ^{2} - e ^{*2} \right) \Bigg] + h(\gamma),
\end{align*}
where $A_{S, \gamma} ^{-1} := (X_S ^T X_S+ \gamma I_p)^{-1}$ and $h(\gamma)$ is a term such that $\lim_{\gamma \to 0+} h(\gamma)/ (\gamma \log(\gamma))=0$ and $h(0)=0$. 
\label{proposition:ridge}
\end{proposition}

In the expression \eqref{eqn:reform_distri_shapley_value}, different choices of $C_{\mathrm{lin}}$ in the utility function $U_{q, \gamma}$ cause constant changes in DShapley.
To be more specific, for a fixed $C \in \mathbb{R}$ and for all $S \subseteq \mathcal{X} \times \mathcal{Y}$, suppose $U_{q, \gamma} (S) = \tilde{U}_{q, \gamma} (S) + C \mathds{1}(|S| \geq q)$. Then DShapley is $\nu((x^*, y^*); U_{q,\gamma}, P_{X,Y},m) = \nu((x^*, y^*); \tilde{U}_{q,\gamma}, P_{X,Y},m) + C/m$.
In this respect, we simply choose a constant $C_{\mathrm{lin}}$ that gives the simplest form in Proposition \ref{proposition:ridge} and the following results.

Proposition \ref{proposition:ridge} simplifies the expected value of the marginal contributions of $(x^*,y^*)$ in Equation \eqref{eqn:reform_distri_shapley_value} with a few terms such as the squared error $e^{*2}$ and the ridge leverage score $x^{*T} A_{S, \gamma} ^{-1} x^*$ \citep{cohen2017, mccurdy2018}.
This new formulation provides mathematical insights and interpretations.
For a fixed $x^*$, DShapley is negatively related to the squared error $e^{*2}$ as long as $\gamma$ is small enough; as the error decreases, DShapley increases.
In addition, DShapley is determined only by the first two conditional moments of $Y$ given $X$, meaning that it does not rely on other higher moments or a particular distribution of $Y$.
Furthermore, it is noteworthy that Proposition \ref{proposition:ridge} does not require a specific distributional assumption on $X$ except for the moment condition $\mathbb{E}(X X^T) = \Sigma_X$.
In the following sections, we pay more attention to the input distribution and propose computationally efficient algorithms for DShapley.

\subsection{Distributional Shapley value when inputs are Gaussian}
\label{s:dsv_gaussian_inputs}
When input data are Gaussian, we introduce a new expression for DShapley in the following theorem. 
To begin with, for $k \geq 1$, we denote the Chi-squared distribution with $k$ degree of freedom by $\chi_k ^2$. 

\begin{theorem}[DShapley when inputs are Gaussian]
Assume $\mathbb{E}[Y \mid X]= X^T \beta$, $\mathrm{Var}(Y \mid X) = \sigma^2$ and $X \sim \mathcal{N}_p(0,\Sigma_X)$. 
For $j \geq p$, let $T_j$ be a Chi-squared random variable with $j-p+1$ degree of freedom, \textit{i.e.}, $T_j \sim \chi_{j-p+1} ^2$.
Then, for any $q \geq p+3$ and some fixed constant $C_{\mathrm{lin}}$, DShapley of a point $(x^*, y^*)$ with the least squares estimator is given by
\begin{align}
& \nu( (x^*, y^*); U_{q,0}, P_{X,Y}, m) \notag \\
& = - \frac{1}{m} \sum_{j=q} ^m  \mathbb{E} \left[ \frac{j-1}{j-p} \frac{ \left(  x^{*T} \Sigma_X ^{-1} x^* e^{*2} + T_{j} \sigma^{2} \right) }{ ( x^{*T} \Sigma_X ^{-1} x^* + T_{j} )^2  } \right],
\label{eqn:DSV_gaussian_exact}
\end{align}
where the expectation is with respect to the Chi-squared distributions.
\label{thm:gaussian_inputs_lse}
\end{theorem}

Theorem \ref{thm:gaussian_inputs_lse} presents a new representation of DShapley when $\gamma=0$ and inputs are Gaussian.
The new form \eqref{eqn:DSV_gaussian_exact} depends only on the two terms, the error $e^{*2}$ and the term $x^{*T} \Sigma_X ^{-1} x^*$, also known as the Mahalanobis distance of $x^*$ from zero with respect to $\Sigma_X$.
Likewise Proposition \ref{proposition:ridge}, a direct implication is that any points with the same error level have the same DShapley when they have the same Mahalanobis distance.
In addition, a role of $e^{*2}$ is also explicitly explained.
DShapley for the point with the smaller squared error is higher than the other point, \textit{i.e.}, $\nu((x^*, y_1 ^*); U_{q,0}, P_{X,Y},m) \leq \nu((x^*, y_2 ^*); U_{q,0}, P_{X,Y},m)$ if $e_1 ^{*2} \geq e_2 ^{*2}$.
This inequality matches our intuitions that the big error $e^{*2}$ is likely to produce small marginal contributions $\Delta((x^*, y^*); U_{q,\gamma}, S)$.
We provide illustrations on how DShapley changes with respect to $x^{*T}\Sigma_X^{-1}x^{*}$ and $e^{*2}$ in Appendix.

\paragraph{Efficient estimation of DShapley} As for the estimation of DShapley $\nu( (x^*, y^*); U_{q,0}, P_{X,Y}, m)$, we propose to use the Monte-Carlo approximation method.
We describe a simple version of the proposed algorithm in Alg.~\ref{alg:gaussian_inputs_lse_simple}. A detailed version is provided in Appendix.

A similar idea has been suggested in a number of algorithms including \texttt{TMC-SHAPLEY} \citep{ghorbani2019} or \texttt{$\mathcal{D}$-SHAPLEY} \citep{ghorbani2020}.
Although the previous state-of-the-art algorithms and the proposed algorithm make use of the Monte-Carlo method, there are notable differences.
Since the previous algorithms are based on Equation \eqref{eqn:reform_distri_shapley_value}, they require the utility evaluation $U_{q,0}(S)$ for every random dataset $S$. This computation is expensive because it includes the matrix inversion $(X_S ^T X_S)^{-1}$.
However, the proposed algorithm avoids such computational costs because the new form \eqref{eqn:DSV_gaussian_exact} has nothing to do with a random dataset $S$.
This characteristic is not obtained with Equation \eqref{eqn:reform_distri_shapley_value} and Proposition \ref{proposition:ridge}.
In terms of the computational complexity, when the maximum number of Monte-Carlo samples is $T$, the previous state-of-the-art algorithms require $O(m T p^3)$ computations.
In contrast, the proposed Alg.~\ref{alg:gaussian_inputs_lse_simple} only needs to perform the matrix inversion once for $\Sigma_X ^{-1}$, and the computational complexity for the proposed algorithm is  $O(mT+p^3)$, which is substantially smaller since $T$ is usually large.

\begin{algorithm}[t]
\caption{DShapley for the least squares estimator when inputs are Gaussian.}
\begin{algorithmic}
\Require Estimates for ${x}^{*T} \Sigma_X ^{-1} {x}^{*}$, $e^{*2}$, and $\sigma^2$.
The maximum number of Monte Carlo samples $T$.
A utility hyperparameter $q \geq p+3$.
\Procedure{}{}
\For{$j \in \{ q ,\dots, m \}$}
\State Sample $t_{[1]}, \dots, t_{[T]}$ from the $\chi_{j-p+1} ^2$.
\State $A_j \leftarrow \frac{1}{T} \sum_{i=1} ^T \frac{j-1}{j-p} \frac{x^{*T} \Sigma_X ^{-1} x^* e^{*2} + t_{[i]} \sigma^2 }{( {x}^{*T} \Sigma_X ^{-1} {x}^{*} + t_{[i]} )^2} $
\State $\hat{\nu}  \leftarrow \hat{\nu} - A_j /m$
\EndFor
\State $\hat{\nu}( (x^*, y^*); U_q, P_{X,Y}, m) \leftarrow \hat{\nu}$ 
\EndProcedure
\end{algorithmic}
\label{alg:gaussian_inputs_lse_simple}
\end{algorithm}

\subsection{Distributional Shapley values when inputs are sub-Gaussian}
\label{s:dsv_sub-Gaussian_inputs}

In this section, we develop closed-form bounds for DShapley when inputs are sub-Gaussian. 
To be more formal, we first define the sub-Gaussian.

\begin{definition}[Sub-Gaussian]
We say that a random variable $X$ in $\mathbb{R}$ is sub-Gaussian if there are positive constants $C_{\mathrm{sub}}$ and $v_{\mathrm{sub}}$ such that for every $t>0$, $P(|X| > t) \leq C_{\mathrm{sub}} e^{-v_{\mathrm{sub}} t^2}$ holds.
In addition, we say that a random vector $X$ in $\mathbb{R}^p$ is sub-Gaussian if the one-dimensional marginals $\langle X, x \rangle$ are sub-Gaussian random variables for all $x \in \mathbb{R}^p$.
\end{definition}
Note that a class of sub-Gaussian includes many useful random variables such as Gaussian and any bounded random variables \citep{vershynin2010}. 
Now we develop bounds for DShapley in the following theorem.

\begin{theorem}[Upper and lower bounds for DShapley when inputs are sub-Gaussian]
Assume that $\mathbb{E}[Y \mid X]= X^T \beta$ and $\mathrm{Var}(Y \mid X) = \sigma^2$. 
Suppose $\mathcal{Y}$ is bounded and $X$ are sub-Gaussian in $\mathbb{R}^p$ with $\mathbb{E}(X X^T) = \Sigma_X$.
Then, for $q \geq 2$ and some fixed constant $C_{\mathrm{lin}}$, DShapley of a point $(x^*, y^*)$ with the ridge regression estimator has the following bounds.
\begin{align*}
&h(\gamma) + \frac{1}{m} \sum_{j=q-1} ^{m-1}  \frac{ {x}^{*T} \Sigma_X ^{-1} {x}^{*} \Lambda_{\mathrm{lower}}^{2} (j) }{ (1+{x}^{*T} \Sigma_X ^{-1} {x}^{*} \Lambda_{\mathrm{upper}} (j) )^2} \\
&\times \left( ( 2 +{x}^{*T} \Sigma_X ^{-1} {x}^{*} \Lambda_{\mathrm{lower}} (j) ) \sigma^2  - \Lambda_{\mathrm{ratio}} ^{-1} (j) e^{*2} \right) \\
\leq& \nu( (x^*,y^*) ; U_{q,\gamma}, P_{X,Y}, m)  + o\left(\frac{1}{m}\right) \\
\leq&  h(\gamma) + \frac{1}{m} \sum_{j=q-1} ^{m-1}  \frac{ {x}^{*T} \Sigma_X ^{-1} {x}^{*} \Lambda_{\mathrm{upper}}^{2} (j) }{ (1+{x}^{*T} \Sigma_X ^{-1} {x}^{*} \Lambda_{\mathrm{lower}} (j) )^2} \\
&\times \left( ( 2 +{x}^{*T} \Sigma_X ^{-1} {x}^{*} \Lambda_{\mathrm{upper}} (j) ) \sigma^2  - \Lambda_{\mathrm{ratio}}(j) e^{*2} \right),
\end{align*}
where the function $h$ is defined in Proposition \ref{proposition:ridge} and
\begin{align*}
\Lambda_{\mathrm{ratio}}(j) = \left( \frac{ 1+{x}^{*T} \Sigma_X ^{-1} {x}^{*} \Lambda_{\mathrm{lower}} (j) }{ 1+{x}^{*T} \Sigma_X ^{-1} {x}^{*} \Lambda_{\mathrm{upper}} (j) } \right)^2,
\end{align*}
$\Lambda_{\mathrm{lower}}(j)$ and $\Lambda_{\mathrm{upper}}(j)$ are two explicit constants that scale $O(1/j)$ and depend only on $\gamma$ and the sub-Gaussian distribution. 
The explicit expression for $\Lambda_{\mathrm{lower}}(j)$ and $\Lambda_{\mathrm{upper}}(j)$ are provided in Appendix.
\label{thm:sub_gaussian_inputs_ridge}
\end{theorem}

Theorem \ref{thm:sub_gaussian_inputs_ridge} provides upper and lower bounds for DShapley when inputs are sub-Gaussian. 
As Theorem \ref{thm:gaussian_inputs_lse}, the main component of the bounds consists of the Mahalanobis distance $x^{*T}\Sigma_X ^{-1} x^*$  and the squared error $e^{*2}$.
Hence, data points with the same Mahalanobis distance lead to having the same bounds if the error levels are the same.
Although the new bounds in Theorem \ref{thm:sub_gaussian_inputs_ridge} are not the exact form of DShapley, they are analytically expressed, and can be efficiently computed without Monte Carlo sampling.

\paragraph{The two assumptions in Theorem \ref{thm:sub_gaussian_inputs_ridge}}
Compared to Proposition \ref{proposition:ridge}, we additionally assume the boundness of $\mathcal{Y}$ and the sub-Gaussian distribution on inputs in Theorem \ref{thm:sub_gaussian_inputs_ridge}.
The former implies the boundness of the marginal contribution $\Delta((x^*, y^*); U_{q,\gamma}, S)$ for all $S \subseteq \mathcal{X} \times \mathcal{Y}$, and the latter ensures that eigenvalues of $A_{S,\gamma}^{-1}$ are in the closed interval $[\Lambda_{\mathrm{lower}}(j), \Lambda_{\mathrm{upper}}(j)]$ with high probability.
Combining these two ingredients, we obtain the bounds for DShapley as a function of $\Lambda_{\mathrm{lower}}(j)$ and $ \Lambda_{\mathrm{upper}}(j)$.

\subsection{Application to binary classification}
\label{s:application-to-classification}
We now study an efficient DShapley estimation method for binary classification datasets. Our approach is to transform binary classification data and apply Theorem \ref{thm:sub_gaussian_inputs_ridge}.
To be more precise, let $(X,Y)$ be a pair of input and output random variables and assume $\mathbb{E}(Y \mid X ) = \pi = \mathrm{logit}^{-1}(X ^T \beta)$. Here, $\mathrm{logit}(\pi) := \pi/(1-\pi)$ for $\pi \in (0,1)$. We define the working dependent variable $Z$ and its corresponding weight $w$ as
\begin{align}
    Z = \eta + (Y - \pi) \frac{\partial \eta }{\partial \pi } \text{ and } w = \pi(1-\pi),
    \label{eqn:irls_variabls}
\end{align}
respectively, where $\eta = X^T \beta$. Note that $\partial \eta / \partial \pi = w ^{-1}$.
We propose to consider DShapley with respect to the transformed random variables $(\tilde{X}, \tilde{Z}) := ( w ^{1/2} X , w ^{1/2} Z )$ instead of $(X,Y)$.
In the following corollary, we provide a lower bound of DShapley in binary classification. An upper bound and detailed notations are provided in Appendix.

\begin{corollary}
[DShapley in binary classification]
Assume $\mathbb{E}[Y \mid X]= \mathrm{logit}^{-1}(X ^T \beta)$ and $X$ are sub-Gaussian in $\mathbb{R}^p$ with $\mathbb{E}(X X^T) = \Sigma_X$. 
For a point $(x^*, y^*)$, let $\pi^* = \mathrm{logit}^{-1} (x^{*T}\beta)$,  $w^*=\pi^* (1-\pi^*)$, and $z^* = x^{*T}\beta + (y^*-\pi^*)/w^*$.
Then, for any $q \geq p+3$ and some fixed constant $C_{\mathrm{lin}}$, DShapley of a point $\left( (w^*)^{1/2}x^*, (w^*)^{1/2}z^* \right)$ has a lower bound given by
\begin{align*}
&\nu \left( \left( (w^*)^{1/2}x^*, (w^*)^{1/2}z^* \right); U_{q,0}, P_{\tilde{X},\tilde{Z}}, m \right) \\
&\geq \frac{1}{m} \sum_{j=q-1} ^{m-1}  \frac{ w^* {x}^{*T} \tilde{\Sigma}_X ^{-1} {x}^{*} \tilde{\Lambda}_{\mathrm{lower}}^{2} (j) }{ (1+ w^* {x}^{*T} \tilde{\Sigma}_X ^{-1} {x}^{*} \tilde{\Lambda}_{\mathrm{upper}} (j) )^2} \\
&\times \left( ( 2 + w^* {x}^{*T} \tilde{\Sigma}_X ^{-1} {x}^{*} \tilde{\Lambda}_{\mathrm{lower}} (j) )  - \tilde{\Lambda}_{\mathrm{ratio}} ^{-1} (j) e_{\mathrm{b}}^{*2} \right) \\
&+ o\left(\frac{1}{m}\right),
\end{align*}
where $e_{\mathrm{b}} ^{*2} := (w^*)^{-1}(y^* -\pi^*)^2$ and the function $h$ is defined in Proposition \ref{proposition:ridge}.
\label{cor:gaussian_inputs_binary}
\end{corollary}

Although a typical choice of the utility function in classification is accuracy, using $U_{q,0}$ with the transformation in Corollary \ref{cor:gaussian_inputs_binary} provides sensible data values.
This is because our approach can be viewed as using the iteratively re-weighted least squares (IRLS) algorithm \citep{green1984iteratively}, a classic algorithm for finding the maximum likelihood estimator (MLE) in generalized linear models.
To be more specific, for a set of random samples $\{(X_i, Y_i)\}_{i=1} ^B$ from $P_{X,Y}$ and their working dependent variables and its corresponding weights $\{(Z_i, w_i)\}_{i=1} ^B$ based on \eqref{eqn:irls_variabls}, the IRLS estimator is defined as
\begin{align}
    \hat{\beta}_{\mathrm{IRLS}} &= (\mathbb{X}^T \mathbb{W} \mathbb{X})^{-1} \mathbb{X}^T \mathbb{W} \mathbb{Z}, \label{eqn:irls_estimator}
\end{align}
where $\mathbb{X}$ is a matrix whose $i$-th row is $X_i ^T$, $\mathbb{W}$ is a diagonal matrix whose $i$-th element is $w_i$, and similarly $\mathbb{Z}$ is a vector whose $i$-th element is $Z_i$. 
Note that the estimator \eqref{eqn:irls_estimator} is the least squares estimator with the transformation in Corollary \ref{cor:gaussian_inputs_binary}.
Hence, the DShapley captures the contribution to finding the MLE in binary classification.

\begin{algorithm}[t]
\caption{DShapley for binary classification.}
\begin{algorithmic}
\Require A datum to be valued $(x^*, y^*)$. A set of random samples $\{(X_i,Y_i)\}_{i=1} ^B$ from $P_{X,Y}$.
\Procedure{}{}
\While{until a convergent condition is met}
\State $\pi_i \leftarrow \mathrm{logit}^{-1}(X_i ^T \hat{\beta}_{\mathrm{IRLS}} )$
\State Update $w_i$ and $Z_i$ based on Equation \eqref{eqn:irls_variabls} and set $\mathbb{W}$ and $\mathbb{Z}$ 
\State $\hat{\beta}_{\mathrm{IRLS}} \leftarrow (\mathbb{X}^T \mathbb{W} \mathbb{X})^{-1} \mathbb{X}^T \mathbb{W} \mathbb{Z}$
\EndWhile
\State $\pi^* \leftarrow \mathrm{logit}^{-1}(x^{*T} \hat{\beta}_{\mathrm{IRLS}} )$
\State $z^* \leftarrow x^{*T} \hat{\beta}_{\mathrm{IRLS}}  + (y^* - \pi^*)/(\pi^*(1-\pi^*))$
\State $w^* \leftarrow \pi^*(1-\pi^*)$
\State Compute a lower bound of DShapley of $\left( (w^*)^{1/2} x^*, (w^*)^{1/2} z^* \right)$.
\EndProcedure
\end{algorithmic}
\label{alg:DSV_IRLS}
\end{algorithm}

In practice, we do not know $\pi$ nor $\mathbb{W}$. To address this issue, we first use the original IRLS algorithm; we iteratively compute \eqref{eqn:irls_variabls} and \eqref{eqn:irls_estimator} until $\hat{\beta}_{\mathrm{IRLS}}$ converges.
After convergence, we apply Corollary~\ref{cor:gaussian_inputs_binary}. A simple version of this process is described in Alg.~\ref{alg:DSV_IRLS}. A detailed version is provided in Appendix.

\section{Distributional Shapley values for non-parametric density estimation}
\label{s:dsv_density}
In this section, we study DShapley for non-parametric density estimation problems.
We let $Z$ be a random variable defined on $\mathcal{Z} \subseteq \mathbb{R}^d$ as in Sec.~\ref{s:preliminaries} and let $p(z)$ be the underlying probability density function.
We consider the kernel density estimator (KDE), a fundamental non-parametric density estimator in statistics \citep{rosenblatt1956, parzen1962}.
For a kernel function\footnote{For a non-negative function $k :\mathcal{Z} \to \mathbb{R}$, we say $k$ is a kernel if $\int k(z) dz = 1$ and $k(z)=k(-z)$ for all $z \in \mathcal{Z}$.} $k:\mathcal{Z} \to \mathbb{R}$, the KDE based on a dataset $S \subseteq \mathcal{Z}$ is denoted by $\hat{p}_{S,k} (z) = \frac{1}{|S|} \sum_{z_i \in S}k(z-z_i)$.
By convention, we assume that a kernel is bounded and parameterized by a bandwidth $h>0$, \textit{i.e.}, $k_h(\cdot) := h^{-d} k(\cdot/h)$ for a kernel $k$.
For notational convenience, we suppress the bandwidth notation and use $k$ instead of $k_h$.
For a constant $C_{\mathrm{den}} > 0$, and a density estimator $\hat{p}$, we define a utility function as $U(S, \hat{p}) = (C_{\mathrm{den}} - \int  (p(z) - \hat{p} (z) )^2 dz) \mathds{1}(|S| \geq 1)$. When the KDE is used, we set $U_k(S) := U(S, \hat{p}_{S,k})$. 
As before, changing the constant $C_{\mathrm{den}}$ simply shifts the value of all the points by the same constant; therefore we just set $C_{\mathrm{den}}$ to simplify expressions of DShapley.

Before going to the analysis, we define DShapley of a set, a natural extension of DShapley of a point, by regarding a set as a point.
More precisely, given a utility function $U$, a data distribution $P_Z$, and some $m \in \mathbb{N}$, we define DShapley of a set as follows.
\begin{align*}
    \nu(S^*; U, P_Z, m) := \mathbb{E}_{ j \sim [m]} \mathbb{E}_{ S \sim P_{Z}^{j-1}} [U(S \cup S^* )-U(S)].
\end{align*}
Similar to DShapley for a point, DShapley for a set describes the expected value of marginal contributions of set $S^*$ over random datasets $S$.
With this notion, we present DShapley for the KDE in the following theorem.
To begin, let $A(n,m) := \frac{1}{m} \sum_{j=1} ^m  \frac{ n^2 }{(j+n-1)^2}$ and $B(n,m) := \frac{1}{m} \sum_{j=2} ^m \frac{ 2n (j-1) }{(j+n-1)^2}$.

\begin{theorem}[DShapley for non-parametric density estimation]
Let $S^* \subseteq \mathcal{Z}$ be a set to be valued such that $|S^*|=n$.
Then, for some fixed constant $C_{\mathrm{den}}$, DShapley of $S^*$ with the KDE is given by
\begin{align*}
    &\nu(S^*; U_k, P_Z, m) \\
    &= - A(n,m) \int  (p(z) - \hat{p}_{S^{*}, k} (z) )^2 dz + B(n,m)g(S^*),
\end{align*}
where $g(S^*) := \int \hat{p}_{S^*,k} (z) (p(z) - \mathbb{E}[k(z-Z)]) dz.$
\label{thm:shapley_density}
\end{theorem}

\paragraph{The term $g(S^*)$} 
Suppose $p(z)$ is twice continuously differentiable, and for all $i \in [d]$, a kernel satisfies $\int z_{(i)}^2 k(z) dz < \infty$ and $\int |z_{(i)}|^3 k(z) dz < \infty$, where $z = (z_{(1)}, \dots, z_{(d)}) \in \mathbb{R}^d$. Then, the bias $(p(z)-\mathbb{E}[k(z-Z)])$ of the KDE is $O(h^2)$ and thus $g(S^*) = O(h^2)$ \citep[Equation (1.131)]{ghosh2018}. Many useful kernels such as the Gaussian kernel or any continuous kernel with bounded support  satisfy the conditions.

Theorem \ref{thm:shapley_density} shows the exact form of DShapley of a set $S^*$.
As discussed above, under the mild conditions, the second term is $O(h^2)$, so we focus on the first term.
The first term is the negative constant $-A(n,m)$ times to the integrated squared error (ISE) of $\hat{p}_{S^*,k}$.
That means, DShapley for a set $S^*$ increases as ISE decreases, and vice versa.
Note that the ISE could be interpreted as performance of $S^*$.

As for the estimation of DShapley, we use the Monte-Carlo approximation method based on Theorem \ref{thm:shapley_density}. For $m, B \in \mathbb{N}$, sets of random samples $\{\tilde{z}_1,\dots, \tilde{z}_B\}$ and $\{\tilde{z}_1 ^*, \dots, \tilde{z}_B ^* \}$ from $P_{Z}$ and $\hat{p}_{S^*, k}$, respectively, the DShapley estimator $\hat{\nu}(S^*; U_k, P_Z, m)$ is given by
\begin{align}
    &\frac{A(|S^*|,m)}{B} \sum_{i=1} ^B \left( \hat{p}_{S^*, k} (\tilde{z}_i ^*) -2 \hat{p}_{S^*, k} (\tilde{z}_i) \right) \notag \\
    &+ \frac{B(|S^*|,m)}{B} \sum_{i=1} ^B \left( \hat{p}_{S^*, k} (\tilde{z}_i) - k (\tilde{z}_i ^* - \tilde{z}_i)  \right). \label{eqn:DSV_density_estimator} 
\end{align}
We provide more details in Appendix. In the following examples, we provide more insights on DShapley with the uniform kernel.
Proofs of Examples \ref{exp:two_set} and \ref{exp:two_set_synergy} are available in Appendix.

\begin{example}[A set with two elements]
Suppose $S^* = \{z_1 ^*, z_2 ^* \}$, $p(z)=1$ for all $z \in [0,1]$ and $k (z- z_i) = \frac{1}{h} \mathds{1} ( | \frac{z-z_i}{h} | \leq \frac{1}{2} )$. 
We set a bandwidth such that $h \leq 2\min\{z_1 ^*, z_2 ^*, (1-z_1 ^*), (1-z_2 ^*) \}$.
Then, we have a closed-form expression for DShapley as follows.
\begin{align*}
    &\nu(S^*; U, P_Z, m) \\
    &= \begin{cases}
A(2,m) \left( 1 - \frac{1}{2h} \right) + C_{\mathrm{set}} & \text{if } \Delta \geq h, \\
A(2,m) \left( 1 - \frac{1}{h} + \frac{\Delta}{2h^2} \right) + C_{\mathrm{set}} & \text{if } \Delta < h,
\end{cases}
\end{align*}
where $\Delta := | z_1 ^* - z_2 ^*|$ and $C_{\mathrm{set}}$ is some explicit constant independent of $S^*$.
DShapley for a set satisfying $\Delta < h$ is less than the value of a set with $\Delta \geq h$.
In other words, if the two data points are farther than $h$, DShapley gets larger. \label{exp:two_set}
\end{example}

\begin{example}[Synergy of two elements]
We suppose the same setting with Example \ref{exp:two_set} and now investigate the case where DShapley of $S^*$ is greater than the sum of two DShapleys of the point, \textit{i.e.}, 
\begin{align}
    \nu( \{z_1 ^*, z_2 ^* \}; U, P_Z, m) \geq  \sum_{z \in \{z_1 ^*, z_2 ^* \}} \nu( z; U, P_Z, m).
    \label{ineqn:synergy}
\end{align}
We say there is a synergy of $z_1 ^*$ and $z_2 ^*$ when the inequality \eqref{ineqn:synergy} holds.
Although a similar analysis used in Example \ref{exp:two_set} gives $\nu( z_1 ^*; U, P_Z, m)$ a closed-form expression, it is difficult to know when the inequality \eqref{ineqn:synergy} holds analytically.
With empirical experiments, we show that synergy happens when $\Delta$ is bigger than some threshold, \textit{i.e.}, when the two points are not too close. 
\label{exp:two_set_synergy}
\end{example}

\section{Numerical experiments}
\label{s:experiments}
We now demonstrate the practical efficacy of the DShapley using real and synthetic datasets.
As for the proposed methods, we use Alg.~\ref{alg:gaussian_inputs_lse_simple}, Alg.~\ref{alg:DSV_IRLS}, and Equation \eqref{eqn:DSV_density_estimator} for linear regression, binary classification, and non-parametric density estimation problems, respectively.
To empirically show the general applicability of the proposed methods, we include complex nonlinear models such as convolutional neural networks (CNNs) on our image datasets. Following the common procedure in prior works, we treat the early layers of an off-the-shelf pre-trained network as fixed feature extractors and apply Shapley to the last layer (\cite{ghorbani2020,koh2017}).
Detailed information about datasets and experiment settings are provided in Appendix. Our implementation codes are available at \url{https://github.com/ykwon0407/fast_dist_shapley}.

\begin{figure}[t]
    \centering
    \includegraphics[width=0.485\columnwidth]{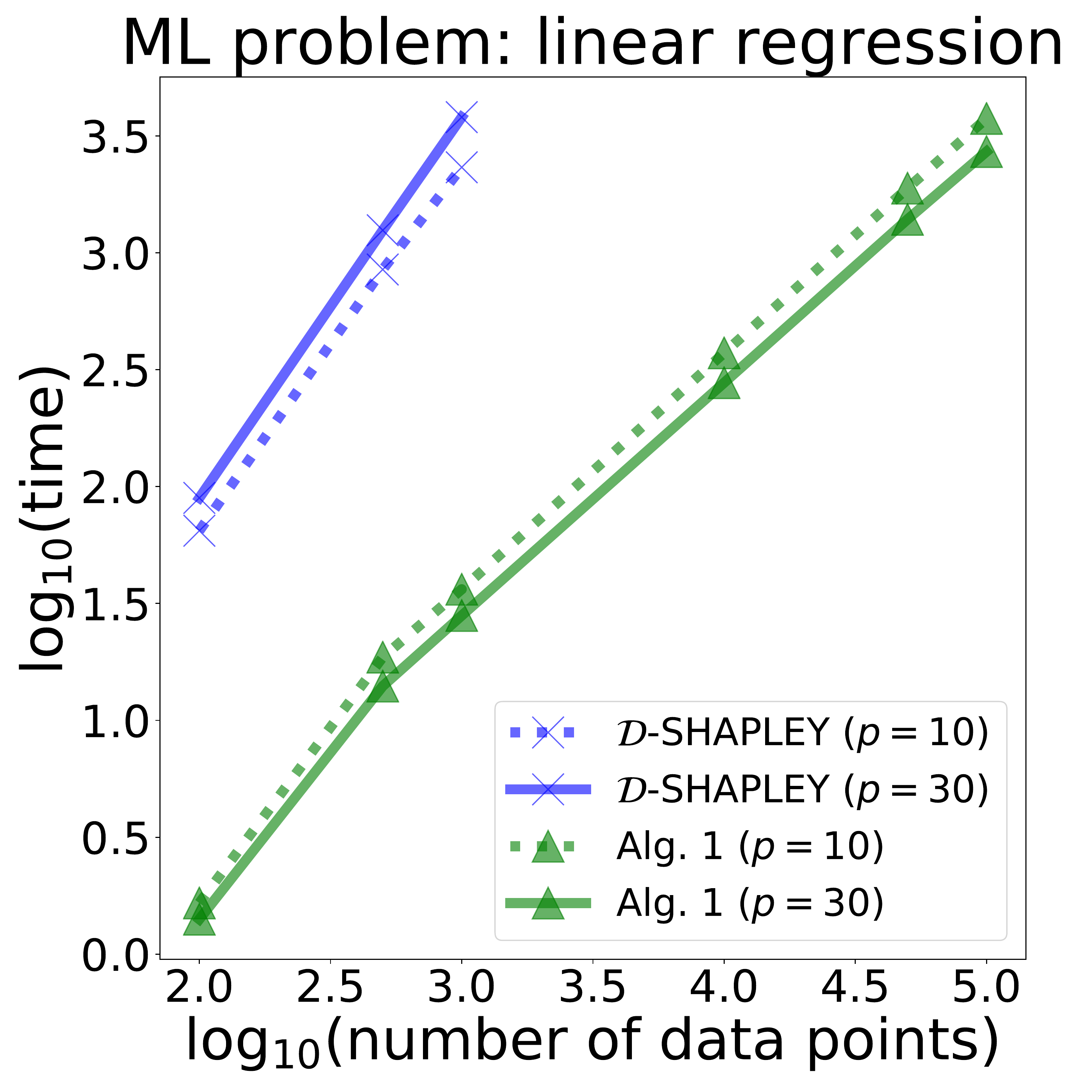}
    \includegraphics[width=0.485\columnwidth]{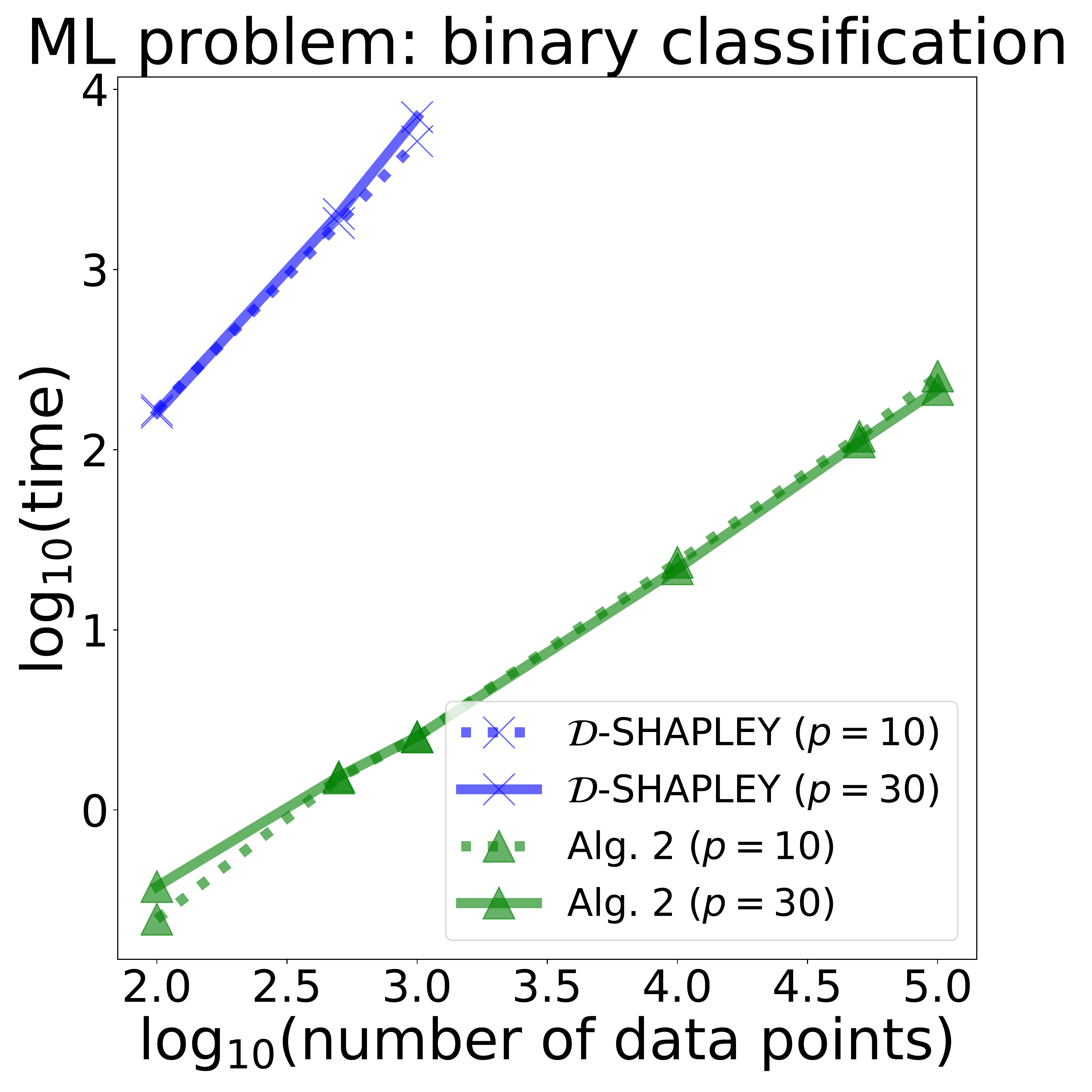}
    \caption{Computation time (in seconds) as a function of number of data points to be valued in logarithmic scale. We compare \texttt{$\mathcal{D}$-SHAPLEY} (blue) with the proposed algorithms (green) when the input dimension $p$ is 10 (dashed) or 30 (solid). Alg.~\ref{alg:gaussian_inputs_lse_simple} is used for regression and Alg.~\ref{alg:DSV_IRLS} for classification.}
    \label{fig:time_comparison}
\end{figure}

\paragraph{Comparison of the computational time}
We compare the computational time of \texttt{$\mathcal{D}$-SHAPLEY} by \citet{ghorbani2020} with the proposed methods in several ML problems.
As we mentioned in Sec.~\ref{s:dsv_gaussian_inputs}, the existing algorithm requires the utility evaluation, and thus it is anticipated to have much heavier computational costs than the proposed algorithm.
All the computation time results in this section are measured with the single Intel\textregistered Xeon\textregistered E5-2640v4 CPU processor and are an average based on 50 repetitions.

Figure~\ref{fig:time_comparison} shows the computational time of state-of-the-art \texttt{$\mathcal{D}$-SHAPLEY} and the proposed methods in various the number of data to be valued, denoted by $m$, and the dimension of input data, denoted by $p$. We consider linear regression and binary classification problems and use the synthetic Gaussian datasets.
For both ML problems, the proposed algorithm is several orders of magnitude faster than \texttt{$\mathcal{D}$-SHAPLEY}.
In particular, in case of classification, while \texttt{$\mathcal{D}$-SHAPLEY} requires 7015.7 seconds, Alg.~\ref{alg:DSV_IRLS} takes 2.6 seconds, which is 2750 times faster, when $(m,p)=(1000,30)$. Our proposed algorithms is scalable to compute the distribution Shapley values of hundreds of thousands of data points in thousands of dimensions. With $(m,p)=(5 \times 10^5, 10^3)$, Alg.~\ref{alg:gaussian_inputs_lse_simple} (Alg.~\ref{alg:DSV_IRLS}) takes around 5.3 hours (\textit{resp.} 30 minutes) to compute the DShapley values for all half million data points for linear regression (\textit{resp.} binary classification). This can be further improved with parallel computing and GPU processors.
The computation time for the binary classification problem is much smaller because we use a computationally cheap analytic lower bound. The sharpness of this lower bound is examined in the point addition experiment below.

\begin{table}[t]
    \centering
    \caption{Computation time (in seconds) of \texttt{$\mathcal{D}$-SHAPLEY} and the proposed algorithms in various ML problems and datasets.
    The number of data to be valued is fixed to 200 for all datasets.}
    \vspace{0.1in}
    \resizebox{\columnwidth}{!}{
        \begin{tabular}{lcccccccccccc}
    \toprule
ML problem & \multirow{2}{*}{Dataset}  & \multirow{2}{*}{\texttt{$\mathcal{D}$-SHAPLEY}} & \multirow{2}{*}{Proposed} \\
(Proposed method) \\
    \midrule
    \multirow[b]{2}{*}{Linear regression} & \texttt{Gaussian-R} &  229.2 & 6.9  \\
     & \texttt{abalone} & 226.3 & 5.2  \\
    \multirow[t]{2}{*}{(Alg.~\ref{alg:gaussian_inputs_lse_simple})} & \texttt{airfoil} & 280.7 & 4.6 \\
     & \texttt{whitewine} & 275.4 & 5.1 \\
    \midrule     
    \multirow[b]{2}{*}{Binary classification} & \texttt{Gaussian-C}  & 470.4 & 0.7 \\
     & \texttt{skin-nonskin} & 788.5 & 1.7 \\
    \multirow[t]{2}{*}{(Alg.~\ref{alg:DSV_IRLS})} & \texttt{CIFAR10}  & 550.7 & 3.0 \\
     & \texttt{MNIST}  & 536.7 & 5.3 \\
     \midrule     
    \multirow[b]{2}{*}{Density estimation} & \texttt{diabetes} &  3307.8 & 0.6 \\
     & \texttt{australian}  & 5219.8 & 0.3 \\
     \multirow[t]{2}{*}{(Equation \eqref{eqn:DSV_density_estimator})} & \texttt{Fashion-MNIST} & 281.7 & 23.7 \\
     & \texttt{CIFAR10} & 338.6 & 28.4 \\
    \bottomrule
    \end{tabular}
    }
    \label{tab:time_comparison_dshapley_real_datasets}
\end{table}

Lastly, Table~\ref{tab:time_comparison_dshapley_real_datasets} shows computational time in various tasks and real and synthetic datasets. We here fix the number of data to be valued as 200. This further demonstrates the computational efficiency of the proposed algorithms across all datasets.

\begin{figure*}[t]
    \centering
    \vspace{0.05in}
    \includegraphics[width=1.55in,height=1.55in]{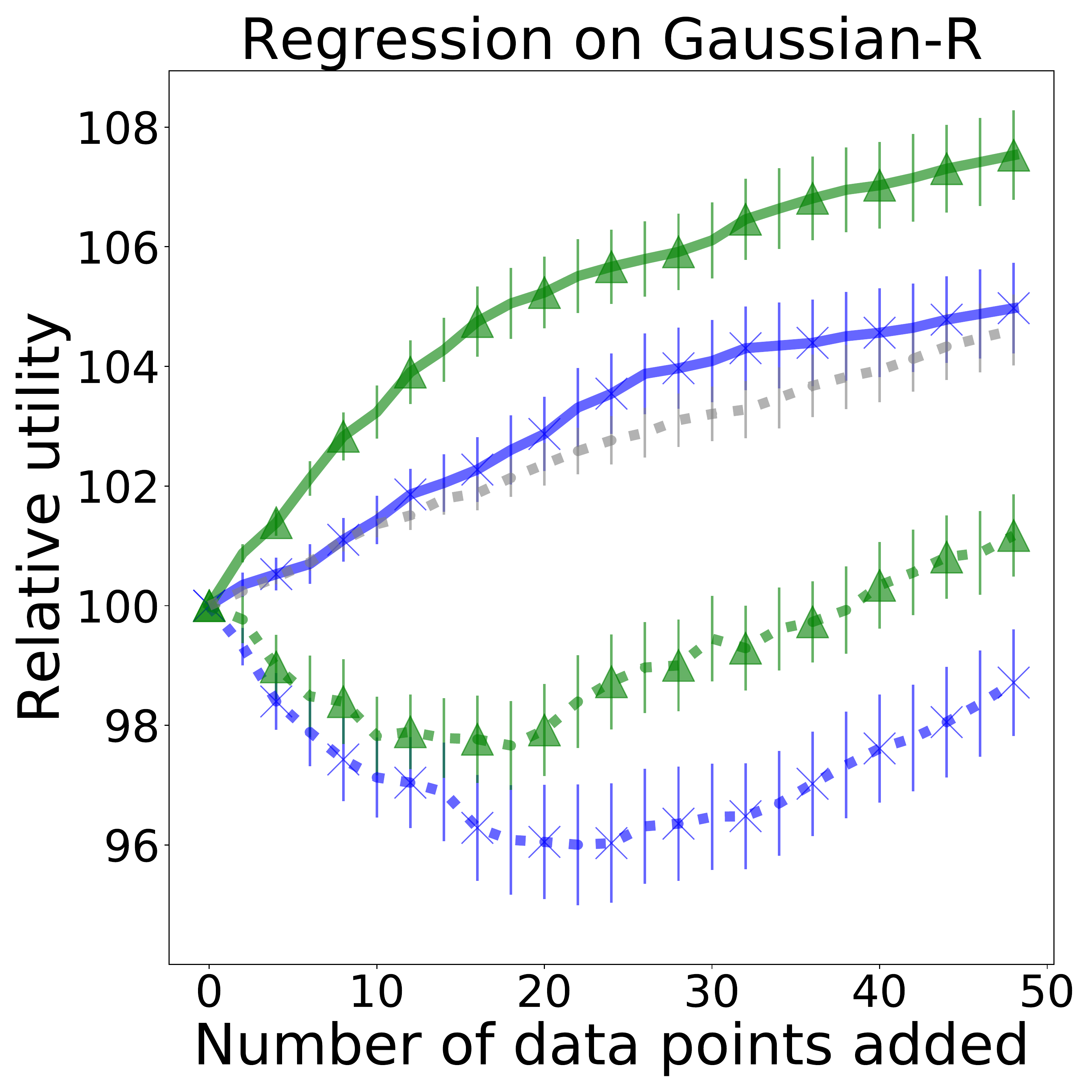}
    \includegraphics[width=1.55in,height=1.55in]{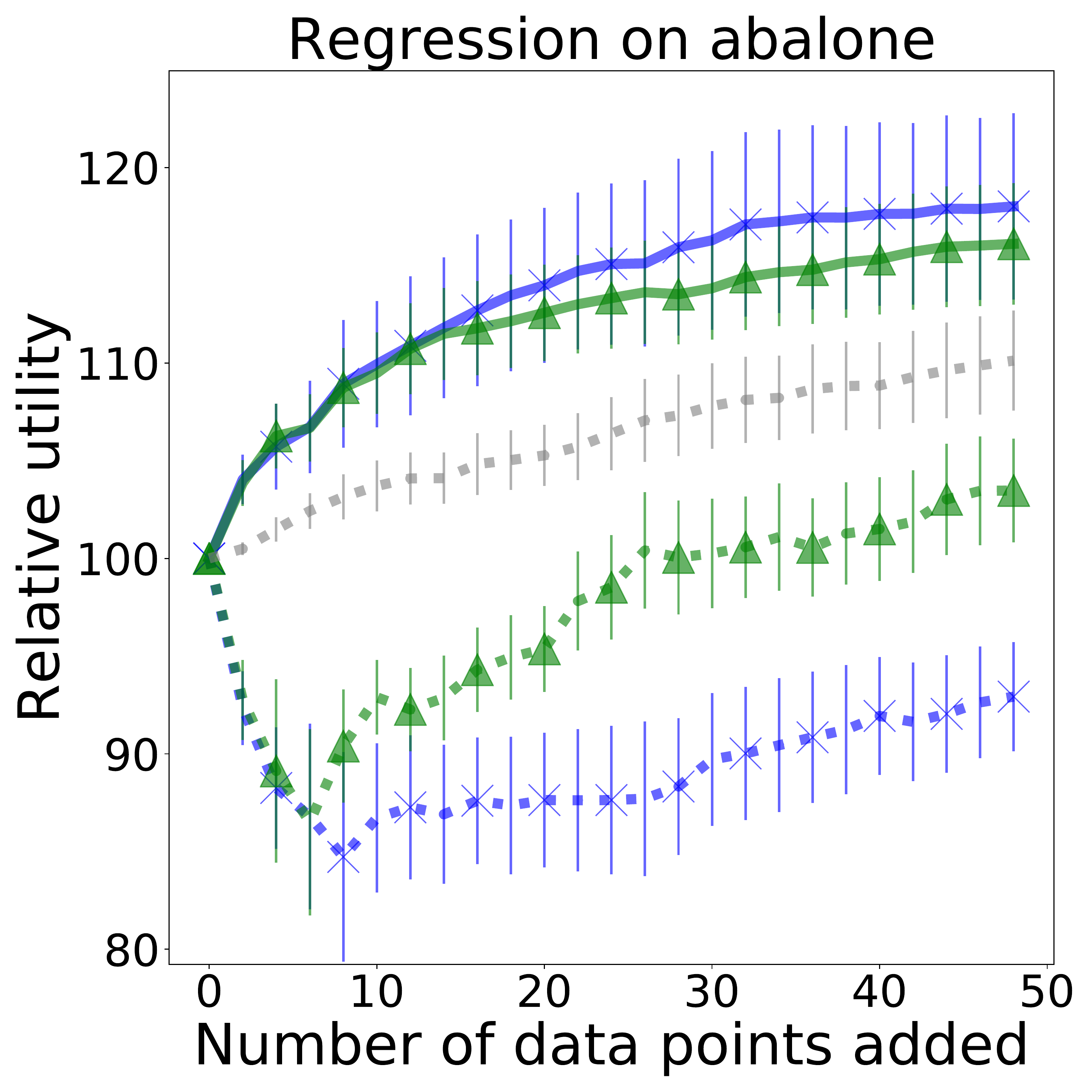}
    \includegraphics[width=1.55in,height=1.55in]{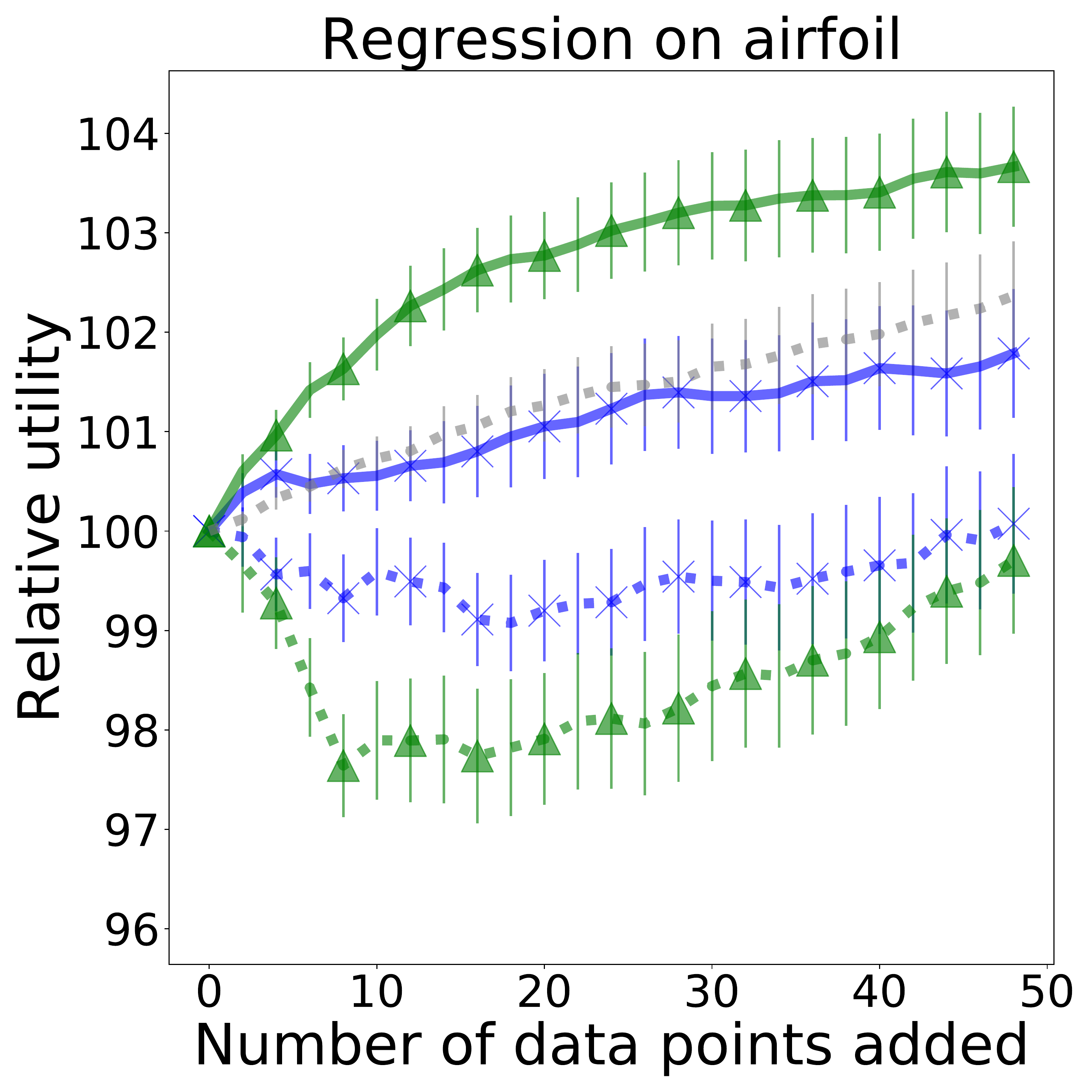}
    \includegraphics[width=1.55in,height=1.55in]{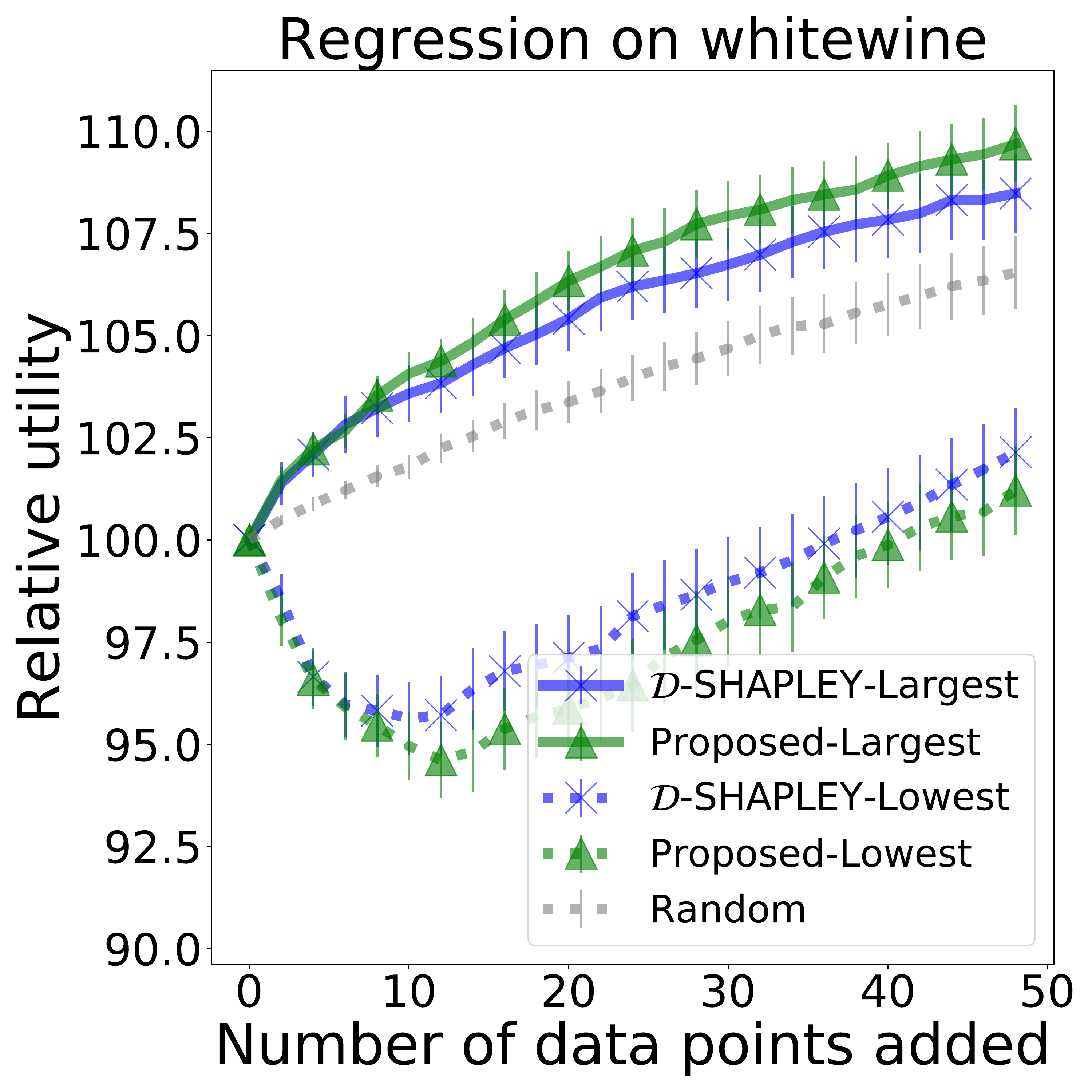}\\
    \includegraphics[width=1.55in,height=1.55in]{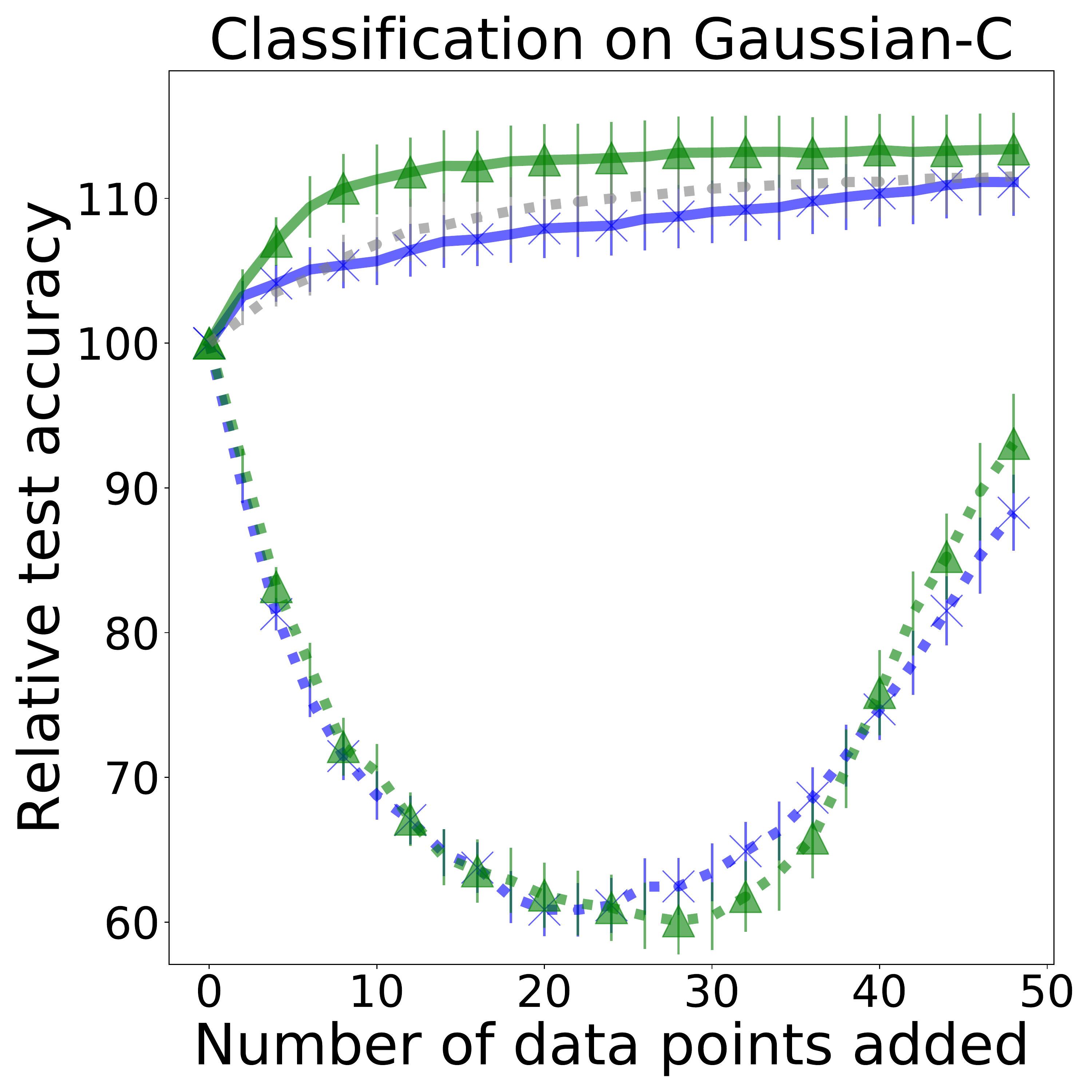}
    \includegraphics[width=1.55in,height=1.55in]{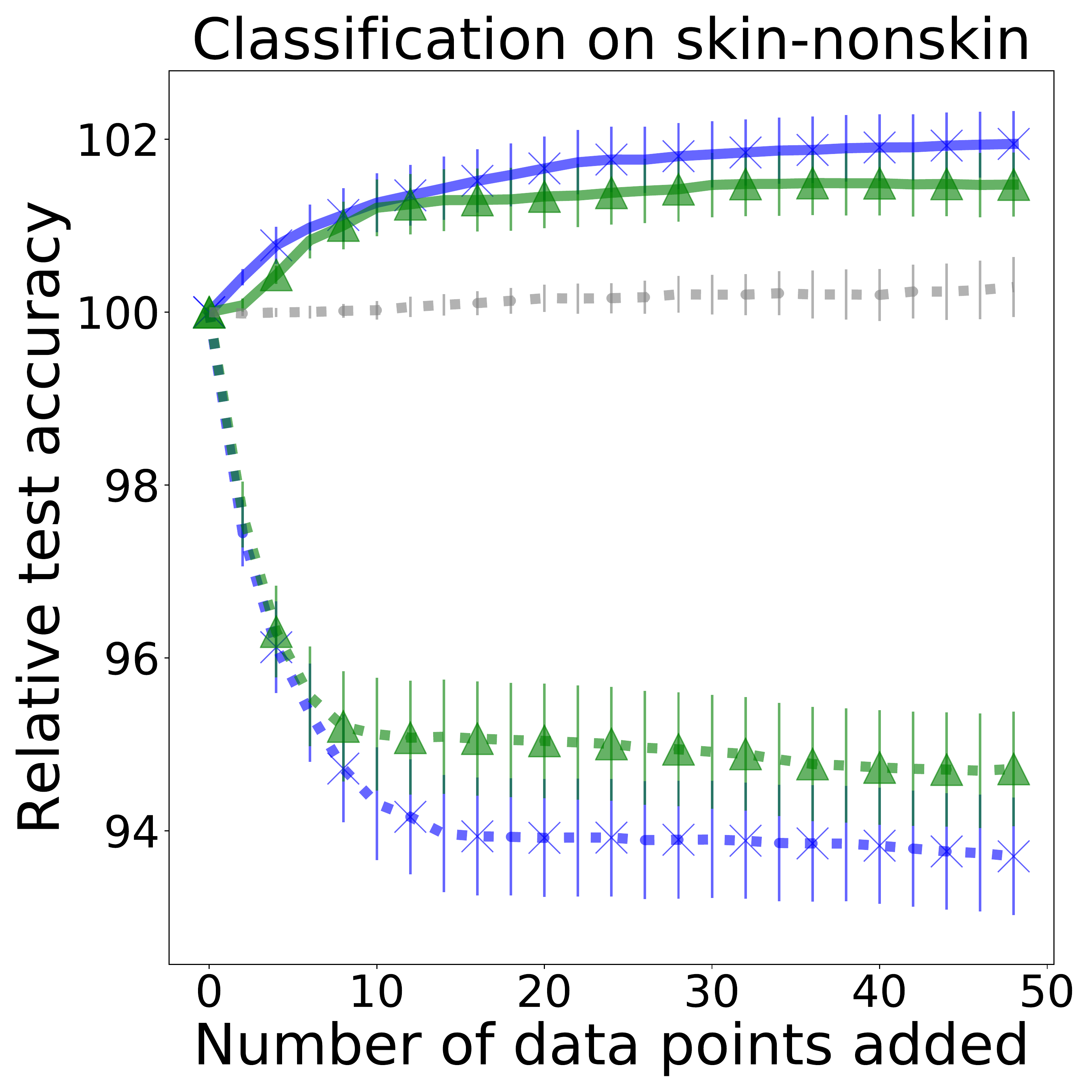}
    \includegraphics[width=1.55in,height=1.55in]{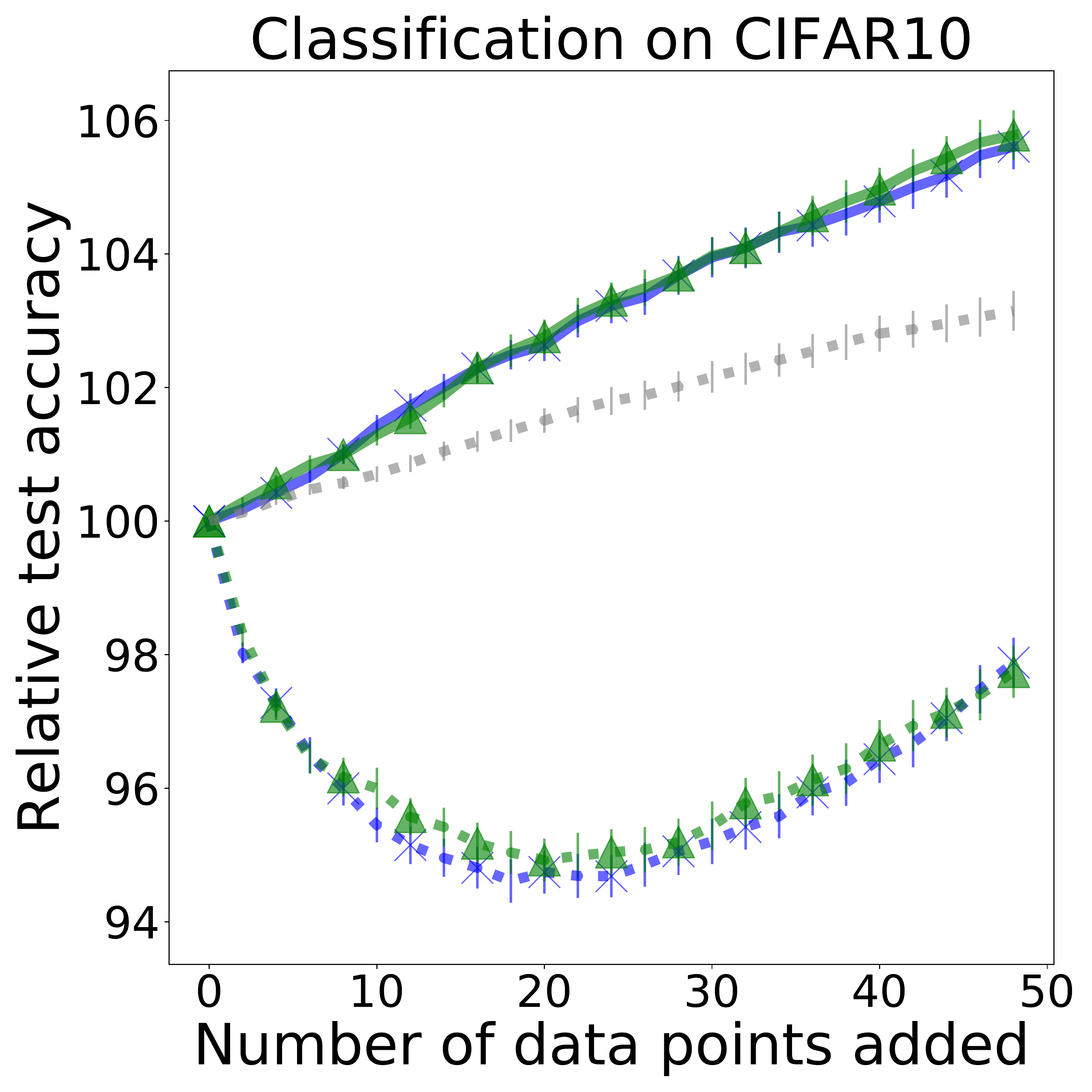}
    \includegraphics[width=1.55in,height=1.55in]{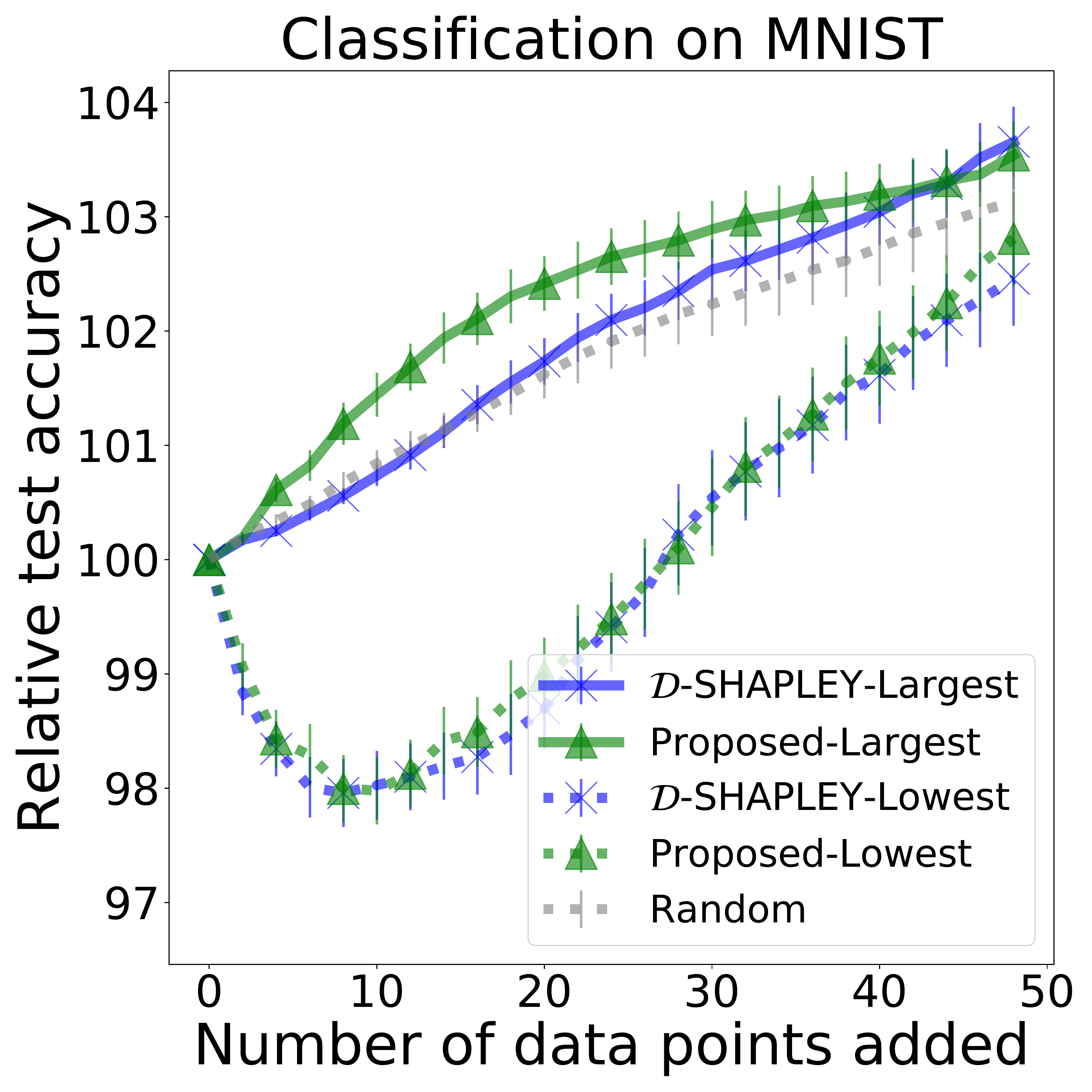}\\
    \includegraphics[width=1.55in,height=1.55in]{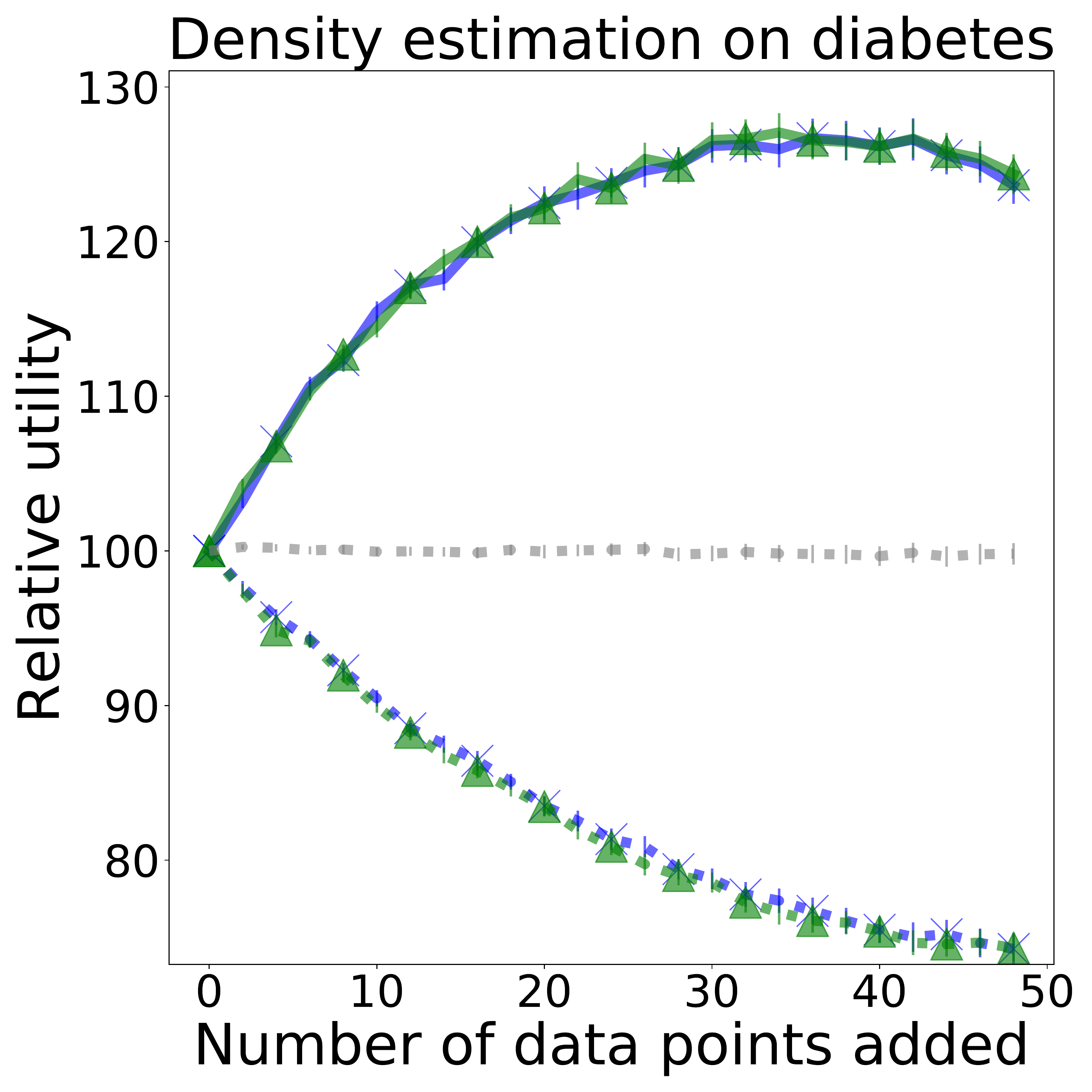}
    \includegraphics[width=1.55in,height=1.55in]{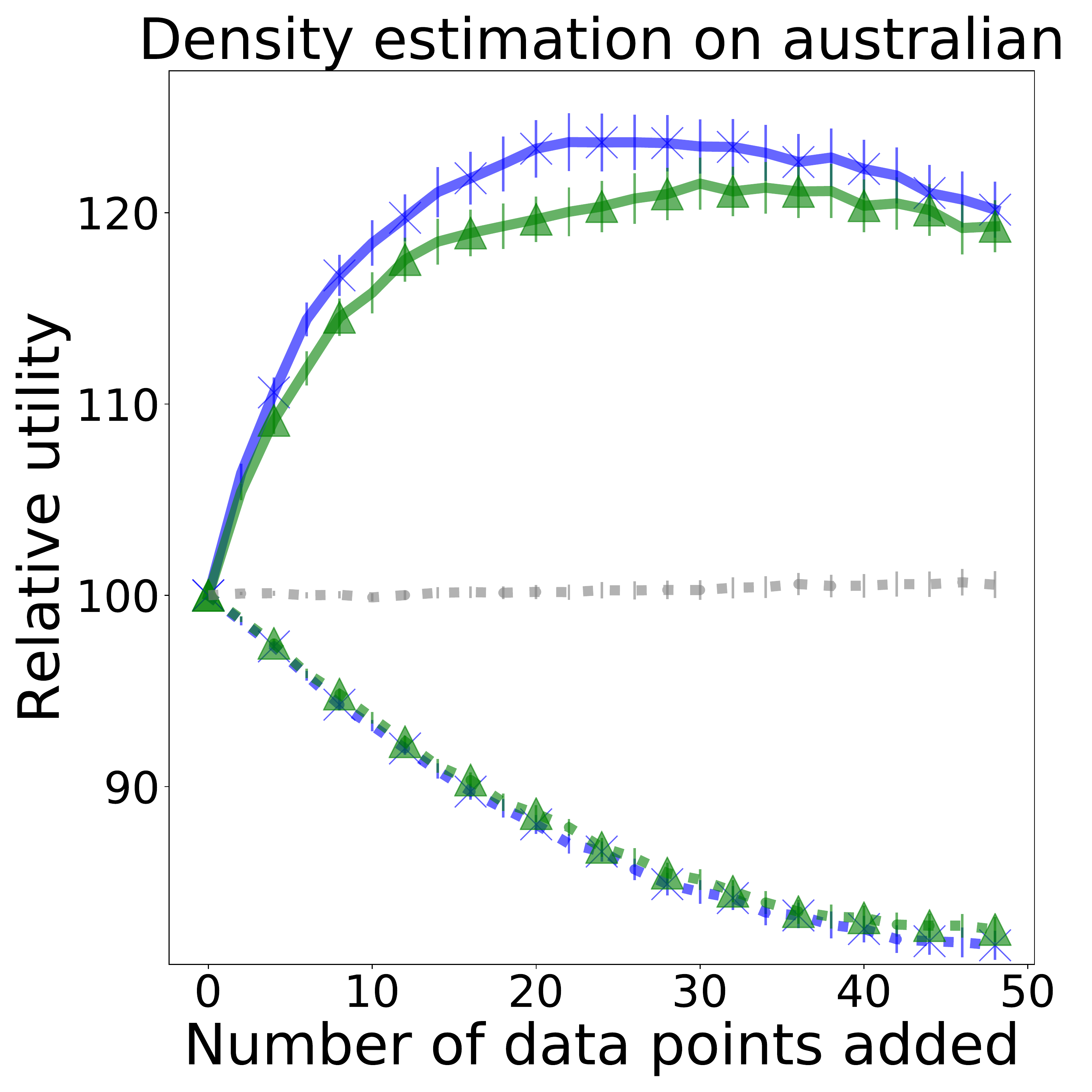}
    \includegraphics[width=1.6in,height=1.55in]{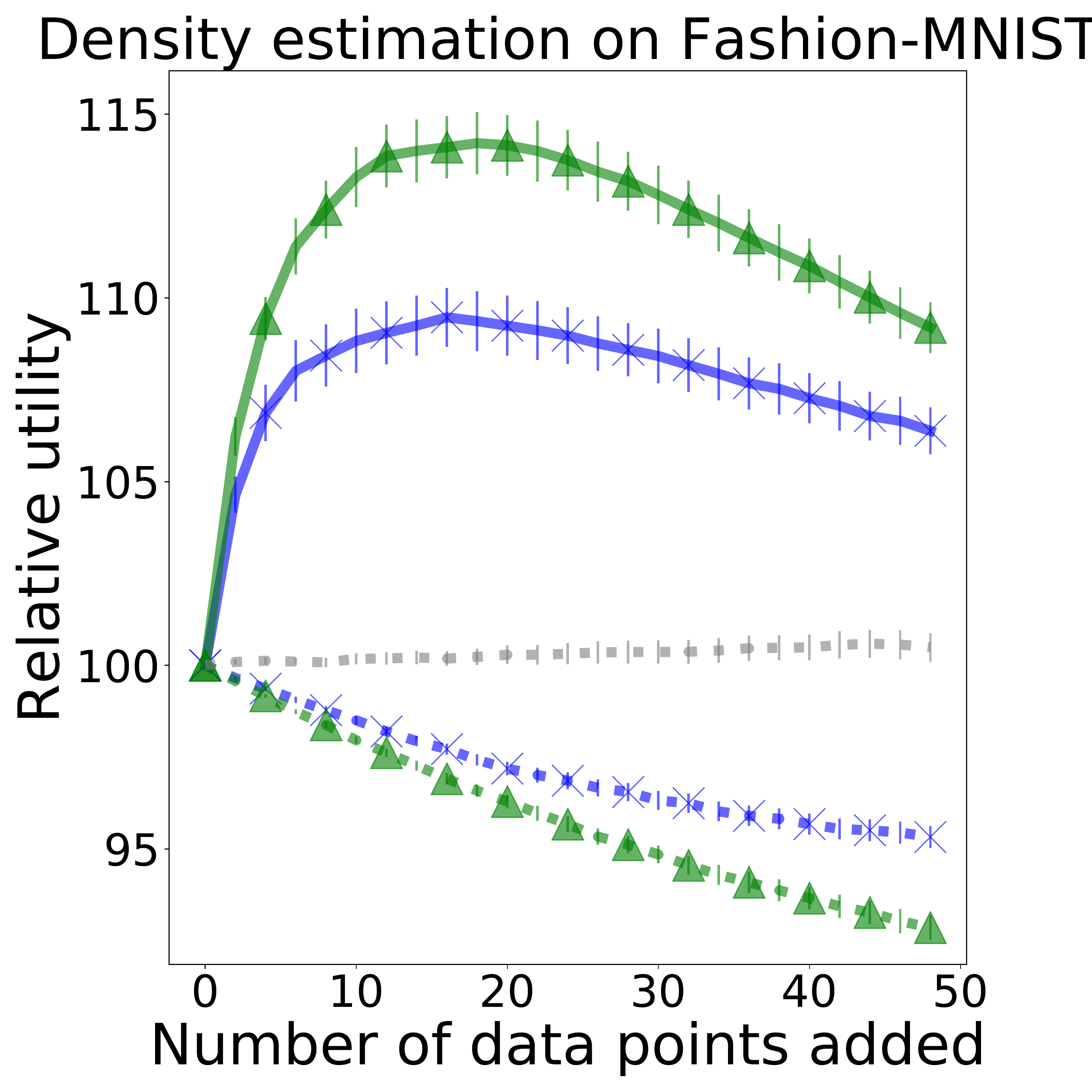}
    \includegraphics[width=1.55in,height=1.55in]{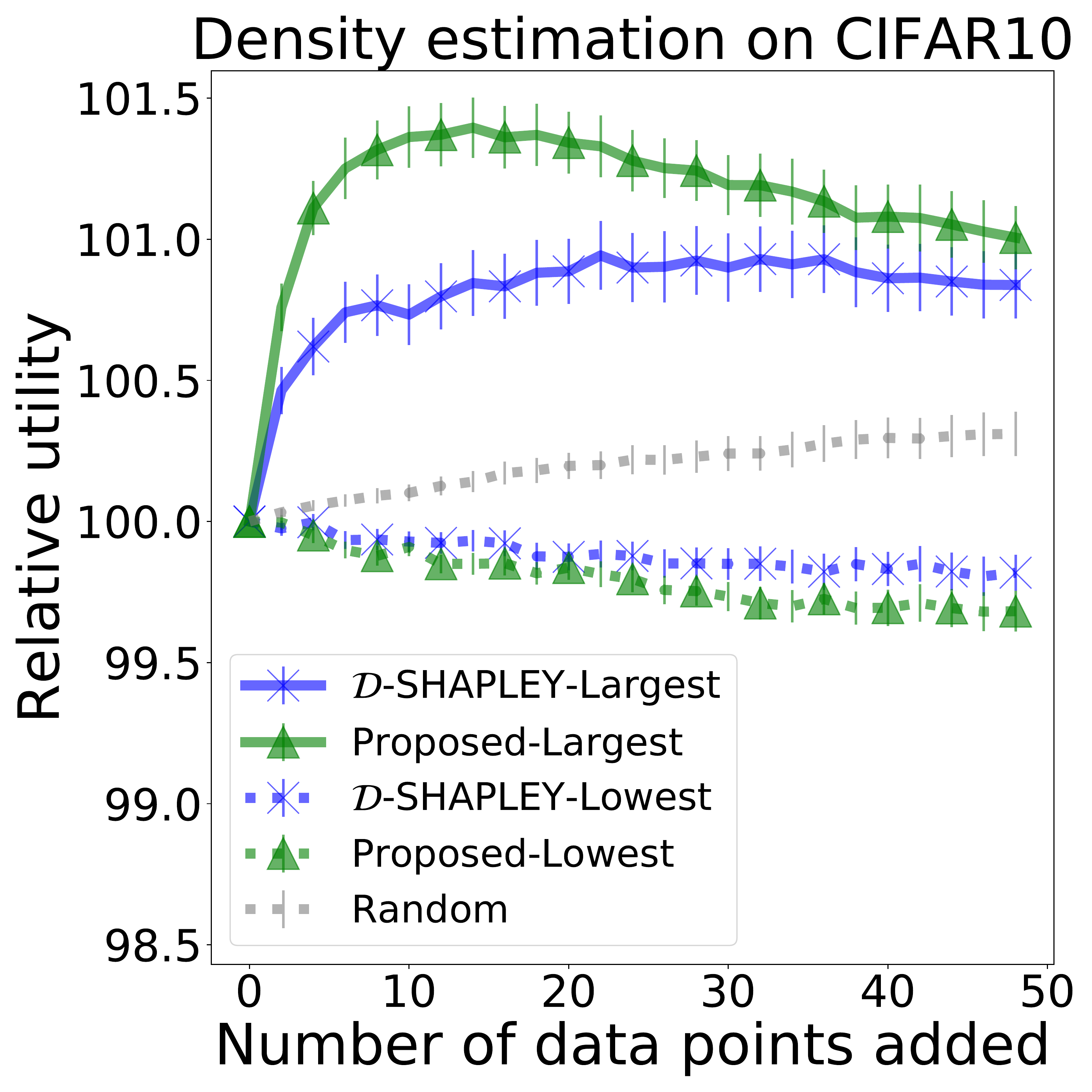}
    \caption{Relative utility and its standard error bar (in \%) as a function of the number of data added in (top) linear regression, (middle) binary classification, and (bottom) non-parametric density estimation settings. We examine the state-of-the-art \texttt{$\mathcal{D}$-SHAPLEY} (blue), random order (gray), and our proposed algorithms (green). The solid and dashed curves correspond to adding points with the largest and smallest values first, respectively. The results are based on 50 repetitions.}
    \label{fig:relative_utility_various_cases}
\end{figure*}

\paragraph{Point addition experiment}
We demonstrate the empirical effectiveness of our DShapley approach by running point addition experiments, proposed by \citet{ghorbani2020}. Given a model and a dataset to be valued, we recursively add points given order (e.g. from largest to lowest values), retrain the model with the remained dataset, and observe how the utility changes on the held-out test dataset.
We compare the three methods: (i) the random order, (ii) the order based on \texttt{$\mathcal{D}$-SHAPLEY}, and (iii) the order based on the proposed methods. 
For \texttt{$\mathcal{D}$-SHAPLEY} and the proposed algorithms, we consider the two different types of orders. One is from largest to lowest, denoted by `Largest', and the other is from lowest to largest, denoted by `Lowest'. In the case of the `Largest' order, they are expected to capture points that help improve performance and show a steeper performance boost than the random order. Similarly, in the case of the `Lowest' order, DShapley is expected to capture outliers first and cause performance degradation. 
As before, we use Alg.~\ref{alg:gaussian_inputs_lse_simple}, Alg.~\ref{alg:DSV_IRLS}, and Equation \eqref{eqn:DSV_density_estimator} for linear regression, binary classification, and non-parametric density estimation problems, respectively.

Figure~\ref{fig:relative_utility_various_cases} shows point addition experiments in linear regression, binary classification, and non-parametric density estimation problems.
As we anticipated, the proposed methods and \texttt{$\mathcal{D}$-SHAPLEY} show reasonable curves; adding data with larger DShapley leads to a greater performance increase than random addition. Also, adding data with lower DShapley causes performance degradation.
Moreover, the proposed methods perform similar to or sometimes better than \texttt{$\mathcal{D}$-SHAPLEY}. 
This phenomenon is because \texttt{$\mathcal{D}$-SHAPLEY} repeatedly evaluates the utility function on a random set, and the utility is unstable when the size of the set is small and it affects instability of DShapley estimation.
In contrast, the proposed methods can avoid such instability and provide reasonable values.

\section{Concluding remarks}
In this work, we derive the first computationally tractable expressions for DShapley for the linear regression, binary classification, and non-parametric density estimation problems. 
The proposed forms provide new mathematical insights, and lead to  efficient algorithms which we demonstrated on large datasets (e.g. $10^5$ data points in $\mathbb{R}^{1000}$) and models such as CNN. We validate our results on several commonly used datasets.

\section*{Acknowledgements}
Y.K. and J.Z. are supported by NSF CCF 1763191, NSF CAREER 1942926, NIH P30AG059307, NIH U01MH098953 and grants from the Silicon Valley Foundation and the Chan-Zuckerberg Initiative. M.A.R. is supported by Stanford University and NIH 5U01HG009080. Y.K. is also partially supported by NIH R01HG010140. We thank anonymous reviewers for helpful comments.
\bibliographystyle{apalike}
\bibliography{ref}

\appendix

\onecolumn
\aistatstitle{Appendix: Efficient Computation and Analysis of Distributional Shapley Values}

\section{Proofs}
\label{s:proofs}
\subsection{Proof of Proposition \ref{proposition:ridge}}
\begin{proof}[Proof of Proposition \ref{proposition:ridge}]
To this end, we fix $S$ and $\gamma$.
The ridge estimator based on $S$ and $S \cup \{ (x^*, y^*) \}$ are given by
\begin{align*}
    \hat{\beta}_{S,\gamma}  = A_{S, \gamma} ^{-1} X_S ^T Y_S,
\end{align*}
and 
\begin{align*}
\hat{\beta}_{S \cup \{ (x^*, y^*) \}, \gamma}  = (X_{S \cup \{ (x^*, y^*) \}} ^T X_{S \cup \{ (x^*, y^*) \}} + \gamma I_p )^{-1} X_{S \cup \{ (x^*, y^*) \}} ^T  Y_{S \cup \{ (x^*, y^*) \}},
\end{align*}
respectively.
By Sherman-Morrison formula, 
\begin{align*}
    (x^* x^{*T} + A_{S,\gamma} )^{-1} = A_{S,\gamma} ^{-1} - \frac{ A_{S,\gamma} ^{-1} x^* x^{*T} A_{S,\gamma} ^{-1} }{1+ x^{*T} A_{S,\gamma} ^{-1} x^* },
\end{align*}
and
\begin{align*}
\hat{\beta}_{S \cup \{ (x^*, y^*) \}, \gamma}  &= \hat{\beta}_{S,\gamma}  + A_{S,\gamma} ^{-1} x^{*} y^* - \frac{ A_{S,\gamma} ^{-1} x^* x^{*T} \hat{\beta}_{S,\gamma}  }{1+ x^{*T} A_{S,\gamma} ^{-1} x^* } - \frac{ A_{S,\gamma} ^{-1} x^* x^{*T} A_{S,\gamma} ^{-1} x^{*} y^* }{1+ x^{*T} A_{S,\gamma} ^{-1} x^* } \\
&= \hat{\beta}_{S,\gamma}  + \underbrace{\frac{ A_{S,\gamma} ^{-1} x^{*} (y^*-x^{*T} \hat{\beta}_{S,\gamma} )}{1+x^{*T} A_{S,\gamma} ^{-1} x^*}}_{=:f_{\gamma} (X_S)}.
\end{align*}

Since $U_{q,\gamma} (S) = (C_{\mathrm{lin}} - \int  (y - x^T \hat{\beta}_{S,\gamma}  )^2 dP_{X,Y}(x,y) ) \mathds{1}(|S| \geq q) = (C_{\mathrm{lin}}  -\sigma^2 - (\hat{\beta}_{S,\gamma}  -\beta)^T \Sigma_X (\hat{\beta}_{S,\gamma} -\beta))\mathds{1}(|S| \geq q)$, for $j-1 \geq q$, we have
\begin{align*}
    &\mathbb{E}_{S \sim P_{X,Y}^{j-1}} [U_{q, \gamma} (S \cup \{ (x^*, y^*) \})] \\
    &= C_{\mathrm{lin}}  -\sigma^2 - \mathbb{E}_{S \sim P_{X,Y}^{j-1}} [(\hat{\beta}_{S \cup \{ (x^*, y^*) \}, \gamma}  -\beta)^T \Sigma_X (\hat{\beta}_{S \cup \{ (x^*, y^*) \}, \gamma}  -\beta)] \\
    &= \mathbb{E}_{S \sim P_{X,Y}^{j-1}} [U_{q, \gamma} (S)] - \mathbb{E}_{S \sim P_{X,Y}^{j-1}} [ f_{\gamma} (X_S)^T \Sigma_X f_{\gamma} (X_S)] -2\mathbb{E}_{S \sim P_{X,Y}^{j-1}} [ f_{\gamma} (X_S)^T \Sigma_X (\hat{\beta}_{S,\gamma} -\beta) ] .
\end{align*}
Therefore, 
\begin{align*}
&\mathbb{E}_{S \sim P_{X,Y}^{j-1}} [U_{q, \gamma}  (S \cup \{ (x^*, y^*) \}) - U_{q, \gamma} (S)] \\
&= -( \mathbb{E}_{S \sim P_{X,Y}^{j-1}} [ f_{\gamma} (X_S)^T \Sigma_X f_{\gamma} (X_S)] +2 \mathbb{E}_{S \sim P_{X,Y}^{j-1}} [ f_{\gamma} (X_S)^T \Sigma_X (\hat{\beta}_{S, \gamma} -\beta) ]),
\end{align*}
and thus for $q \geq p+1$, DShapley is
\begin{align*}
&\nu((x^*, y^*); U_{q, \gamma}, P_{X,Y},m) \\
&= \frac{1}{m} \sum_{j=1} ^m \mathbb{E}_{S \sim P_{X,Y}^{j-1}} [ U_{q, \gamma}  (S \cup \{ (x^*, y^*) \})-U_{q, \gamma}  (S)]  \\
&= (C_{\mathrm{lin}}  -\sigma^2 - \mathbb{E}_{S \sim P_{X,Y}^{q-1}} [(\hat{\beta}_{S, \gamma}  -\beta)^T \Sigma_X (\hat{\beta}_{S, \gamma} -\beta)]) \\
&-\frac{1}{m} \sum_{j=q} ^m \left( \mathbb{E}_{S \sim P_{X,Y}^{j-1}} [ f_{\gamma} (X_S)^T \Sigma_X f_{\gamma} (X_S)] +2 \mathbb{E}_{S  \sim P_{X,Y}^{j-1}} [ f_{\gamma} (X_S)^T \Sigma_X (\hat{\beta}_{S, \gamma} -\beta) ] \right). 
\end{align*}

[Step 1] Computation of $\mathbb{E} [ f_{\gamma} (X_S)^T \Sigma_X f_{\gamma} (X_S) \mid X_S]$.\\
We set $e_{S, \gamma} ^* = y^* - x^{*T} \mathbb{E} [\hat{\beta}_{S, \gamma}  \mid X_S] = y^* - x^{*T} A_{S, \gamma} ^{-1} (X_S ^T X_S) \beta $, then
\begin{align*}
\mathbb{E}[ f_{\gamma} (X_S) \mid X_S ] = \frac{ A_{S, \gamma} ^{-1} x^{*} e_{S, \gamma} ^* }{1+x^{*T} A_{S, \gamma} ^{-1} x^*},
\end{align*}
and $\mathrm{Cov}[\hat{\beta}_{S, \gamma} \mid X_S ] = A_{S, \gamma} ^{-1} (X_S ^T X_S) A_{S, \gamma} ^{-1} \sigma^2 = A_{S, \gamma} ^{-1} (A_{S, \gamma} -\gamma I_p )A_{S, \gamma} ^{-1} \sigma^2  = A_{S, \gamma} ^{-1} (I_p -\gamma A_{S, \gamma} ^{-1}) \sigma^2 = (A_{S, \gamma} ^{-1} - \gamma A_{S, \gamma} ^{-2}) \sigma^2 =: M_{S, \gamma} \sigma^2$ gives

\begin{align*}
\mathrm{Cov}[f_{\gamma} (X_S) \mid X_S ] = \frac{ A_{S, \gamma} ^{-1} x^{*} x^{*T} M_{S, \gamma} x^{*} x^{*T} A_{S, \gamma} ^{-1} }{(1+x^{*T} A_{S, \gamma} ^{-1} x^*)^2} \sigma^2.
\end{align*}
Thus, 
\begin{align*}
    \mathbb{E} [ f_{\gamma} (X_S)^T \Sigma_X f_{\gamma} (X_S) \mid X_S ] &= \frac{x^{*T} A_{S, \gamma} ^{-1} \Sigma_X A_{S, \gamma} ^{-1} x^{*} }{(1+x^{*T} A_{S, \gamma} ^{-1} x^*)^2} e_{S, \gamma} ^{*2} + \frac{x^{*T} A_{S, \gamma} ^{-1} \Sigma_X A_{S, \gamma} ^{-1} x^{*} }{(1+x^{*T} A_{S, \gamma} ^{-1} x^*)^2} x^{*T} M_{S, \gamma} x^{*} \sigma^{2}.
\end{align*}
Since
\begin{align*}
    e_{S, \gamma} ^* = e^* +  x^{*T}(\beta - A_{S, \gamma} ^{-1} (X_S ^T X_S) \beta ) = e^* + \gamma x^{*T} A_{S, \gamma} ^{-1} \beta,
\end{align*}
and $M_{S, \gamma} = A_{S, \gamma} ^{-1}-\gamma A_{S, \gamma} ^{-2}$, we have
\begin{align*}
    \mathbb{E} [ f_{\gamma} (X_S)^T \Sigma_X f_{\gamma} (X_S) \mid X_S ] &= \frac{x^{*T} A_{S, \gamma} ^{-1} \Sigma_X A_{S, \gamma} ^{-1} x^{*} }{(1+x^{*T} A_{S, \gamma} ^{-1} x^*)^2} ( \frac{e ^{*2}}{\sigma^{2}} +  x^{*T} A_{S, \gamma} ^{-1} x^{*} ) \sigma^{2} + h_1(\gamma),
\end{align*}
where $h_1(\gamma)$ is some explicit term such that $\lim_{\gamma \to 0+}  h_1 (\gamma) / (\gamma \log(\gamma))$ and $h_1(0)=0$.

[Step 2] Computation of $\mathbb{E} [ f_{\gamma} (X_S)^T \Sigma_X (\hat{\beta}_{S, \gamma} -\beta) \mid X_S ]$.

\begin{align*}
    &\mathbb{E} [ f_{\gamma} (X_S)^T \Sigma_X (\hat{\beta}_{S, \gamma} -\beta) \mid X_S ]\\
    &= \mathbb{E} [ \frac{  (y^*-x^{*T} \hat{\beta}_{S, \gamma} ) x^{*T}  A_{S, \gamma} ^{-1} \Sigma_X (\hat{\beta}_{S, \gamma} -\beta) }{ 1+x^{*T} A_{S, \gamma} ^{-1} x^* } \mid X_S ] \\
    &= -\gamma \frac{ e^{*} x^{*T}  A_{S, \gamma} ^{-1} \Sigma_X A_{S, \gamma} ^{-1} \beta }{ 1+x^{*T} A_{S, \gamma} ^{-1} x^*} - \mathbb{E}_{S} [ \frac{  (\hat{\beta}_{S, \gamma} -\beta )^T x^{*} x^{*T}  A_{S, \gamma} ^{-1} \Sigma_X (\hat{\beta}_{S, \gamma} -\beta) }{ 1+x^{*T} A_{S, \gamma} ^{-1} x^* } \mid X_S ] \\
    &= -\gamma \frac{ e^{*} x^{*T}  A_{S, \gamma} ^{-1} \Sigma_X A_{S, \gamma} ^{-1} \beta  }{ 1+x^{*T} A_{S, \gamma} ^{-1} x^*} - \gamma^{2} \frac{ \beta^T A_{S, \gamma} ^{-1} x^{*} x^{*T} A_{S, \gamma} ^{-1} \Sigma_X A_{S, \gamma} ^{-1} \beta }{ 1+x^{*T} A_{S, \gamma} ^{-1} x^* } - \frac{ x^{*T} A_{S, \gamma} ^{-1} \Sigma_X M_{S, \gamma} x^{*} }{ 1+x^{*T} A_{S, \gamma} ^{-1} x^* } \sigma^2 \\
    &= - \frac{ x^{*T} A_{S, \gamma} ^{-1} \Sigma_X A_{S, \gamma} ^{-1} x^{*} }{ 1+x^{*T} A_{S, \gamma} ^{-1} x^* } \sigma^2 + h_2(\gamma).
\end{align*}
where $h_2(\gamma)$ is some explicit term such that $\lim_{\gamma \to 0+} h_2(\gamma) / (\gamma \log(\gamma))=0$ and $h_2(0)=0$.

Hence, by setting $C_{\mathrm{lin}} = \sigma^2 + \mathbb{E}_{S \sim P_{X,Y}^{q-1}} [(\hat{\beta}_{S, \gamma} -\beta)^T \Sigma_X (\hat{\beta}_{S, \gamma} -\beta)]$, we have
\begin{align}
& \nu((x^*, y^*); U_{q, \gamma} , P_{X,Y},m) \notag \\
&= C_{\mathrm{lin}} - \sigma^2 - \mathbb{E}_{S \sim P_{X,Y}^{q-1}} [(\hat{\beta}_{S, \gamma} -\beta)^T \Sigma_X (\hat{\beta}_{S, \gamma} -\beta)] \notag \\
&- \frac{1}{m} \sum_{j=q} ^m  \mathbb{E}_{X_S \sim P_X ^{j-1}} \Bigg[ \frac{x^{*T} A_{S, \gamma} ^{-1} \Sigma_X A_{S, \gamma} ^{-1} x^{*} }{(1+x^{*T} A_{S, \gamma} ^{-1} x^*)^2} ( e ^{*2}- (2 + x^{*T} A_{S, \gamma} ^{-1} x^{*})\sigma ^{2} ) \Bigg] + h(\gamma), \label{eqn:dsv_constant_form} \\
&= - \frac{1}{m} \sum_{j=q} ^m  \mathbb{E}_{X_S \sim P_X ^{j-1}} \Bigg[ \frac{x^{*T} A_{S, \gamma} ^{-1} \Sigma_X A_{S, \gamma} ^{-1} x^{*} }{(1+x^{*T} A_{S, \gamma} ^{-1} x^*)^2} ( e ^{*2}- (2 + x^{*T} A_{S, \gamma} ^{-1} x^{*})\sigma ^{2} ) \Bigg] + h(\gamma), \notag \\
&= \frac{1}{m} \sum_{j=q} ^m  \mathbb{E}_{X_S \sim P_X ^{j-1}} \Bigg[ \frac{x^{*T} A_{S, \gamma} ^{-1} \Sigma_X A_{S, \gamma} ^{-1} x^{*} }{(1+x^{*T} A_{S, \gamma} ^{-1} x^*)^2} ( (2 + x^{*T} A_{S, \gamma} ^{-1} x^{*})\sigma ^{2} - e ^{*2} ) \Bigg] + h(\gamma), \notag
\end{align}
for some $h(\gamma)$ such that $\lim_{\gamma \to 0+} h(\gamma) / (\gamma \log(\gamma)) =0$ and $h(0)=0$.
\end{proof}

\subsection{Proof of Theorem \ref{thm:gaussian_inputs_lse}}

\begin{proof}[Proof of Theorem \ref{thm:gaussian_inputs_lse}]
By plugging $\gamma=0$ into Equation \eqref{eqn:dsv_constant_form}, for $q \geq p+3$, DShapley is given by
\begin{align}
&\nu( (x^*,y^*); U_{q,0}, P_{X,Y}, m) \notag \\
&= C_{\mathrm{lin}} - \sigma^2 - \mathbb{E}_{S \sim P_{X,Y}^{q-1}} [(\hat{\beta}_{S, \gamma} -\beta)^T \Sigma_X (\hat{\beta}_{S, \gamma} -\beta)] \notag \\
&+ \frac{\sigma ^{2}}{m}  \sum_{j=q} ^m \Bigg( (1 - \frac{e^{*2}}{\sigma^{2}} ) \mathbb{E}_{X_S \sim P_X ^{j-1}} \Bigg[ \frac{\tilde{x}^{*T} (\tilde{X}_S ^T \tilde{X}_S)^{-2} \tilde{x}^{*} }{(1+\tilde{x}^{*T} (\tilde{X}_S ^T \tilde{X}_S)^{-1} \tilde{x}^*)^2} \Bigg] +  \mathbb{E}_{X_S \sim P_X ^{j-1}} \Bigg[ \frac{\tilde{x}^{*T} (\tilde{X}_S ^T \tilde{X}_S)^{-2} \tilde{x}^{*}}{1+\tilde{x}^{*T} (\tilde{X}_S ^T \tilde{X}_S)^{-1} \tilde{x}^*} \Bigg] \Bigg),
\label{eq:shapley_lse_reformulation}
\end{align}
where $\tilde{X}_S =  X_S \Sigma_X ^{-1/2}$ and $\tilde{x}^* = \Sigma_X ^{-1/2} x^*$, \textit{i.e.}, a normalized version.
Note that $(\tilde{X}_S ^T \tilde{X}_S)^{-1}$ follows an inverse-Wishart distribution and its mean is $I_p/(q-1-p-1)$.
Therefore, 
\begin{align*}
    -\frac{\sigma^2}{m} \mathrm{tr}(\mathbb{E}_{X_S \sim P_X ^{q-1}} [(\tilde{X}_S ^T \tilde{X}_S)^{-1} ] )  = -\frac{\sigma^2}{m} \frac{p}{q-p-2}.
\end{align*}
Now it is enough to compute the following expectations:
\begin{align*}
    \mathbb{E}_{X_S \sim P_X ^{j-1}}[\frac{\tilde{x}^{*T} (\tilde{X}_S ^T \tilde{X}_S)^{-2} \tilde{x}^{*} }{(1+\tilde{x}^{*T} (\tilde{X}_S ^T \tilde{X}_S)^{-1} \tilde{x}^*)^2}] \quad \mathrm{ and } \quad \mathbb{E}_{X_S \sim P_X ^{j-1}}[\frac{\tilde{x}^{*T} (\tilde{X}_S ^T \tilde{X}_S)^{-2} \tilde{x}^{*}}{1+\tilde{x}^{*T} (\tilde{X}_S ^T \tilde{X}_S)^{-1} \tilde{x}^*}].
\end{align*}

[Step 1]
For any $p \times p$ orthogonal matrix $\Gamma$, we have $\Gamma (\tilde{X}_S ^T \tilde{X}_S ) \Gamma^T \sim W_p (|S|, I_p)$ due to $\tilde{X}_S ^T \tilde{X}_S \sim W_p (|S|, I_p)$.
We choose an orthogonal matrix $\Gamma$ with the first column is $(\tilde{x}^{*T}\tilde{x}^{*})^{-1/2} \tilde{x}^{*}$ and let $V := \Gamma (\tilde{X}_S ^T \tilde{X}_S ) \Gamma^T$.
Then, $\tilde{x}^{*T} (\tilde{X}_S ^T \tilde{X}_S)^{-1} \tilde{x}^{*} = (\Gamma \tilde{x}^{*})^T V^{-1} (\Gamma \tilde{x}^{*}) = \tilde{x}^{*T}\tilde{x}^{*} v^{11}$ where $V^{-1} = (v^{ij})$.
Similarly, we obtain $\tilde{x}^{*T} (\tilde{X}_S ^T \tilde{X}_S)^{-2} \tilde{x}^{*} = \tilde{x}^{*T}\tilde{x}^{*} \sum_{j=1} ^p (v^{1j})^2$.

Now we let $V = T T^T$ where $T$ is an upper triangular matrix with positive diagonal elements as
\begin{align*}
    T=
  \left( {\begin{array}{cc}
   t_{11} & \mathbf{t}^T \\
   0 & T_{22} \\
  \end{array} } \right).
\end{align*}
Then, 
\begin{align*}
    T^{-1}=
  \left( {\begin{array}{cc}
   t_{11} ^{-1} & -t_{11} ^{-1} \mathbf{t}^T T_{22} ^{-1} \\
   0 & T_{22} ^{-1} \\
  \end{array} } \right), \quad 
  V^{-1} =
  \left( {\begin{array}{cc}
   t_{11} ^{-2} & -t_{11} ^{-2} \mathbf{t}^T T_{22} ^{-1} \\
   -t_{11} ^{-2} (T_{22}^T) ^{-1} \mathbf{t} & (T_{22} T_{22}^T )^{-1} + t_{11} ^{-2} ( T_{22}^T )^{-1} \mathbf{t} \mathbf{t}^T T_{22}^{-1}\\
  \end{array} } \right).
\end{align*}
Therefore, 
\begin{align*}
    \frac{\tilde{x}^{*T} (\tilde{X}_S ^T \tilde{X}_S)^{-2} \tilde{x}^{*} }{(1+\tilde{x}^{*T} (\tilde{X}_S ^T \tilde{X}_S)^{-1} \tilde{x}^*)^2} = \frac{ \tilde{x}^{*T}\tilde{x}^{*} ( t_{11}^{-4}+ t_{11}^{-4} \mathbf{t}^T(T_{22} ^T T_{22})^{-1} \mathbf{t}) }{( 1+ \tilde{x}^{*T}\tilde{x}^{*} t_{11} ^{-2} )^2} = \frac{ \tilde{x}^{*T}\tilde{x}^{*} ( 1+ \mathbf{t}^T(T_{22} ^T T_{22})^{-1} \mathbf{t}) }{( \tilde{x}^{*T}\tilde{x}^{*} + t_{11} ^2 )^2}.
\end{align*}
Due to \citet[Theorem 3.3.5]{gupta1999}, $t_{11}^2$ is independent to $\mathbf{t}^T (T_{22} ^T T_{22})^{-1} \mathbf{t}$ with $t_{11}^2 \sim \chi_{|S|-p+1} ^2$.
Furthermore, by \citet[Theorem 3.3.28]{gupta1999}, $\mathbf{t}^T(T_{22} ^T T_{22})^{-1} \mathbf{t} \sim \frac{p-1}{|S|-p+2} F_{p-1,|S|-p+2}$.
That is,
\begin{align*}
    \mathbb{E}_{X_S \sim P_X ^{j-1}}[\frac{\tilde{x}^{*T} (\tilde{X}_S ^T \tilde{X}_S)^{-2} \tilde{x}^{*} }{(1+\tilde{x}^{*T} (\tilde{X}_S ^T \tilde{X}_S)^{-1} \tilde{x}^*)^2}] &= \tilde{x}^{*T}\tilde{x}^{*} \mathbb{E}[(1+ \mathbf{t}^T(T_{22} ^T T_{22})^{-1} \mathbf{t})] \mathbb{E}[ \frac{1}{( \tilde{x}^{*T}\tilde{x}^{*} + t_{11} ^2 )^2}] \\
    &= \tilde{x}^{*T}\tilde{x}^{*} \frac{|S|-1}{|S|-p} \mathbb{E}[ \frac{1}{( \tilde{x}^{*T}\tilde{x}^{*} + t_{11} ^2 )^2}].
\end{align*}

[Step 2] Similarly, we have
\begin{align*}
    \frac{ \tilde{x}^{*T} (\tilde{X}_S ^T \tilde{X}_S)^{-2} \tilde{x}^{*} }{1+\tilde{x}^{*T} (\tilde{X}_S ^T \tilde{X}_S)^{-1} \tilde{x}^*} &= \frac{\tilde{x}^{*T}\tilde{x}^{*} ( t_{11}^{-4}+ t_{11}^{-4} \mathbf{t}^T(T_{22} ^T T_{22})^{-1} \mathbf{t}) }{ 1+ \tilde{x}^{*T}\tilde{x}^{*} t_{11} ^{-2} } \\
    &= (1+ \mathbf{t}^T(T_{22} ^T T_{22})^{-1} \mathbf{t} ) \left( \frac{1}{ t_{11} ^2} - \frac{1}{ t_{11} ^2 + \tilde{x}^{*T}\tilde{x}^{*} } \right),
\end{align*}
and
\begin{align*}
     \mathbb{E}_{X_S \sim P_X ^{j-1}}[\frac{\tilde{x}^{*T} (\tilde{X}_S ^T \tilde{X}_S)^{-2} \tilde{x}^{*}}{1+\tilde{x}^{*T} (\tilde{X}_S ^T \tilde{X}_S)^{-1} \tilde{x}^*}] &= \frac{|S|-1}{|S|-p} \mathbb{E}[  \left( \frac{1}{ t_{11} ^2} - \frac{1}{ t_{11} ^2 + \tilde{x}^{*T}\tilde{x}^{*} } \right)] \\
     &= \frac{|S|-1}{|S|-p} \left( \frac{1}{ |S|-p-1} - \mathbb{E}[\frac{1}{ t_{11} ^2 + \tilde{x}^{*T}\tilde{x}^{*} }] \right). 
\end{align*}

[Step 3]
Therefore, for any $q \geq p+3$ and Chi-squared distributions $T_j \sim \chi_{j-p+1} ^2$ (or equivalently Gamma distributions $T_j \sim \mathrm{Gamma}((j-p+1)/2, 1/2)$), we have
\begin{align*}
&\nu( (x^*, y^*); U_{q,0}, P_{X,Y}, m) \\
&= C_{\mathrm{lin}} - \sigma^2 - \mathbb{E}_{S \sim P_{X,Y}^{q-1}} [(\hat{\beta}_{S, \gamma} -\beta)^T \Sigma_X (\hat{\beta}_{S, \gamma} -\beta)] \\
&+ \frac{\sigma ^{2}}{m} \sum_{j=q} ^m \Bigg(   (1-\frac{e^{*2}}{\sigma^{2}})  \frac{j-1}{j-p} \mathbb{E}[ \frac{ \tilde{x}^{*T}\tilde{x}^{*} }{( \tilde{x}^{*T}\tilde{x}^{*} + T_j )^2}] + \frac{j-1}{j-p} \left( \frac{1}{j-p-1} - \mathbb{E}[\frac{1}{  \tilde{x}^{*T}\tilde{x}^{*}+T_j }] \right) \Bigg)
\end{align*}
By setting 
\begin{align*}
    C_{\mathrm{lin}} = \sigma^2 + \mathbb{E}_{S \sim P_{X,Y}^{q-1}} [(\hat{\beta}_{S, \gamma} -\beta)^T \Sigma_X (\hat{\beta}_{S, \gamma} -\beta)] - \frac{\sigma ^{2}}{m} \sum_{j=q} ^m \frac{j-1}{j-p}  \frac{1}{j-p-1},
\end{align*}
we have
\begin{align*}
&\nu( (x^*, y^*); U_{q,0}, P_{X,Y}, m) \\
&= \frac{\sigma ^{2}}{m} \sum_{j=q} ^m \Bigg(   (1-\frac{e^{*2}}{\sigma^{2}})  \frac{j-1}{j-p} \mathbb{E}[ \frac{ \tilde{x}^{*T}\tilde{x}^{*} }{( \tilde{x}^{*T}\tilde{x}^{*} + T_j )^2}] - \frac{j-1}{j-p}  \mathbb{E}[\frac{1}{  \tilde{x}^{*T}\tilde{x}^{*}+T_j }] \Bigg)\\
&= - \frac{1}{m} \sum_{j=q} ^m  \mathbb{E} \left[ \frac{j-1}{j-p} \frac{ \left(  x^{*T} \Sigma_X ^{-1} x^* e^{*2} + T_j \sigma^{2} \right) }{ ( x^{*T} \Sigma_X ^{-1} x^* + T_j)^2  } \right].
\end{align*}
\end{proof}

\subsection{Proof of Theorem \ref{thm:sub_gaussian_inputs_ridge}}

To begin, we first define some notations and a useful lemma.
Let $\lambda_{\mathrm{min}}(A)$ and $\lambda_{\mathrm{max}}(A)$ be the smallest and largest singular values of a matrix $A$.
For a sub-Gaussian random variable $X$, we denote its sub-Gaussian norm by $\norm{X}_{\psi_2} := \sup_{p \geq 1} p^{-1/2} (\mathbb{E}|X|^p)^{1/p}$.
For a sub-Gaussian random vector $X$, we denote its sub-Gaussian norm by $\norm{X}_{\psi_2} := \sup_{ x^T x =1  } \norm{\langle X, x \rangle}_{\psi_2}$. 
Lastly, we quote the non-asymptotic eigenvalue bounds by \citet[Theorem 5.39]{vershynin2010}.
\begin{lemma}
Suppose that $\tilde{X}_S$ is a matrix whose rows are independent sub-Gaussian isotropic random vectors in $\mathbb{R}^p$, then for every $t \geq 0$, with probability at least $1-2 \exp(-ct^2)$ one has
\begin{align*}
\sqrt{|S|}(1-\delta_{|S|}) = \sqrt{|S|}-C \sqrt{p}-t \leq \lambda_{\mathrm{min}} (\tilde{X}_S) \leq \lambda_{\mathrm{max}} (\tilde{X}_S) \leq \sqrt{|S|} + C \sqrt{p} + t = \sqrt{|S|}(1+\delta_{|S|}),
\end{align*}
where $\delta_{|S|} = (C\sqrt{p}+ t) / \sqrt{|S|}$ and $C$, $c$ are two constants depending only on the sub-Gaussian norm.
\label{lemma:eigenvalue_bounds}
\end{lemma}

\begin{proof}[Proof of Theorem \ref{thm:sub_gaussian_inputs_ridge}]
[Step 1]
We provide a proof for the upper bound only, but the similar procedure can show the lower bound.
To this end, we fix $S$ and let $\tilde{X}_S = X_S \Sigma_X ^{-1/2}$, $\tilde{x}^* = \Sigma_X ^{-1/2} x^*$, and $\tilde{A}_{\gamma} = (\tilde{X}_S ^T \tilde{X}_S+ \gamma \Sigma_X^{-1} )$.
Then, we have
\begin{align}
    \frac{x^{*T} A_{\gamma} ^{-1} \Sigma_X A_{\gamma} ^{-1} x^{*} }{ 1+x^{*T} A_{\gamma} ^{-1} x^*} \frac{ (2 + x^{*T} A_{\gamma} ^{-1} x^{*})\sigma ^{2} - e ^{*2} }{1+x^{*T} A_{\gamma} ^{-1} x^*} = \frac{(\tilde{x}^{*T} \tilde{A}_{\gamma} ^{-2} \tilde{x}^{*}) (2 + \tilde{x}^{*T} \tilde{A}_{\gamma} ^{-1} \tilde{x}^{*}) \sigma ^{2} - e ^{*2} }{(1 + \tilde{x}^{*T} \tilde{A}_{\gamma} ^{-1} \tilde{x}^*)^2}.
    \label{eq:ridge_part}
\end{align}
Due to $\lambda_{\mathrm{max}}(AB) \leq \lambda_{\mathrm{max}}(A)\lambda_{\mathrm{max}}(B)$, we have
\begin{align*}
&\frac{(\tilde{x}^{*T} \tilde{A}_{\gamma} ^{-2} \tilde{x}^{*}) (2 + \tilde{x}^{*T} \tilde{A}_{\gamma} ^{-1} \tilde{x}^{*}) \sigma ^{2} - e ^{*2} }{(1 + \tilde{x}^{*T} \tilde{A}_{\gamma} ^{-1} \tilde{x}^*)^2} \\
&\leq \frac{ \tilde{x}^{*T} \tilde{x}^{*} \lambda_{\mathrm{max}}( \tilde{A}_{\gamma} ^{-2}) (2 + \tilde{x}^{*T} \tilde{x}^{*} \lambda_{\mathrm{max}}( \tilde{A}_{\gamma} ^{-1})) }{ (1+\tilde{x}^{*T} \tilde{x}^* \lambda_{\mathrm{min}}(\tilde{A}_{\gamma} ^{-1}) )^2} \sigma^2 - \frac{1}{ (1+\tilde{x}^{*T} \tilde{x}^* \lambda_{\mathrm{max}}(\tilde{A}_{\gamma} ^{-1}) )^2} e^{*2}.
\end{align*}
Since $|y_i| \leq B_{\mathcal{Y}} $ and $\hat{\beta}_{S} ^{\mathrm{R}} = \mathrm{argmin}_{\beta} (Y_S - X_S \beta)^T (Y_S - X_S \beta) + \gamma \norm{\beta}_2 ^2$, we obtain boundedness of $\norm{ \hat{\beta}_{S} ^{\mathrm{R}}}_2 ^2$, \textit{i.e.}, $\norm{ \hat{\beta}_{S} ^{\mathrm{R}}}_2 ^2 \leq \gamma^{-1} Y_S ^T Y_S \leq \gamma^{-1} m B_{\mathcal{Y}} ^2$ for any $S \subseteq \mathcal{X} \times \mathcal{Y}$.
That means, $U_q ^{\mathrm{R}}(S)$ is bounded, and thus Equation \eqref{eq:ridge_part} is bounded as well.
Let say the bound is $C_{\mathrm{bdd}}$.

[Step 2] Using Lemma \ref{lemma:eigenvalue_bounds} with $t_{|S|} = \sqrt{\frac{\log(|S| m^{1/2})}{c}}$, the following holds with probability at least $1-2/(|S|m^{1/2})$.
\begin{align*}
\sqrt{|S|}(1-\delta_{|S|}) = \sqrt{|S|}-C \sqrt{p}-t \leq \lambda_{\mathrm{min}} (\tilde{X}_S) \leq \lambda_{\mathrm{max}} (\tilde{X}_S) \leq \sqrt{|S|} + C \sqrt{p} + t = \sqrt{|S|}(1+\delta_{|S|}),
\end{align*}
where $\delta_{|S|} = (C\sqrt{p}+\sqrt{\frac{\log(|S|m)}{2c}}) / \sqrt{|S|}$.
We denote the set where the inequalities hold by $\Omega_{|S|}$ and we obtain the following bounds.
\begin{align*}
& \mathbb{E}_{X_S \sim P_X ^{j-1}} \Bigg[ \frac{(\tilde{x}^{*T} \tilde{A}_{\gamma} ^{-2} \tilde{x}^{*}) (2 + \tilde{x}^{*T} \tilde{A}_{\gamma} ^{-1} \tilde{x}^{*}) \sigma ^{2} - e ^{*2} }{(1 + \tilde{x}^{*T} \tilde{A}_{\gamma} ^{-1} \tilde{x}^*)^2}  \Bigg] \\
&\leq \int_{\Omega_{|S|}} \frac{  \tilde{x}^{*T} \tilde{x}^{*} \lambda_{\mathrm{max}}( \tilde{A}_{\gamma} ^{-2}) (2 + \tilde{x}^{*T} \tilde{x}^{*} \lambda_{\mathrm{max}}( \tilde{A}_{\gamma} ^{-1})) }{ (1+\tilde{x}^{*T} \tilde{x}^* \lambda_{\mathrm{min}}(\tilde{A}_{\gamma} ^{-1}) )^2} \sigma^2 dP - \int_{\Omega_{|S|}} \frac{1}{ (1+\tilde{x}^{*T} \tilde{x}^* \lambda_{\mathrm{max}}(\tilde{A}_{\gamma} ^{-1}) )^2} e^{*2} dP \\
&+ \int_{ \Omega^c } C_{\mathrm{bdd}} dP \\
&\leq \frac{ \tilde{x}^{*T} \tilde{x}^{*} ( |S|(1-\delta_{|S|})^2 + \gamma \lambda_{\mathrm{min}}(\Sigma_X^{-1}))^{-2} }{ (1+\tilde{x}^{*T} \tilde{x}^* ( |S|(1+\delta_{|S|})^2 + \gamma \lambda_{\mathrm{max}}(\Sigma_X^{-1}) )^{-1})^2} \left( 2 + \tilde{x}^{*T} \tilde{x}^{*} ( |S|(1-\delta_{|S|})^2 + \gamma \lambda_{\mathrm{min}}(\Sigma_X^{-1}))^{-1} \right) \sigma^2 \\
&- \frac{e^{*2}}{ (1+\tilde{x}^{*T} \tilde{x}^* ( |S|(1-\delta_{|S|})^2 + \gamma \lambda_{\mathrm{min}}(\Sigma_X^{-1}) )^{-1})^2} + C_{\mathrm{bdd}} P(\Omega_{|S|} ^c),
\end{align*}
where the second inequality is due to $\lambda_{\mathrm{min}}(A+B) \geq \lambda_{\mathrm{min}} (A)+\lambda_{\mathrm{min}} (B) $ and $\lambda_{\mathrm{max}}(A+B) \leq \lambda_{\mathrm{max}} (A)+\lambda_{\mathrm{max}} (B)$.
Hence,
\begin{align*}
&\nu((x^*, y^*); U_{q,\gamma}, P_{X,Y},m) \\
&\leq \frac{1}{m} \sum_{j=q-1} ^{m-1}  \frac{ \tilde{x}^{*T} \tilde{x}^{*} ( j(1-\delta_{j})^2 + \gamma \lambda_{\mathrm{min}}(\Sigma_X^{-1}))^{-2} }{ (1+\tilde{x}^{*T} \tilde{x}^* ( j(1+\delta_{j})^2 + \gamma \lambda_{\mathrm{max}}(\Sigma_X^{-1}) )^{-1})^2} \left( 2 + \tilde{x}^{*T} \tilde{x}^{*} ( j(1-\delta_{j})^2 + \gamma \lambda_{\mathrm{min}}(\Sigma_X^{-1}))^{-1} \right) \sigma^2 \\
&- \frac{1}{m} \sum_{j=q-1} ^{m-1} \frac{ e^{*2} }{ (1+\tilde{x}^{*T} \tilde{x}^* (j(1-\delta_{j})^2 +  \gamma \lambda_{\mathrm{min}}(\Sigma_X^{-1}) )^{-1})^2} + \frac{C_{\mathrm{bdd}} }{m} \sum_{j=q-1} ^{m-1} P(\Omega_{j}^c) + h(\gamma) \\
&= \frac{1}{m} \sum_{j=q-1} ^{m-1}  \frac{ \tilde{x}^{*T} \tilde{x}^{*} \Lambda_{\mathrm{upper}}^{2} (j)  }{ (1+\tilde{x}^{*T} \tilde{x}^* \Lambda_{\mathrm{lower}} (j) )^2} \left( 2 + \tilde{x}^{*T} \tilde{x}^{*} \Lambda_{\mathrm{upper}} (j) \right) \sigma^2 \\
&- \frac{1}{m} \sum_{j=q-1} ^{m-1} \frac{ e^{*2} }{ (1+\tilde{x}^{*T} \tilde{x}^* \Lambda_{\mathrm{upper}} (j) )^2} + \frac{C_{\mathrm{bdd}} }{m} \sum_{j=q-1} ^{m-1} P(\Omega_{j}^c) + h(\gamma),
\end{align*}
where $\Lambda_{\mathrm{upper}}(j) := (j(1-\delta_{j})^2 + \gamma \lambda_{\mathrm{min}}(\Sigma_X^{-1}))^{-1}$ and $\Lambda_{\mathrm{lower}}(j) := (j(1+\delta_{j})^2 + \gamma \lambda_{\mathrm{max}}(\Sigma_X^{-1}))^{-1}$ for $j \in \mathbb{N}$.
Lastly, $\frac{1}{m} \sum_{j=q-1} ^{m-1} P(\Omega_{j}^c) = \frac{1}{m} \sum_{j=q-1} ^{m-1} \frac{2}{j\sqrt{m}} \leq 4 \frac{\log(m)}{m^{3/2}}$ concludes a proof.
\end{proof}

\begin{remark}
It is noteworthy that the eigenvalues of $A_{S,\gamma}^{-1}$ are contained in $[\Lambda_{\mathrm{lower}}(j),$ and $\Lambda_{\mathrm{upper}}(j)]$ with high probability.
By Lemma \ref{lemma:eigenvalue_bounds}, on $\Omega_{j}$, we have
\begin{align*}
    j(1-\delta_{j})^2 + \gamma \lambda_{\mathrm{min}}(\Sigma_X^{-1}) \leq \lambda_{\mathrm{min}}(A_{S,\gamma}) \leq \lambda_{\mathrm{max}}(A_{S,\gamma}) \leq j(1+\delta_{j})^2 + \gamma \lambda_{\mathrm{max}}(\Sigma_X^{-1}),
\end{align*}
and thus
\begin{align*}
    \Lambda_{\mathrm{lower}}(j) \leq \lambda_{\mathrm{min}}(A_{S,\gamma} ^{-1}) \leq \lambda_{\mathrm{max}}(A_{S,\gamma} ^{-1}) \leq \Lambda_{\mathrm{upper}}(j).
\end{align*}
\end{remark}

\subsection{Proof of Corollary \ref{cor:gaussian_inputs_binary}}
We first provide a detailed version of Corollary \ref{cor:gaussian_inputs_binary}.

\begin{corollary}
[DShapley in binary classification; a detailed version]
Assume $\mathbb{E}[Y \mid X]= \mathrm{logit}^{-1}(X ^T \beta)$ and $X$ are sub-Gaussian in $\mathbb{R}^p$ with $\mathbb{E}(X X^T) = \Sigma_X$. 
For a point $(x^*, y^*)$, let $\pi^* = \mathrm{logit}^{-1} (x^{*T}\beta)$,  $w^*=\pi^* (1-\pi^*)$, and $z^* = x^{*T}\beta + (y^*-\pi^*)/w^*$.
Then, for any $q \geq p+3$ and some fixed constant $C_{\mathrm{lin}}$, DShapley of a point $\left( (w^*)^{1/2}x^*, (w^*)^{1/2}z^* \right)$ has the following upper and lower bounds.
\begin{align*}
&\frac{1}{m} \sum_{j=q-1} ^{m-1}  \frac{ w^* {x}^{*T} \tilde{\Sigma}_X ^{-1} {x}^{*} \tilde{\Lambda}_{\mathrm{lower}}^{2} (j) }{ (1+ w^* {x}^{*T} \tilde{\Sigma}_X ^{-1} {x}^{*} \tilde{\Lambda}_{\mathrm{upper}} (j) )^2} \left( ( 2 + w^* {x}^{*T} \tilde{\Sigma}_X ^{-1} {x}^{*} \tilde{\Lambda}_{\mathrm{lower}} (j) )  - \tilde{\Lambda}_{\mathrm{ratio}} ^{-1} (j) e_{\mathrm{b}}^{*2} \right) \\
&\leq \nu \left( \left( (w^*)^{1/2}x^*, (w^*)^{1/2}z^* \right); U_{q,0}, P_{\tilde{X},\tilde{Z}}, m \right) + o\left(\frac{1}{m}\right) \\
&\leq \frac{1}{m} \sum_{j=q-1} ^{m-1}  \frac{ w^* {x}^{*T} \tilde{\Sigma}_X ^{-1} {x}^{*} \tilde{\Lambda}_{\mathrm{upper}}^{2} (j) }{ (1+ w^* {x}^{*T} \tilde{\Sigma}_X ^{-1} {x}^{*} \tilde{\Lambda}_{\mathrm{lower}} (j) )^2} \left( ( 2 + w^* {x}^{*T} \tilde{\Sigma}_X ^{-1} {x}^{*} \tilde{\Lambda}_{\mathrm{upper}} (j) )  - \tilde{\Lambda}_{\mathrm{ratio}} (j) e_{\mathrm{b}}^{*2} \right),
\end{align*}
where $e_{\mathrm{b}} ^{*2} := (w^*)^{-1}(y^* -\pi^*)^2$, the function $h$ is defined in Proposition \ref{proposition:ridge}, $\tilde{\Sigma}_X := \mathbb{E}[wXX^T]$, $\tilde{\Lambda}_{\mathrm{upper}}(j) := (j(1-\delta_{j})^2)^{-1}$,  $\tilde{\Lambda}_{\mathrm{lower}}(j) := (j(1+\delta_{j})^2)^{-1}$, and $\delta_{j} = (C\sqrt{p}+\sqrt{\frac{\log(jm)}{2c}}) / \sqrt{j}$ for $j \in \mathbb{N}$ and certain constants $c$, $C$ as in the proof of Theorem \ref{thm:sub_gaussian_inputs_ridge}. Lastly,
\begin{align*}
\tilde{\Lambda}_{\mathrm{ratio}}(j) = \left( \frac{ 1+ w^* {x}^{*T} \tilde{\Sigma}_X ^{-1} {x}^{*} \tilde{\Lambda}_{\mathrm{lower}} (j) }{ 1 + w^* {x}^{*T} \tilde{\Sigma}_X ^{-1} {x}^{*} \tilde{\Lambda}_{\mathrm{upper}} (j) } \right)^2.
\end{align*}
\label{cor:gaussian_inputs_binary_full}
\end{corollary}

\begin{proof}[Proof of Corollary \ref{cor:gaussian_inputs_binary}]
Since $\mathbb{E}[ Y \mid X]= \mathrm{logit}^{-1}(X ^T \beta)=\pi$, we have $\mathbb{E}[ w^{1/2} Z \mid w^{1/2} X] = w^{1/2} X ^T \beta$ and $\mathrm{Var}[ w^{1/2} Z \mid w^{1/2} X] = 1$.
Furthermore, by the definition of sub-Gaussian, $w^{1/2}X$ is also sub-Gaussian with $\tilde{\Sigma}_X$ because $w \leq 1$ and $X$ are sub-Gaussian with $\mathbb{E}(XX^T)=\Sigma_X$.
With the notations, Theorem \ref{thm:sub_gaussian_inputs_ridge} with $\gamma=0$ gives the upper and lower bounds.
\end{proof}

\subsection{Proof of Theorem \ref{thm:shapley_density}}
\begin{proof}[Proof of Theorem \ref{thm:shapley_density}]
Let $S^* = \{z_1 ^*, \dots, z_n ^*\}$.
A simple algebra gives
$\hat{p}_{S \cup S^* } (z) = \frac{1}{|S|+n} (\sum_{j=1} ^n k(z, z_j ^*) + |S|\hat{p}_{S} (z) ) = \frac{1}{|S|+n} \sum_{j=1} ^n k(z, z_j ^*) + \frac{|S|}{|S|+n} \hat{p}_{S} (z) = \hat{p}_{S} (z)  + \frac{n}{|S|+n} ( \frac{1}{n} \sum_{j=1} ^n k(z, z_j ^*) - \hat{p}_{S} (z) )$.
Note that $\frac{1}{n} \sum_{j=1} ^n k(z, z_j ^*) = \hat{p}_{S^*} (z)$.
For $|S| \geq 1$, we have
\begin{align*}
    &U(S \cup S^*) - U(S) \\
    &= -\int (p(z) - \hat{p}_{S \cup S^* } (z) )^2 - (p(z) - \hat{p}_S (z) )^2  dz \\
    &= -\int  \left(p(z) - \hat{p}_S(z) - \frac{n}{|S|+n} \left( \hat{p}_{S^*} (z) - \hat{p}_S(z) \right) \right)^2 - (p(z) - \hat{p}_S (z) )^2  dz \\
    &= -\int  \frac{n^2}{(|S|+n)^2} \left( \hat{p}_{S^*} (z) -\hat{p}_S(z) \right)^2  -\frac{2n}{|S|+n} \left\{ (p(z) - \hat{p}_S(z)) \left( \hat{p}_{S^*} (z) - \hat{p}_S(z) \right) \right\} dz.
\end{align*}
Furthermore,
\begin{align}
    ( \hat{p}_{S^*} (z) - \hat{p}_S(z) )^2
    &= ( \hat{p}_{S^*} (z) - p(z) + p(z) -\hat{p}_S(z) )^2  \notag  \\
    &= ( \hat{p}_{S^*} (z) - p(z))^2 + (p(z) -\hat{p}_{S} (z))^2 + 2( \hat{p}_{S^*} (z) - p(z))(p(z) -\hat{p}_{S} (z)),
    \label{eq:part1_set}
\end{align}
and
\begin{align}
    (p(z) - \hat{p}_S(z)) (\hat{p}_{S^*} (z) - \hat{p}_{S} (z))  &= (p(z) - \hat{p}_S(z)) (\hat{p}_{S^*} (z) - p(z) + p(z) -\hat{p}_{S} (z)) \notag \\
    &= (p(z) - \hat{p}_S(z)) (\hat{p}_{S^*} (z) -p(z)) + (p(z) - \hat{p}_S(z))^2.
    \label{eq:part2_set}
\end{align}
Equations \eqref{eq:part1_set} and \eqref{eq:part2_set} give
\begin{align*}
    \mathbb{E}[U(S \cup S^*) - U(S)] &= - \frac{ n^2 }{(|S|+n)^2} \int ( \hat{p}_{S^*} (z)- p(z))^2  dz \\
    &+ \frac{ n^2 + 2n|S| }{(|S|+n)^2} \int \mathbb{E}[(p(z) -\hat{p}_{S} (z))^2]  dz \\
    &+ \frac{  2n|S| }{(|S|+n)^2} \int (\hat{p}_{S^*} (z)- p(z)) \mathbb{E}[p(z) -\hat{p}_{S} (z)] dz.
\end{align*}

We can decompose $\mathbb{E}[U(S \cup S^*) - U(S)]$ into two terms by dependency of $S^*$. 
To be more specific, $\mathbb{E}[U(S \cup S^*) - U(S)] = h_1(S^*, |S|)+ h_2(|S|^*, |S|)$ where
\begin{align*}
    h_1(S^*, |S|) &= - \frac{ n^2 }{(|S|+n)^2}  \int ( \hat{p}_{S^*} (z)- p(z))^2 dz  + \frac{ 2n|S| }{(|S|+n)^2} \int \hat{p}_{S^*} (z) \mathbb{E}[p(z) -\hat{p}_{S} (z)] dz.
\end{align*}   
Also,
\begin{align*}
    h_2 (n,|S|) &= \frac{ n^2 +2n |S| }{ (|S|+n)^2 } \int \mathbb{E}[(p(z) -\hat{p}_{S} (z))^2] dz - \frac{ 2 n|S| }{ (|S|+n)^2 } \int p(z) \mathbb{E}[p(z) -\hat{p}_{S} (z)] dz\\
    &= \frac{ n^2 +2n |S| }{ (|S|+n)^2 } \int \mathbb{E}[(p(z) -\hat{p}_{S} (z))^2] dz - \frac{ 2 n|S| }{ (|S|+n)^2 } \int p(z) (p(z) - \mathbb{E}[k(z, Z)]) dz.
\end{align*}
Therefore, by \citet[Theorem 2.3]{ghorbani2020}, we have
\begin{align}
    \nu(S^* ; U, P, m) &= \frac{1}{m} \sum_{j=1} ^m \mathbb{E}_{S \sim P ^{j-1}} [U (S \cup S^* )-U (S)] \notag \\
    &= - \frac{1}{m} \sum_{j=1} ^m  \frac{ n^2 }{(j+n-1)^2} \int ( \hat{p}_{S^*} (z)- p(z))^2 dz \notag \\
    &+  \frac{1}{m} \sum_{j=2} ^m \frac{ 2n (j-1) }{(j+n-1)^2} \int \hat{p}_{S^*} (z) (p(z) - \mathbb{E}[k(z, Z)]) dz + C_0(n,m) \notag \\
    &= -A(n,m)  \int ( \hat{p}_{S^*} (z)- p(z))^2 dz + B(n,m) g(S^*) + C_0(n,m) \label{eqn:full_dsv_density},
\end{align}
and
\begin{align}
    C_0(n,m) &= \frac{1}{m} C_{\mathrm{den}} + \frac{1}{m} \sum_{j=2} ^m h_2 (n,j-1).
    \label{eqn:density_constant}
\end{align}
Hence, it concludes a proof by choosing the constant $C_{\mathrm{den}}$ as follows.
\begin{align}
    C_{\mathrm{den}} = - \sum_{j=2} ^m h_2 (n,j-1). \label{eqn:density_constant_full}
\end{align}
\end{proof}

\section{Details for Examples in Section \ref{s:dsv_density}}
\subsection{Details for Example \ref{exp:two_set}}
\begin{proof}[Proof of Example \ref{exp:two_set}]
A key idea is to develop Equation \eqref{eqn:full_dsv_density}.

[Step 1]
In this step we compute 
\begin{align*}
    h_2 (n,|S|) &= \frac{ n^2 +2n |S| }{ (|S|+n)^2 } \int \mathbb{E}[(p(z) -\hat{p}_{S} (z))^2] dz - \frac{ 2 n|S| }{ (|S|+n)^2 } \int p(z) (p(z) - \mathbb{E}[k(z, Z)]) dz.
\end{align*}
We first compute the term $\int \mathbb{E}[(p(z) -\hat{p}_{S} (z))^2] dz$.
Note that $\hat{p}_{S} ^2 (z) =\frac{1}{|S|^2} (\sum_{z_i \in S} k(z, z_i) ^2 + \sum_{i \neq j : z_i, z_j \in S} k(z, z_i) k(z, z_j))$.
We have
\begin{align*}
\mathbb{E}[ \hat{p}_{S} (z) ] = \mathbb{E}[k (z, Z)] = \begin{cases}
\frac{1}{2}+\frac{z}{h} & 0 \leq z \leq h/2,\\
1 & h/2 \leq z \leq 1-h/2, \\
\frac{1}{2}+\frac{1-z}{h} & 1-h/2 \leq z \leq 1,
\end{cases}    
\end{align*}
and due to $p(z)=1$,
\begin{align*}
p(z) - \mathbb{E}[k (z, Z)] = \begin{cases}
\frac{1}{2}-\frac{z}{h} & 0 \leq z \leq h/2,\\
0 & h/2 \leq z \leq 1-h/2, \\
\frac{1}{2}-\frac{1-z}{h} & 1-h/2 \leq z \leq 1.
\end{cases}    
\end{align*}
Since $S$ are randomly sampled, we have
\begin{align*}
    \mathbb{E}[ \hat{p}_{S} ^2 (z) ] = \frac{ |S|\mathbb{E}[ k (z, Z) ]/h + |S|(|S|-1)\mathbb{E}[ k (z, Z) ]^2 }{|S|^2} = \frac{\mathbb{E}[ k (z, Z) ]}{|S|h} + \frac{|S|-1}{|S|}\mathbb{E}[ k (z, Z) ]^2.
\end{align*}
Furthermore, we have $\int \mathbb{E}[ k (z, Z) ] dz = 1-h/4$ and $\int \mathbb{E}[k (z, Z) ]^2 dz = 1-5h/12$.
Hence, $\int \mathbb{E}[ \hat{p}_{S} ^2 (z) ] dz = \frac{1}{|S|h} -\frac{1}{4|S|} + \frac{|S|-1}{|S|} (1-\frac{5h}{12} )$ and we have 
\begin{align*}
    \int \mathbb{E}[(p(z) -\hat{p}_{S} (z))^2] dz &= 1 + \left(\frac{1}{|S|h} -\frac{1}{4|S|} + \frac{|S|-1}{|S|} (1-\frac{5h}{12} ) \right) -2 \left( 1-\frac{h}{4} \right) \\
    &= \frac{1}{|S|h} - \frac{5}{4|S|} + \frac{(5+|S|)h}{12|S|} \\
    &= \frac{12-15h+(5+|S|)h^2 }{12|S|h}.
\end{align*}
Lastly, $\int p(z) (p(z) - \mathbb{E}[k(z, Z)]) dz = h/4$ gives
\begin{align*}
    h_2 (n,|S|) &= \frac{ n^2 +2n |S| }{ (|S|+n)^2 } \frac{12-15h+(5+|S|)h^2 }{12|S|h} - \frac{ 2 n|S| }{ (|S|+n)^2 } \frac{h}{4}.
\end{align*}

[Step 2]
By construction of $h$, $g(S^*)=0$.
If $\Delta \geq h$, since $z_1 ^*$ and $z_2 ^*$ are apart at least $h$,  
\begin{align*}
    -\int  (p(z) - \hat{p}_{S^*} (z) )^2 dz &= - \int  (1 - \hat{p}_{S^*} (z) )^2 dz \\
    &= - \left( |S^*|h \left( 1- \frac{1}{|S^*|h} \right)^2 + (1-|S^*|h) \right) \\
    &= 1 - \frac{1}{|S^*|h}.
\end{align*}
Therefore, by aggregating all the results in [Step 1] and [Step 2], we have
\begin{align*}
    \nu(S^*; U, P,m) &= A(2,m) \left( 1 - \frac{1}{2h} \right) + C_0(2, m).
\end{align*}
Note that by Equation \eqref{eqn:density_constant},
\begin{align*}
    C_0(2, m) = \frac{1}{m} C_{\mathrm{den}} + \frac{1}{m} \sum_{j=2} ^m h_2 (2,j-1).
\end{align*}
Note that we set $C_{\mathrm{set}} = C_0(2, m)$ in the manuscript.

[Step 3] We now consider the case $\Delta < h$.
To this end, without loss of generality, we assume that $z_1^* \leq z_2 ^*$.
Then there is overlap between $(z_1 ^*-h/2, z_1 ^*+ h/2)$ and $(z_2 ^*-h/2, z_2 ^*+ h/2)$.
\begin{align*}
    \hat{p}_{S^*}(z) = \begin{cases}
\frac{1}{2h} & z_1 ^* - h/2 \leq z \leq z_2 ^* - h/2, \\
\frac{1}{h} & z_2 ^* - h/2 \leq z \leq z_1 ^* + h/2, \\
\frac{1}{2h} & z_1 ^* + h/2 \leq z \leq z_2 ^* + h/2, \\
0 & \text{otherwise}.
\end{cases}    
\end{align*}
Therefore, $\int (p(z) -\hat{p}_{S^*}(z) )^2 dz = -1 + \frac{1}{h} - \frac{\Delta}{2h^2}$.
Hence, we have
\begin{align*}
    \nu(S^*; U, P,m) &= A(2,m) \left( 1 - \frac{1}{h} + \frac{\Delta}{2h^2} \right) + C_0(2, m).
\end{align*}
\end{proof}

\subsection{Details for Example \ref{exp:two_set_synergy}}

A similar analysis used in Example \ref{exp:two_set} gives 
\begin{align*}
    \nu( z_1 ^*; U, P_Z, m) = A(1,m)(1-\frac{1}{h}) + C_0(1,m),    
\end{align*}
where $C_0(1, m) = \frac{1}{m} C_{\mathrm{den}} + \frac{1}{m} \sum_{j=2} ^m h_2 (1,j-1)$ by Equation \eqref{eqn:density_constant}.
Since it is difficult to solve \eqref{ineqn:synergy} analytically, we numerically examine when \eqref{ineqn:synergy} holds when $C_{\mathrm{den}} = 0.2$ and $m=100$.
For fixed bandwidth $h$, we randomly draw $S^*$ 5000 times and observe if there is a synergy.
We empirically find that the synergy is determined by $\Delta$, so we define the synergy threshold as the smallest $\Delta$ when the synergy happens, \textit{i.e}, if $\Delta$ is greater than the synergy threshold, the inequality $\nu( \{z_1 ^*, z_2 ^* \}; U, P_Z, m) \geq \nu( z_1 ^*; U, P_Z, m) + \nu( z_2 ^*; U, P_Z, m)$ holds.
Also, among the 5000 random sampled sets $S^*$, we estimate probability that the synergy happens.
Figure \ref{fig:density_synergy_plot} shows that the synergy threshold and the corresponding synergy probability as a function of $h$.
As $h$ increases, the synergy threshold (in red dashed) increases and the synergy probability (in blue solid) decreases, meaning that in all bandwidths $h \in (0,0.35)$, the synergy happens when the two points in $S^*$ is far apart to some extent.

\begin{figure}[t]
    \centering
    \includegraphics[scale=0.5]{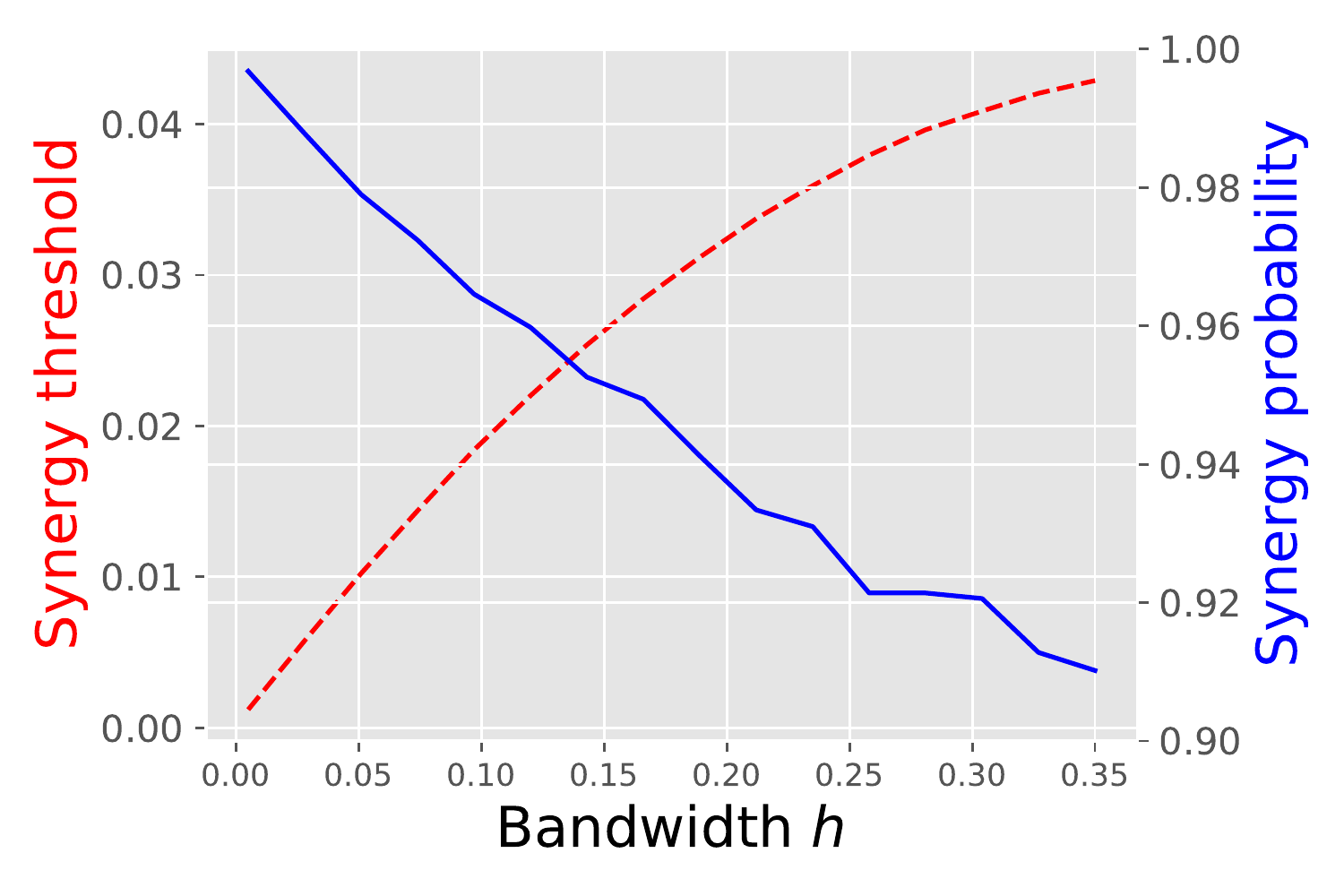}
    \caption{The synergy threshold (red dashed) and the corresponding synergy probability (blue solid) as a function of bandwidth.}
    \label{fig:density_synergy_plot}
\end{figure}

\section{Implementation details}
In this section, we provide implementation details including comprehensive information for algorithms, datasets, and experiment settings. Our implementation codes are available at \url{https://github.com/ykwon0407/fast_dist_shapley}.

\subsection{The proposed algorithms}
In order to estimate DShapley, we implicitly assume that we have a set of random samples $\{ (\tilde{x}_i, \tilde{y}_i) \}_{i=1} ^N$  (\textit{resp.} $\{ \tilde{z}_1, \dots, \tilde{z}_N\}$) from the data distribution $P_{X,Y}$ (\textit{resp.} $P_Z$). This set is used to estimate unknown quantities. For example, the covariance matrix of inputs $\Sigma_X^{-1}$ and the squared error $e^{*2}$ for Alg. \ref{alg:gaussian_inputs_lse_simple} and Alg. \ref{alg:DSV_IRLS}, and the optimal bandwidth in kernel for density estimation problem. 

\paragraph{Linear regression}
DShapley in Theorem \ref{thm:gaussian_inputs_lse} can be viewed as a cumulative sum of decreasing elements, so the computation of every element would be computationally inefficient. Instead of computing the cumulative sum, we consider the partial sum by ignoring negligible expectation terms. We present a detailed version of Alg. \ref{alg:gaussian_inputs_lse_simple}.

\begin{algorithm}[h]
\caption{(Detailed) DShapley for the least squares estimator under Gaussian inputs}
\begin{algorithmic}
\Require True value or estimates for ${x}^{*T} \Sigma_X ^{-1} {x}^{*}$, $e^{*2}$, and $\sigma^2$.
Thresholds $\rho_1=0.01, \rho_2= 0.005$.
The maximum number of Monte Carlo samples $T=10000$.
A constant $q \geq p+3$.
\Procedure{}{}
\State Initialize $\hat{\nu}^{\mathrm{old}} \leftarrow 0$
\For{$j \in \{ q ,\dots, m \}$}
\State Initialize $A_j ^{\mathrm{old}} \leftarrow 0$
\For{$i \in \{ 1 ,\dots, T \}$}
\State Sample $t_{[i]}$ from the $\chi_{j-p+1} ^2$.
\State $A_j ^{\mathrm{new}} \leftarrow \left( (i-1) A_j ^{\mathrm{old}} + \frac{j-1}{j-p}  \frac{x^{*T} \Sigma_X ^{-1} x^* e^{*2} + t_{[i]} \sigma^2 }{( {x}^{*T} \Sigma_X ^{-1} {x}^{*} + t_{[i]} )^2} \right)/i $ \Comment{Based on Theorem \ref{thm:gaussian_inputs_lse}}
\If{$|A_j ^{\mathrm{new}} / A_j ^{\mathrm{old}} - 1| \leq \rho_1$}
\State \textbf{break}
\EndIf
\State $A_j ^{\mathrm{old}} \leftarrow A_j ^{\mathrm{new}}$
\EndFor
\State $\hat{\nu} ^{\mathrm{new}} \leftarrow \hat{\nu}^{\mathrm{old}} - A_j ^{\mathrm{new}}/m$
\If{$| \hat{\nu} ^{\mathrm{old}} / \hat{\nu} ^{\mathrm{new}}  - 1 | \leq \rho_2$}
\State \textbf{break}
\EndIf
\State $\hat{\nu} ^{\mathrm{old}} \leftarrow \hat{\nu} ^{\mathrm{new}}$ 
\EndFor
\State $\hat{\nu}( (x^*, y^*); U_q, P_{X,Y}, m) \leftarrow \hat{\nu}^{\mathrm{new}}$ \Comment{Estimates for DShapley}
\EndProcedure
\end{algorithmic}
\label{alg:gaussian_inputs_lse}
\end{algorithm}

\paragraph{Binary classification}
Likewise Alg. \ref{alg:gaussian_inputs_lse}, the lower bound in Corollary \ref{cor:gaussian_inputs_binary} can be viewed as a cumulative sum of decreasing elements, so we again consider the partial sum. A detailed version of Alg. \ref{alg:DSV_IRLS} is presented in Alg. \ref{alg:DSV_IRLS_detail}.

\begin{algorithm}[h]
\caption{(Detailed) DShapley for binary classification}
\begin{algorithmic}
\Require A datum to be valued $(x^*, y^*)$. A set of random samples $\{(X_i,Y_i)\}_{i=1} ^B$ from $P_{X,Y}$.
\Procedure{Transform\_data}{}
\While{until a convergent condition is met}
\State $\pi_i \leftarrow \mathrm{logit}^{-1}(X_i ^T \hat{\beta}_{\mathrm{IRLS}} )$
\State Update $w_i$ and $Z_i$ based on Equation \eqref{eqn:irls_variabls} and set $\mathbb{W}$ and $\mathbb{Z}$ 
\State $\hat{\beta}_{\mathrm{IRLS}} \leftarrow (\mathbb{X}^T \mathbb{W} \mathbb{X})^{-1} \mathbb{X}^T \mathbb{W} \mathbb{Z}$
\EndWhile
\State $\pi^* \leftarrow \mathrm{logit}^{-1}(x^{*T} \hat{\beta}_{\mathrm{IRLS}} )$
\State $z^* \leftarrow x^{*T} \hat{\beta}_{\mathrm{IRLS}}  + (y^* - \pi^*)/(\pi^*(1-\pi^*))$
\State $w^* \leftarrow \pi^*(1-\pi^*)$
\State Compute a lower bound of DShapley of $\left( (w^*)^{1/2} x^*, (w^*)^{1/2} z^* \right)$.
\EndProcedure
\Require A datum to be valued $((w^*)^{1/2} x^*, (w^*)^{1/2} z^*)$. A set of random samples $\{((w_i) ^{1/2} X_i, (w_i) ^{1/2} Z_i)\}_{i=1} ^B$. Hyperparameters $c=C=1$ and $\rho=0.005$.
\Procedure{Compute\_lower\_bound}{}
\State Initialize $\hat{\nu}^{\mathrm{old}} \leftarrow 0$ and estimate $\tilde{\Sigma}_X$ with $\{(w_1) ^{1/2} X_1, \dots, (w_B) ^{1/2} X_B\}$
\State $e_{\mathrm{b}}^{*2} \leftarrow (w^*)^{-1}(y^* - \pi^*)^2$
\For{$j \in \{ q-1, \dots, m-1 \}$}
\State $\delta_{j} \leftarrow (C\sqrt{p}+\sqrt{\frac{\log(jm)}{2c}}) / \sqrt{j}$
\State $\tilde{\Lambda}_{\mathrm{upper}}(j), \tilde{\Lambda}_{\mathrm{lower}}(j) \leftarrow (j(1-\delta_{j})^2)^{-1}, (j(1+\delta_{j})^2)^{-1}$
\State $\tilde{\Lambda}_{\mathrm{ratio}}(j) \leftarrow \left( \frac{ 1+ w^* {x}^{*T} \tilde{\Sigma}_X ^{-1} {x}^{*} \tilde{\Lambda}_{\mathrm{lower}} (j) }{ 1 + w^* {x}^{*T} \tilde{\Sigma}_X ^{-1} {x}^{*} \tilde{\Lambda}_{\mathrm{upper}} (j) } \right)^2$
\State $A_j \leftarrow \frac{ w^* {x}^{*T} \tilde{\Sigma}_X ^{-1} {x}^{*} \tilde{\Lambda}_{\mathrm{lower}}^{2} (j) }{ (1+ w^* {x}^{*T} \tilde{\Sigma}_X ^{-1} {x}^{*} \tilde{\Lambda}_{\mathrm{upper}} (j) )^2} \left( ( 2 + w^* {x}^{*T} \tilde{\Sigma}_X ^{-1} {x}^{*} \tilde{\Lambda}_{\mathrm{lower}} (j) )  - \tilde{\Lambda}_{\mathrm{ratio}} ^{-1} (j) e_{\mathrm{b}}^{*2} \right)$ \Comment{Based on Corollary \ref{cor:gaussian_inputs_binary_full}}
\State $\hat{\nu} ^{\mathrm{new}} \leftarrow \hat{\nu}^{\mathrm{old}} + A_j ^{\mathrm{new}}/m$
\If{$| \hat{\nu} ^{\mathrm{old}} / \hat{\nu} ^{\mathrm{new}}  - 1 | \leq \rho$}
\State \textbf{break}
\EndIf
\State $\hat{\nu} ^{\mathrm{old}} \leftarrow \hat{\nu} ^{\mathrm{new}}$ 
\EndFor
\EndProcedure
\end{algorithmic}
\label{alg:DSV_IRLS_detail}
\end{algorithm}

\paragraph{Non-parametric density estimation}
DShapley in Theorem \ref{thm:shapley_density} consists of the two integral terms, $\int  (p(z) - \hat{p}_{S^{*}, k} (z) )^2 dz$ and $g(S^*)$. Our approach is to use the MC approximation and to estimate the integrals. Since the first term includes a constant term $\int \{ p(z) \}^2 dz$, we ignore the term as in \citet[Equation (1.183)]{ghosh2018}. We present a practical example of estimation in Alg. \ref{alg:DSV_density}.

\begin{algorithm}[h]
\caption{DShapley for non-parametric density estimation}
\begin{algorithmic}
\Require A set to be valued $S^*$. A Gaussian kernel $k_h$. $B=2000$. A given bandwidth grid $\mathcal{H} := \{ h_1, \dots, h_G \}$. A set of random samples $\{\tilde{z}_1,\dots, \tilde{z}_B\}$ from $P_{Z}$.
\Procedure{}{}
\State Find optimal bandwidth $h^* \in \mathcal{H}$ which minimizes the five-fold cross-validation error
\State Set $k \leftarrow k_{h^*}$
\State Sample $\{\tilde{z}_1 ^*, \dots, \tilde{z}_B ^* \}$ from $\hat{p}_{S^*, k}$
\State $\hat{\nu}(S^*; U_k, P_Z, m) \leftarrow - \frac{A(|S^*|,m)}{B} \sum_{i=1} ^B \left( \hat{p}_{S^*, k} (\tilde{z}_i ^*) -2 \hat{p}_{S^*, k} (\tilde{z}_i) \right) + \frac{B(|S^*|,m)}{B} \sum_{i=1} ^B \left( \hat{p}_{S^*, k} (\tilde{z}_i) - k (\tilde{z}_i ^* - \tilde{z}_i)  \right)$
\EndProcedure
\end{algorithmic}
\label{alg:DSV_density}
\end{algorithm}

\paragraph{Accuracy of the proposed algorithms}
Our algorithms use the Monte-Carlo (MC) method to provide unbiased approximation. When this MC converges, then it guarantees to converge to the true value. In our experiments, we stop the MC when the new increment is small enough compared to the current DShapley estimate to ensure good convergence to the true values. For example, we stop iterations when the new increment is within $0.5\%$ of the current estimates in Alg.~\ref{alg:gaussian_inputs_lse} or we use large samples ($B=2000$) in Alg.~\ref{alg:DSV_density}.

\subsection{Datasets}

\paragraph{Datasets used in time comparison experiment}
We use the two synthetic datasets for the time comparison experiment (Figure~\ref{fig:time_comparison}) as follows.
\begin{itemize}
    \item Linear regression: Given $(m,p)$ and $\beta \sim \mathcal{N}(0, I_p)$, we generate $y_i = x_i ^T \beta + \epsilon_i$ for all $i \in [m]$. Here, $x_i \sim \mathcal{N}(0, I_p)$ and $\epsilon_i \sim \mathcal{N}(0, 1)$ for all $i \in [m]$. We call this data distribution \texttt{Gaussian-R}.
    \item Binary classification: Given $(m,p)$, we generate $y_i = \mathbf{Bern}(0.5)$ and $x_i \sim \mathcal{N}([2 \times y_i, 0, \dots, 0]^T, I_p)$ for all $i \in [m]$.
\end{itemize}

\paragraph{Datasets used in point addition experiment}
We use the two synthetic datasets and eight real datasets for the point addition experiment in Sec. \ref{s:experiments}. For the synthetic datasets, we generate the two types of datasets, \texttt{Gaussian-R} and \texttt{Gaussian-C} for regression and classification, respectively. 
\texttt{Gaussian-R} is described above. As for the \texttt{Gaussian-C}, we first fix $p=3$ and set $\beta=(2,0,0)$. Then, we generate $x_i \sim \mathcal{N}(0, I_p)$ and $y_i = \mathbf{Bern}(\pi_i)$ for all $i \in [m]$. Here $\pi_i := \exp(x_i ^T \beta)/(1+\exp(x_i ^T \beta))$. 
For the real datasets, we collect datasets from multiple sources. For instance, \texttt{abalone}, \texttt{airfoil}, and \texttt{whitewine} are from \texttt{UCI Machine Learning Repository} \citep{dua2019} and \texttt{diabetes} is from \citet{efron2004least}.
A comprehensive list of datasets and details on sample size are provided in Table \ref{tab:summary_real_datasets}.

For the image datasets \texttt{Fashion-MNIST}, \texttt{MNIST} and \texttt{CIFAR10}, we follow the common procedure in prior works \citep{ghorbani2020, koh2017}: we first extract the penultimate layer outputs from the ResNet18 \citep{he2016deep} pre-trained with the ImageNet dataset \citep{russakovsky2015imagenet}. After the extraction, we fit the principal component analysis model and extract the first 32 principal components.

\begin{table}[h]
    \centering
    \caption{A summary of datasets for point addition experiment.}
    \resizebox{\textwidth}{!}{
    \begin{tabular}{lcccccccccccc}
    \toprule
    Dataset & \# of random samples & \# of held-out test data & Input dimension & ML problem & Source &   \\ 
    \midrule
    \texttt{Gaussian-R} & 49000 & 1000 & 10 & Regression & Synthetic dataset \\
    \texttt{abalone} & 3177 & 1000 & 10 & Regression & UCI Repository \\
    \texttt{airfoil} & 1003 & 500 & 5 & Regression & UCI Repository \\
    \texttt{whitewine} & 3898 & 1000 & 11 & Regression & UCI Repository \\
    \midrule
    \texttt{Gaussian-C} & 49000 & 1000 & 3 & Classification & Synthetic dataset \\
    \texttt{skin-nonskin} & 244057 & 1000 & 3 & Classification & \citet{chang2011libsvm} \\
    \texttt{MNIST} & 60000 & 5000 & 32 & Classification & \citet{lecun2010mnist} \\
    \midrule
    \texttt{diabetes} & 342 & 100 & 10 & Density estimation & \citet{efron2004least} \\
    \texttt{australian} & 349 & 100 & 12 & Density estimation & \citet{chang2011libsvm} \\
    \texttt{Fashion-MNIST} & 60000 & 5000 & 32 & Density estimation & \citet{xiao2017fashion} \\
    \midrule
    \multirow{2}{*}{\texttt{CIFAR10}} & \multirow{2}{*}{50000} & \multirow{2}{*}{5000}  & \multirow{2}{*}{32} & Classification & \multirow{2}{*}{\citet{krizhevsky2009learning}} \\
    & & & & Density estimation \\
    \bottomrule
    \end{tabular}}
    \label{tab:summary_real_datasets}
\end{table}

\subsection{Experiment settings}

\paragraph{Point addition experiment}
As for the point addition experiment, we use datasets summarized in Table \ref{tab:summary_real_datasets}. Throughout the experiments, for each dataset, we first randomly select 200 data points to be valued from datasets. For regression and classification problems, all other data points are used to estimate the DShapley, but for the density estimation problem, we randomly pick $2000$ samples. Please note that all the proposed methods, namely Alg.~\ref{alg:gaussian_inputs_lse}, Alg.~\ref{alg:DSV_IRLS_detail}, and Alg.~\ref{alg:DSV_density}, require some data points to estimate unknown-quantities ($\Sigma_X ^{-1}$, $e^{*2}$, or bandwidth)
Every time point we add a data point given order, we evaluate the test accuracy using the held-out dataset. The held-out dataset sizes are provided in Table \ref{tab:summary_real_datasets}.

For linear regression cases, we use the utility function constant $C_{\mathrm{lin}} = 2\hat{\sigma}^2$, where $\hat{\sigma} := \frac{1}{m-p} \sum_{i=1} ^{m-p} (y_i - x_i ^T \hat{\beta})^2$ and $\hat{\beta}$ is the least squares estimator.
For classification, the utility function is classification accuracy.
Lastly, for density estimation, we considered
\begin{align*}
    C_{\mathrm{den}} = m A(n,m) \int \{ p(z) \}^2 dz - \sum_{j=2} ^m h_2 (n,j-1),
\end{align*}
which corresponds to the sum of $m A(n,m) \int \{ p(z) \}^2 dz$ and  \eqref{eqn:density_constant_full}, in order to avoid computing $\int \{ p(z) \}^2 dz$. As for finding the optimal bandwidth, we select one from $\{10^{-2}, 10^{-1.5}, 10^{-1}, 10^{-0.5},  10^{0},  10^{0.5},  10^1 \}$ using the five-fold cross-validation error.

\section{Additional numerical experiments}
\subsection{Illustration of DShapley}
To see how DShapley changes with respect to $x^{*T} \Sigma_X ^{-1} x^{*}$ and $e^{*2}$, we estimate DShapley using Algorithm \ref{alg:gaussian_inputs_lse}.
We consider $m \in \{ 100, 300, 500\}$, $e^{*2} \in \{0, 1, 2, 4, 8\}$, the Gaussian input distribution $X \sim \mathcal{N}_p(0, I_p)$ with $p \in \{10,30\}$. 
Here, we assume that $\Sigma_X ^{-1}$ and $e^{*2}$ are known.
Figure \ref{fig:DSV_vs_x-norm} illustrates DShapley as a function of $x^{*T} \Sigma_X ^{-1} x^{*}$.
As anticipated, for a fixed $x^{*T}\Sigma_X ^{-1} x^*$, DShapley decreases as $e^{*2}$ increases. Moreover, DShapley exhibits different behavior depending on the error level. When $e^{*2}$ is small, DShapley increases as $x^{*T}\Sigma_X ^{-1} x^*$ increases. However, when $e^{*2}$ is big enough, DShapley shows non-monotonic curves in $x^{*T}\Sigma_X ^{-1} x^*$.  This is because of its form \eqref{eqn:DSV_gaussian_exact}. The fraction in \eqref{eqn:DSV_gaussian_exact} has a form of a weighted sum of $e^{*2}$ and $\sigma^2$, so it mainly relies on $e^{*2}$ for small values of $x^{*T}\Sigma_X ^{-1} x^*$. 
Lastly, the absolute magnitude of DShapley gets smaller as $m$ increases.

\begin{figure*}[h]
    \centering
    \includegraphics[scale=0.25]{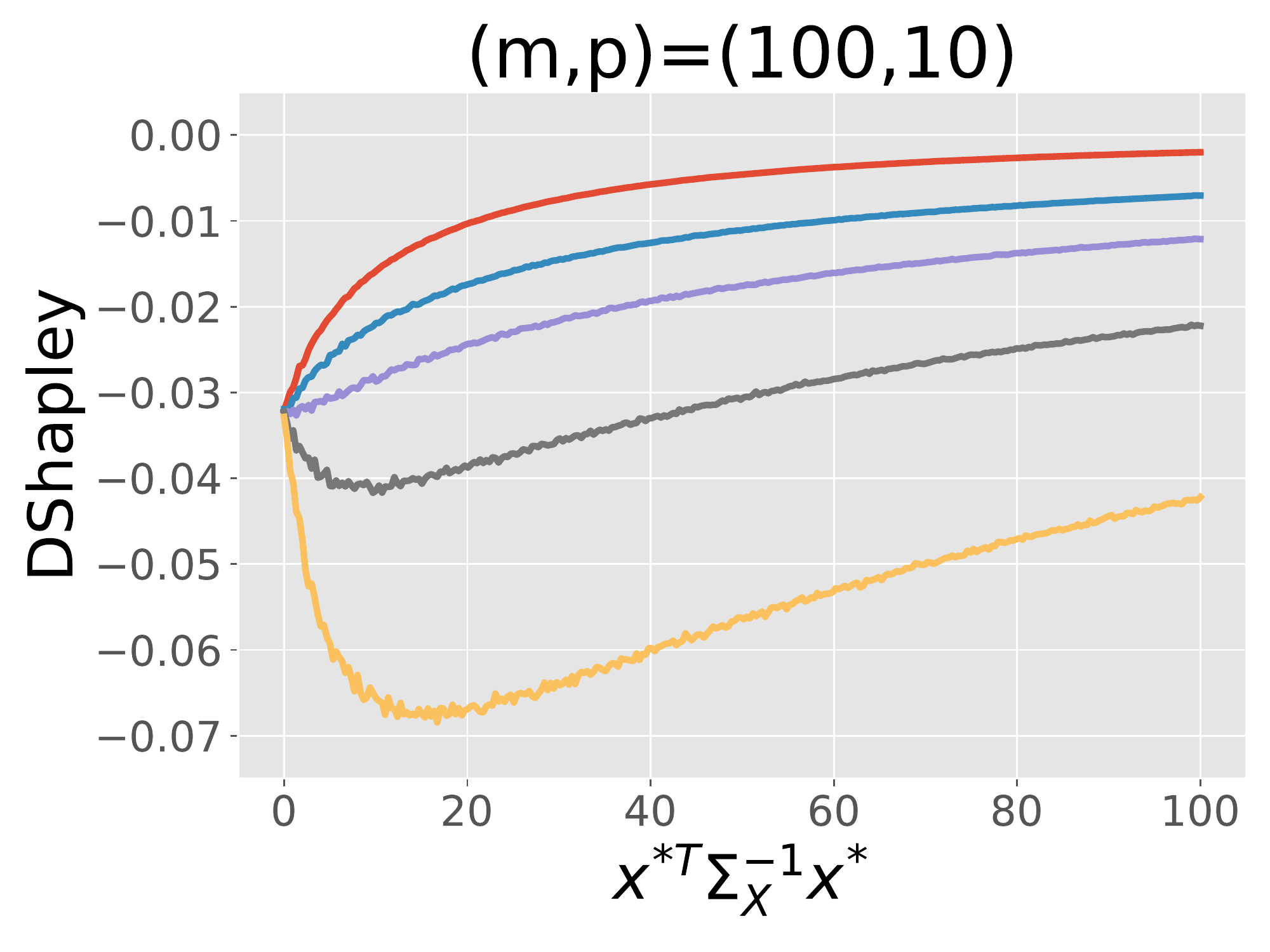}
    \includegraphics[scale=0.25]{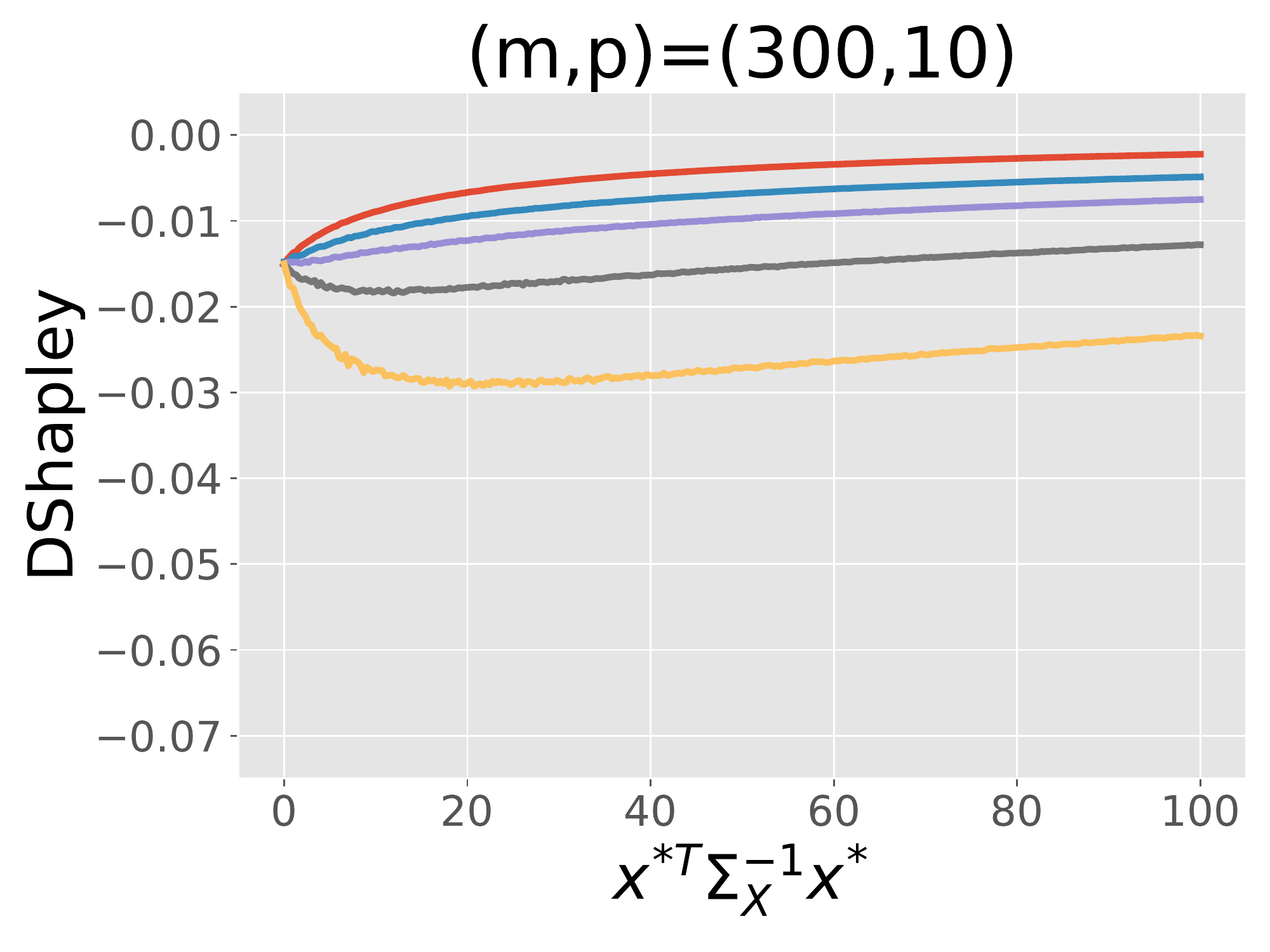}
    \includegraphics[scale=0.25]{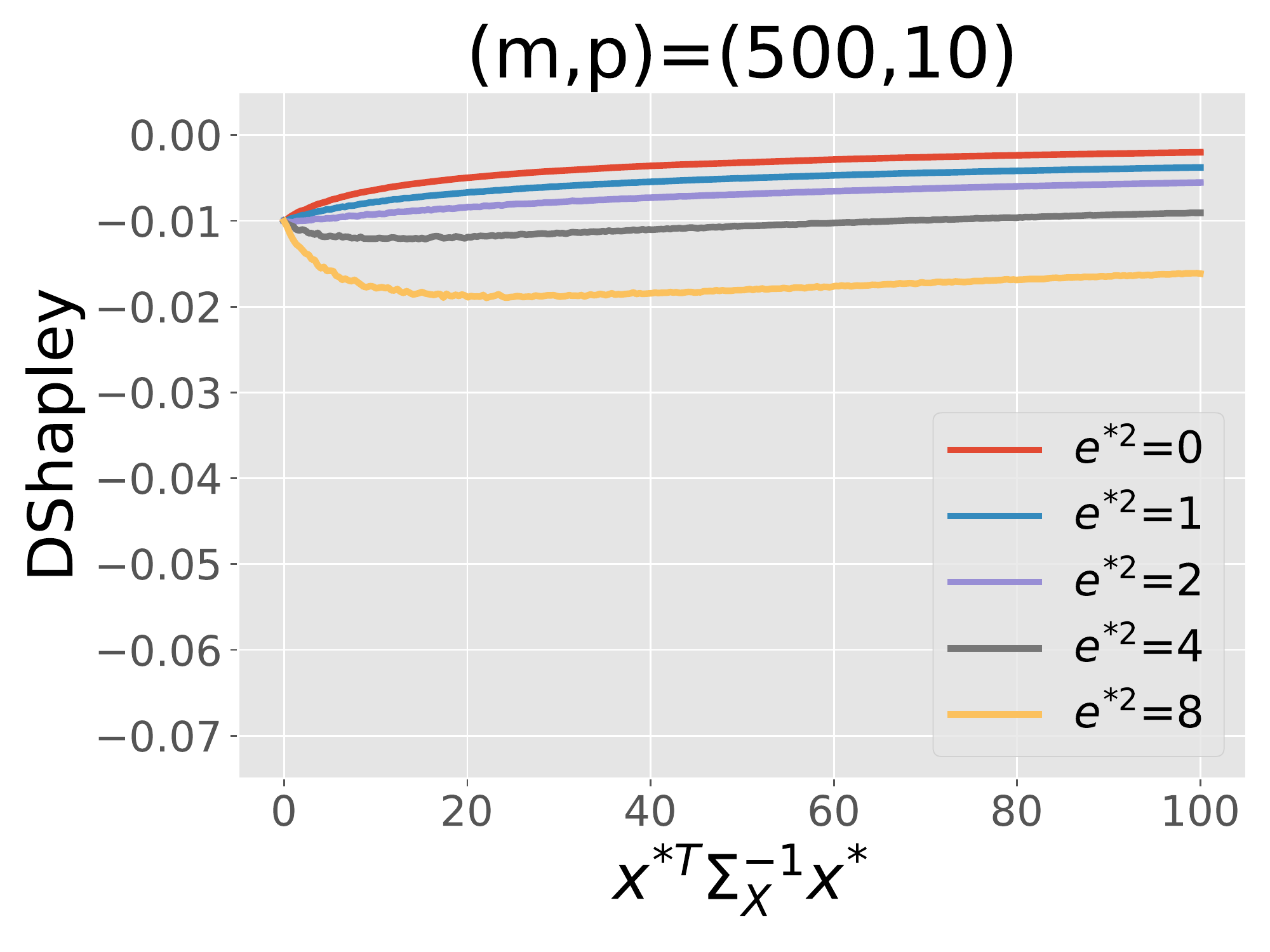}\\
    \includegraphics[scale=0.25]{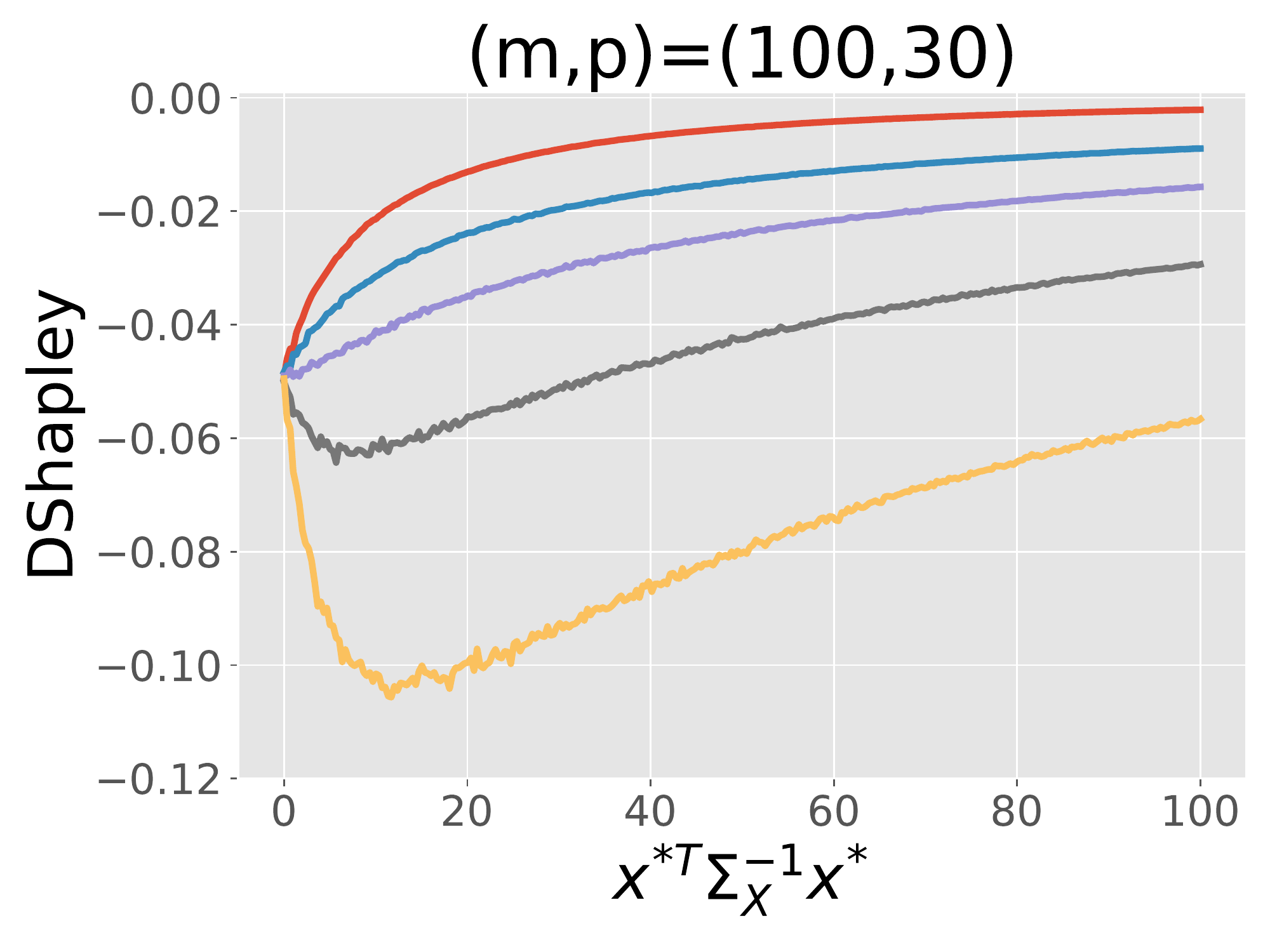}
    \includegraphics[scale=0.25]{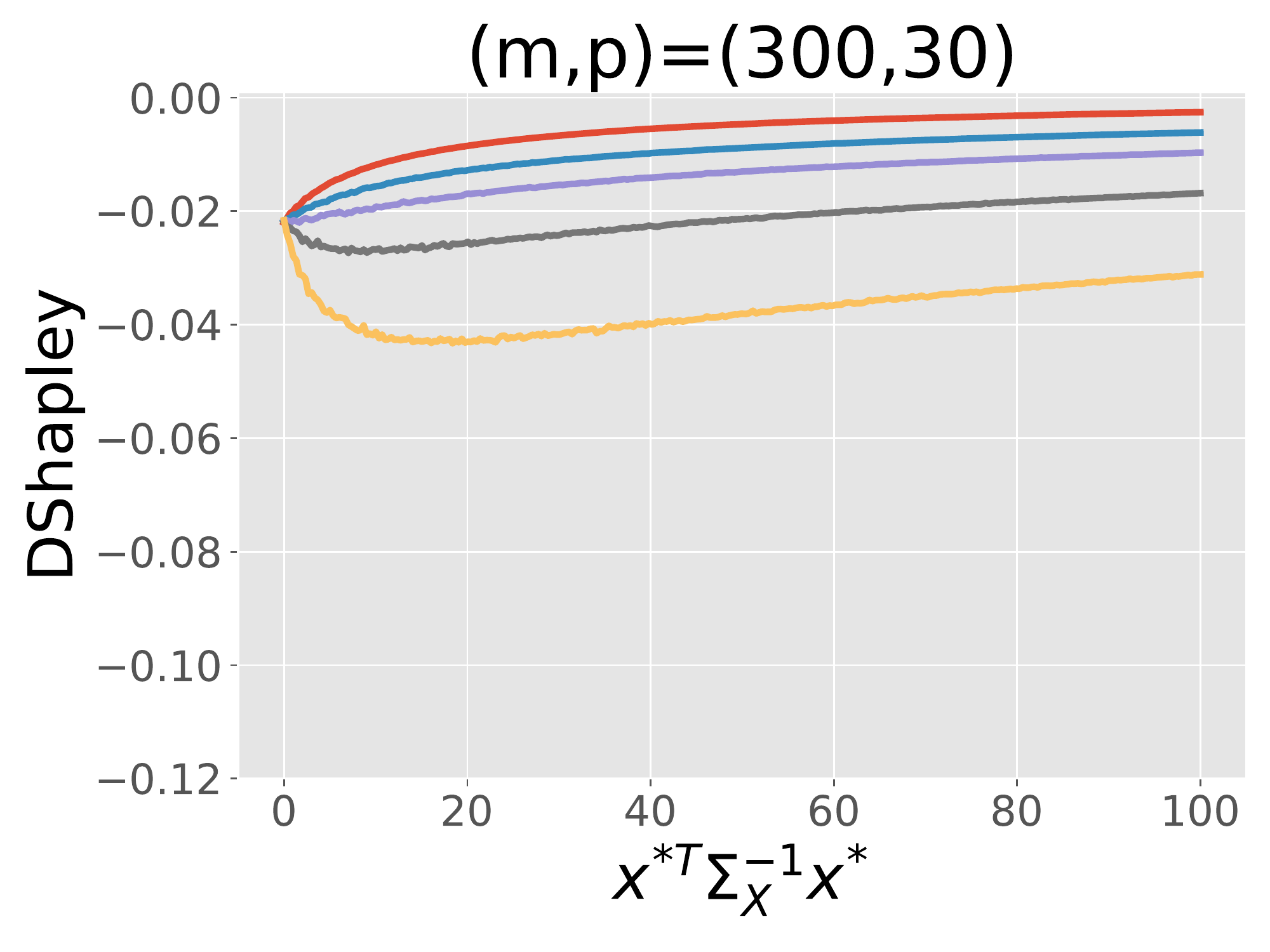}
    \includegraphics[scale=0.25]{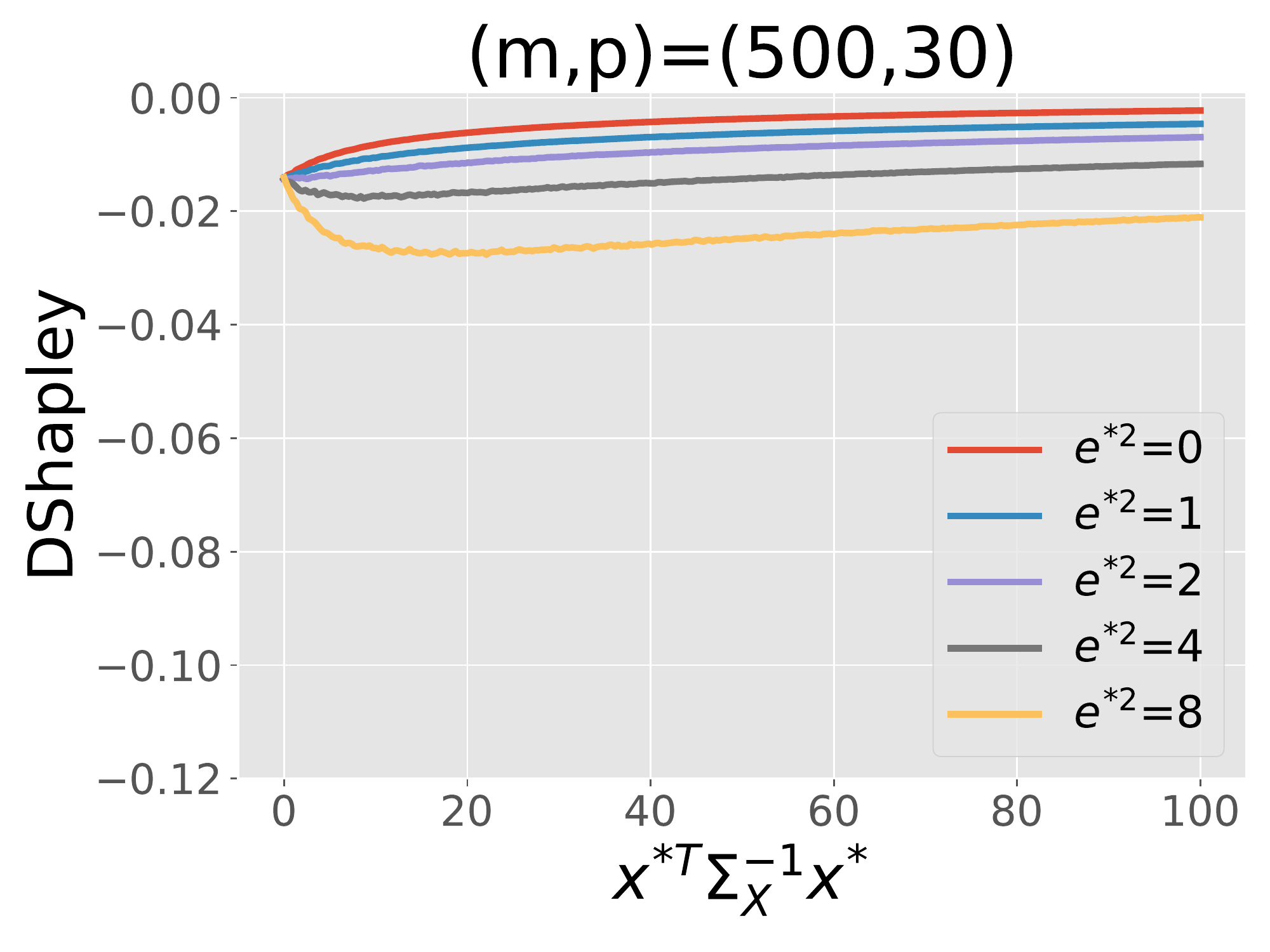}
    \caption{Illustration of DShapley as a function of the Mahalanobis distance $x^{*T} \Sigma_X ^{-1} x^{*}$ when the input dimension $p$ is either (top) 10 or (bottom) 30. Different colors indicate different error levels.}
    \label{fig:DSV_vs_x-norm}
\end{figure*}

\subsection{Point addition experiment with the upper and lower bounds of DShapley}
We additionally conduct the point addition experiment with the upper and lower bounds in \ref{thm:sub_gaussian_inputs_ridge}. Although the specific algorithm is not presented, it is straightforward from the `COMPUTE\_LOWER\_BOUND' procedure in Alg.~\ref{alg:DSV_IRLS_detail}. We use the same constants $C$, $c$, and $\rho$ defined in Alg.~\ref{alg:DSV_IRLS_detail}, but we here set $\gamma=1/200$. 

Figure~\ref{fig:exact_upper_lower_reg} and Figure~\ref{fig:exact_upper_lower_clf} show the upper and lower bounds of DShapley when ML problems are regression and classification, respectively. Note that $\mathcal{D}$-SHAPLEY shows the same plots. In our experiments, although the upper bound curves tend to show poor performance, the lower bound curves show promising results. The approximation of DShapley provides computationally efficient solutions, yet this phenomena shows one should be careful when using the approximation based on Theorem \ref{thm:sub_gaussian_inputs_ridge}.

\begin{figure*}[t]
    \centering
    \includegraphics[scale=0.125]{figures/regression_exact_airfoil.pdf}
    \includegraphics[scale=0.125]{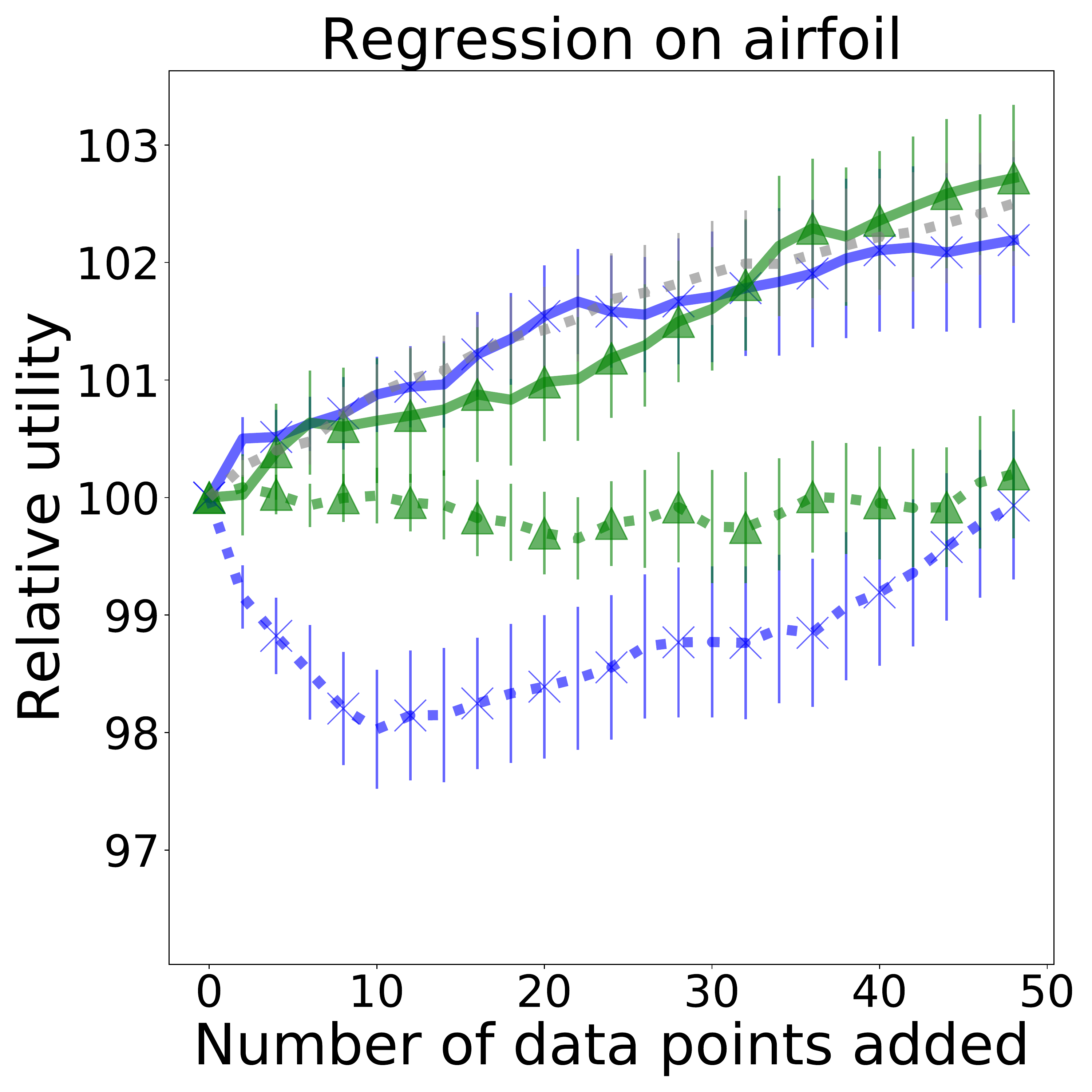}
    \includegraphics[scale=0.125]{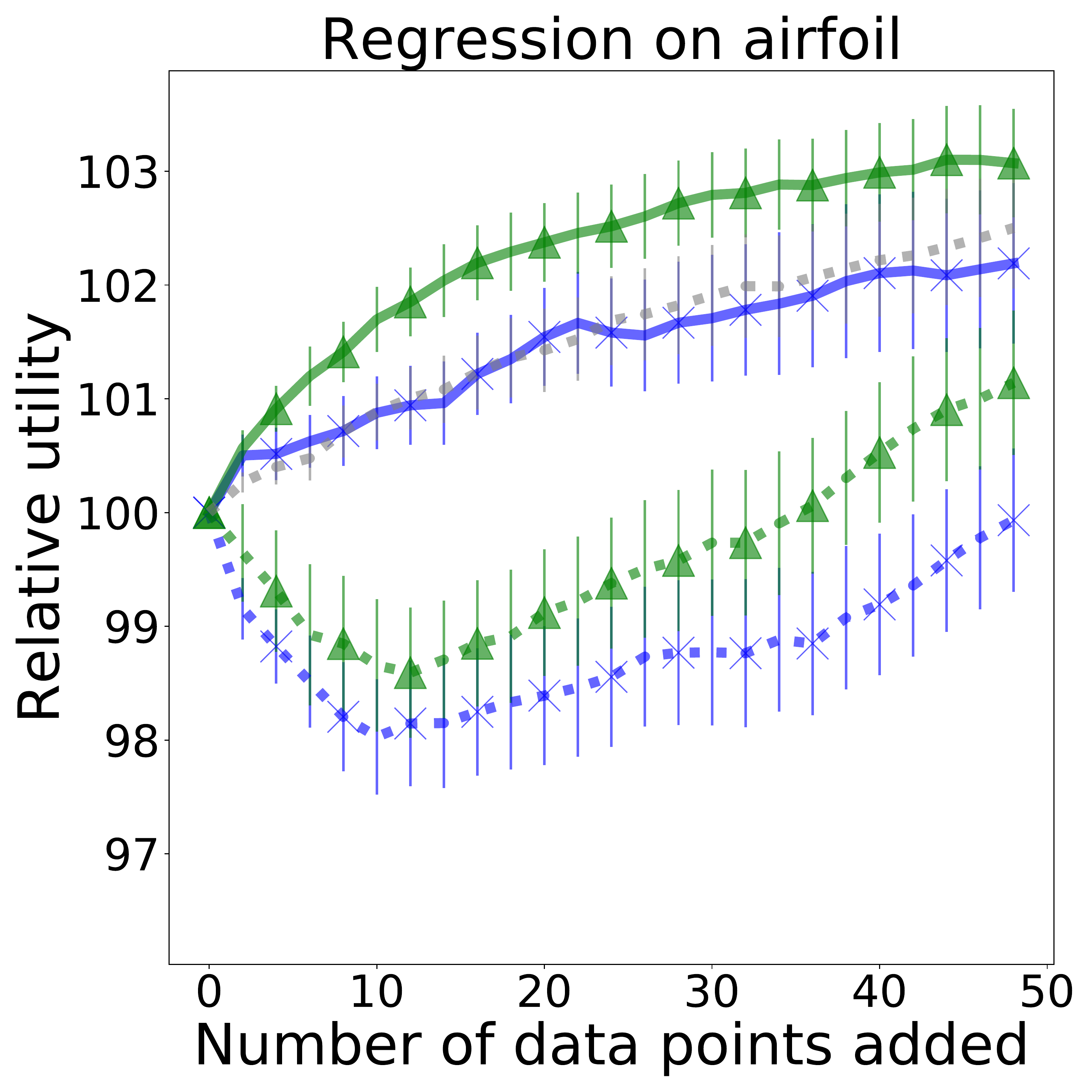}\\
    \includegraphics[scale=0.125]{figures/regression_exact_abalone.pdf}
    \includegraphics[scale=0.125]{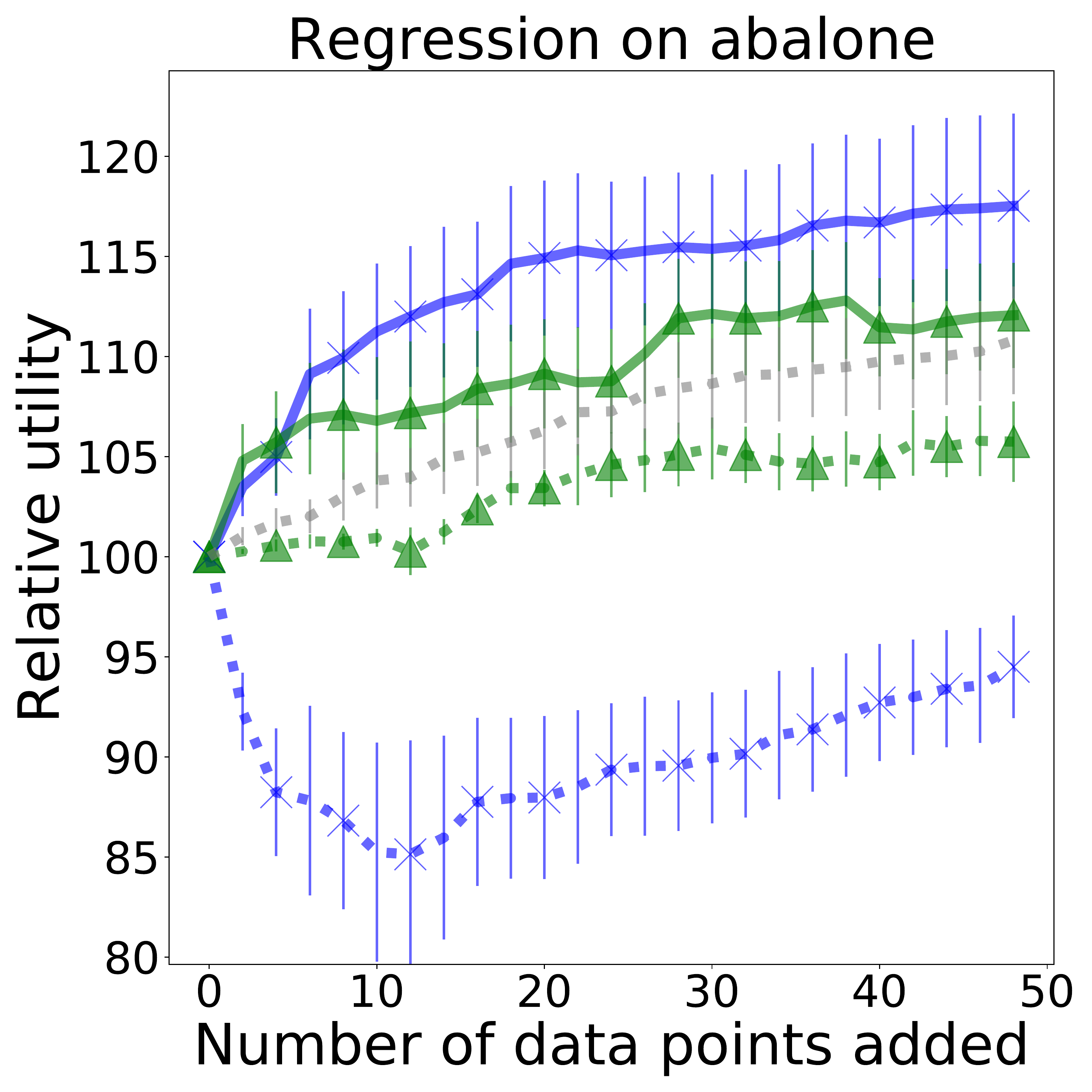}
    \includegraphics[scale=0.125]{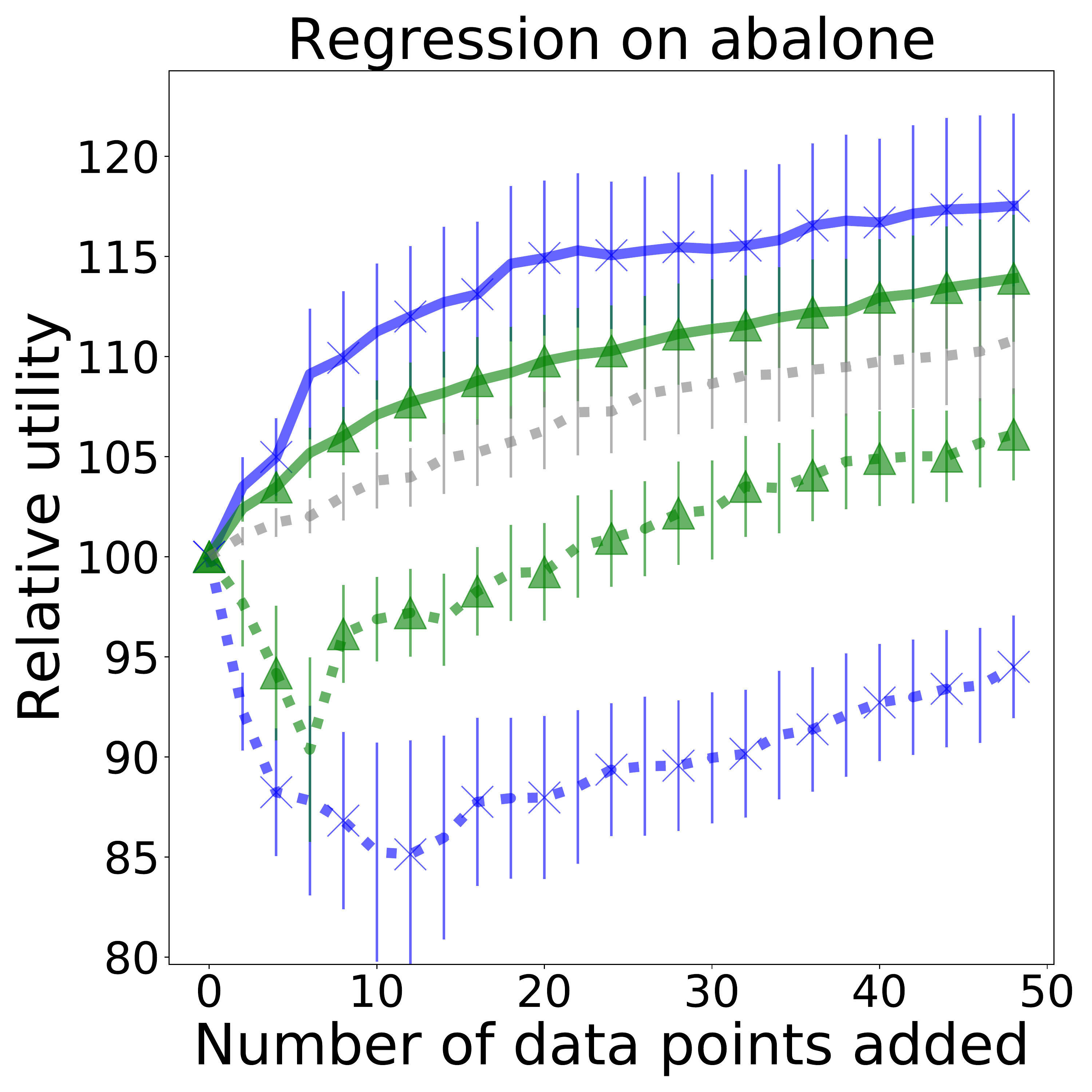}
    \caption{Relative utility and its standard error bar (in \%) as a function of the number of data added in linear regression settings. We examine the state-of-the-art \texttt{$\mathcal{D}$-SHAPLEY} (blue), random order (gray), and our proposed algorithms (green). As for the proposed algorithms, the exact DShapley based on Theorem~\ref{thm:gaussian_inputs_lse} (left), the upper (center), and the lower bounds based on Theorem~\ref{thm:sub_gaussian_inputs_ridge} (right). The solid and dashed curves correspond to adding points with the largest and smallest values first, respectively. The results are based on 50 repetitions.}
    \label{fig:exact_upper_lower_reg}
\end{figure*}
\begin{figure*}[t]
    \centering
    \includegraphics[scale=0.125]{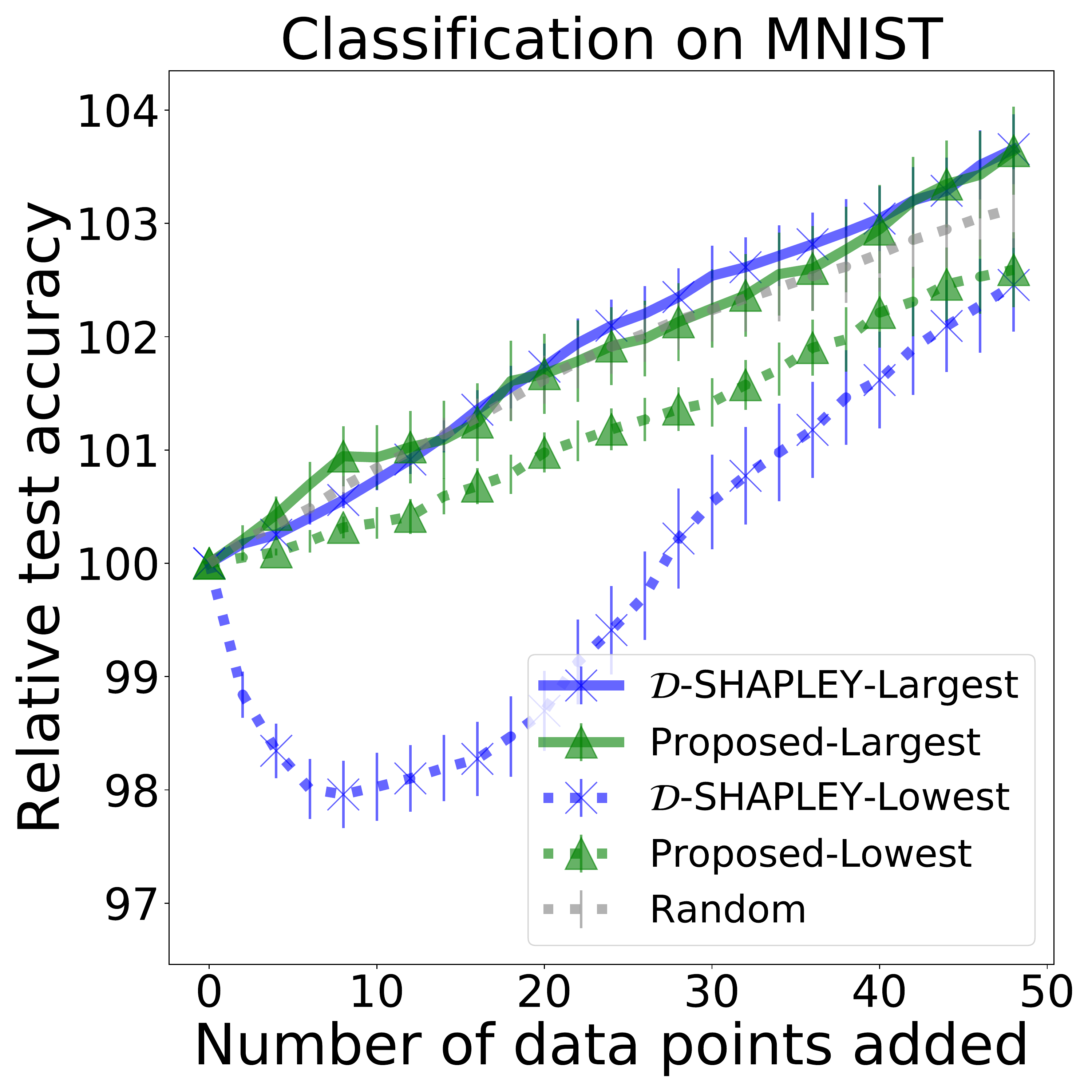}
    \includegraphics[scale=0.125]{figures/classification_lower_mnist.pdf}\\
    \includegraphics[scale=0.125]{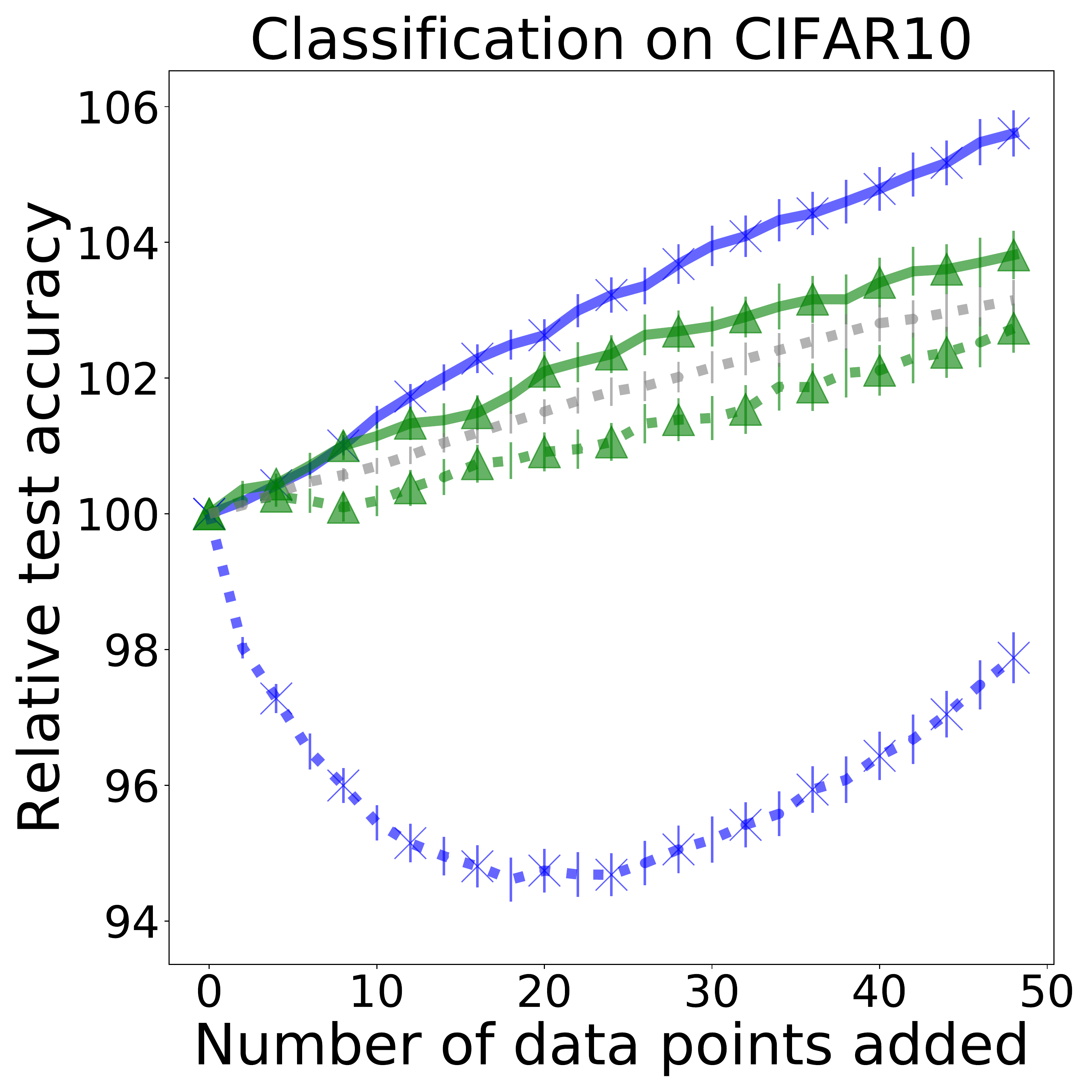}
    \includegraphics[scale=0.125]{figures/classification_lower_cifar10.pdf}
    \caption{Relative utility and its standard error bar (in \%) as a function of the number of data added in classification settings. We examine the state-of-the-art \texttt{$\mathcal{D}$-SHAPLEY} (blue), random order (gray), and our proposed algorithms (green). As for the proposed algorithms, the upper (left) and lower bounds based on Corollary~\ref{cor:gaussian_inputs_binary_full} (right). The solid and dashed curves correspond to adding points with the largest and smallest values first, respectively. The results are based on 50 repetitions.
    }
    \label{fig:exact_upper_lower_clf}
\end{figure*}

\section{A review of Shapley value and its uniqueness}
We briefly review the Shapley axioms: symmetry, null player, and additivity.
Under the axioms, we describe a fair valuation function \citep{shapley1953}.
Let $U$ be a utility function and $B$ be a dataset.
The three Shapley axioms are symmetry, null player, and additivity defined as follows.
\begin{itemize}
    \item \textit{Symmetry}: Let $z_i, z_j \in B$. For all $S \subseteq B \backslash \{z_i, z_j \}$, if $U(S\cup\{z_i\})=U(S\cup\{z_j\})$, then
    \begin{align*}
        \phi(z_i;U,B)=\phi(z_i;U,B).
    \end{align*}
    \item \textit{Null player}: Let $z_i \in B$. For all $S \subseteq B \backslash \{z_i \}$, if $U(S\cup\{z_i\})=U(S)$, then
    \begin{align*}
        \phi(z_i;U,B)=0.
    \end{align*}
    \item \textit{Additivity}: Let $U_1, U_2$ be two utility functions. For all $z \in B$,
    \begin{align*}
        \phi(z;U_1 + U_2,B)=\phi(z;U_1,B)+\phi(z;U_2,B).
    \end{align*}
\end{itemize}

Under the axioms, we provide the following uniqueness theorem quote from \citet[Proposition 293.1]{osborne1994course}. 
\begin{theorem}
Under the three Shapley axioms, the Shapley value is the unique valuation.
\end{theorem}

\end{document}